\documentclass[twoside]{article}

\usepackage[accepted]{aistats2026}

\usepackage[utf8]{inputenc}
\usepackage[T1]{fontenc}
\usepackage[english]{babel}
\usepackage{amsmath,amssymb,epsfig,amsthm}
\usepackage{graphicx}
\usepackage{comment}
\usepackage{array}
\usepackage{algorithm,algpseudocode}
\usepackage{enumitem}
\usepackage{algpseudocode}
\usepackage{xurl}
\usepackage{xargs}     
\usepackage{subcaption}
\usepackage{mwe}
\usepackage[normalem]{ulem}
\usepackage[pdftex,dvipsnames]{xcolor}  %
\usepackage[colorinlistoftodos,prependcaption,textsize=tiny]{todonotes}
\usepackage{appendix}
\usepackage{hyperref}
\usepackage{textalpha}
\usepackage{mathrsfs}

\usepackage[round]{natbib}
\usepackage[margin=1in]{geometry} 
\usepackage[round]{natbib}

\newcommandx{\unsure}[2][1=]{\todo[linecolor=red,backgroundcolor=red!25,bordercolor=red,#1]{#2}}
\newcommandx{\change}[2][1=]{\todo[linecolor=blue,backgroundcolor=blue!25,bordercolor=blue,#1]{#2}}
\newcommandx{\info}[2][1=]{\todo[linecolor=OliveGreen,backgroundcolor=OliveGreen!25,bordercolor=OliveGreen,#1]{#2}}
\newcommandx{\improvement}[2][1=]{\todo[linecolor=Plum,backgroundcolor=Plum!25,bordercolor=Plum,#1]{#2}}
\newcommandx{\thiswillnotshow}[2][1=]{\todo[disable,#1]{#2}}

\allowdisplaybreaks[1]

\usepackage{tikz}
\usepackage{randomwalk}

\allowdisplaybreaks[1]

\newtheorem{theorem}{Theorem}

\newtheorem{corollary}{Corollary}

\newtheorem{definition}{Definition}

\newtheorem{lemma}{Lemma}

\newtheorem{remark}{Remark}

\numberwithin{equation}{section}

\newcommand{\calC}{\ensuremath{\mathcal{C}}}

\newcommand{\calR}{\ensuremath{\mathcal{R}}}
\newcommand{\calS}{\ensuremath{\mathcal{S}}}

\newcommand{\calN}{\ensuremath{\mathcal{N}}}

\newcommand{\norm}[1]{\left\|{#1}\right\|}
\newcommand{\abs}[1]{\left|{#1}\right|}

\newcommand{\expec}{\ensuremath{\mathbb{E}}}

\newcommand{\prob}{\ensuremath{\mathbb{P}}}

\newcommand{\E}{\ensuremath{\mathbf{E}}}
\DeclareMathOperator{\sgn}{sgn}

\definecolor{asparagus}{rgb}{0.53, 0.66, 0.42}

\newcommand{\indic}{\ensuremath{\mathbf{1}}}

\newcommand{\R}{\ensuremath{\mathbb{R}}}
\newcommand{\mS}{\ensuremath{\mathbb{S}}}
\newcommand{\wsig}{w_{\text{sig}}}

\newcommand{\wopp}{w_{\text{opp}}}
\newcommand{\xii}{\xi_{\setminus i}}
\newcommand{\ix}{x_{\setminus i}}
\newcommand{\wpe}{w_{\perp}}

\newcommand{\nlrz}[1]{\nabla_{#1} L_{0}}

\newcommand{\ncl}{\nabla^{\text{cl}}}

\newcommand{\batch}{V}

\newcommand{\init}{0}

\newcommand{\cz}{c_{\zeta}}
\newcommand{\ct}{C_{\theta}}

\usepackage{tikz}

\allowdisplaybreaks

\begin{document}

\twocolumn[

\aistatstitle{Neuron Block Dynamics for XOR Classification with Zero-Margin}

\aistatsauthor{ Guillaume Braun  \And Masaaki Imaizumi }

\aistatsaddress{ RIKEN AIP \And  RIKEN AIP  \\ University of Tokyo } ] 

\addtocontents{toc}{\protect\setcounter{tocdepth}{-10}}

\begin{abstract}
  The ability of neural networks to learn useful features through stochastic gradient descent (SGD) is a cornerstone of their success. Most theoretical analyses focus on regression or on classification tasks with a positive margin, where worst-case gradient bounds suffice. In contrast, we study zero-margin nonlinear classification by analyzing the Gaussian XOR problem, where inputs are Gaussian and the XOR decision boundary determines labels. In this setting, a non-negligible fraction of data lies arbitrarily close to the boundary, breaking standard margin-based arguments. Building on Glasgow’s (2024) analysis, we extend the study of training dynamics from discrete to Gaussian inputs and develop a framework for the dynamics of neuron blocks. We show that neurons cluster into four directions and that block-level signals evolve coherently, a phenomenon essential in the Gaussian setting where individual neuron signals vary significantly. Leveraging this block perspective, we analyze generalization without relying on margin assumptions, adopting an average-case view that distinguishes regions of reliable prediction from regions of persistent error. Numerical experiments confirm the predicted two-phase block dynamics and demonstrate their robustness beyond the Gaussian setting.
\end{abstract}

\section{Introduction}\label{sec:intro}
In deep learning theory, feature learning has emerged as a central theme. It refers to the ability of neural networks to uncover latent low-dimensional structure from high-dimensional data, often yielding significantly better sample complexity than kernel methods. A line of recent work demonstrates this advantage in regression under the multi-index model, where labels depend on a nonlinear function of a low-dimensional projection of the inputs \citep{AbbeAM23,bietti2022learning,mousavi-hosseini2023neural,bietti2023learning,barak2022hidden,oko24a,dandi_reuse}. We refer to \cite{bruna2025surveyalgorithmsmultiindexmodels} for a recent survey.

While most work on feature learning has centered on regression, there has also been progress in \textit{classification}. Existing analyses fall into two main strands. The first is the literature on \emph{benign overfitting} \citep{osti_10344164, xu2023benign}, which shows that over-parameterized neural networks trained by stochastic gradient descent (SGD) can generalize even when they interpolate the data. The second strand investigates \textit{sparse parity learning} with Boolean inputs \citep{glasgow2023sgd,abbe2025learning}, a discrete setting where classes enjoy a positive separation margin. Both lines of work highlight that neural networks achieve feature learning in classification under favorable structural conditions such as linear separability or positive margin.

A major open challenge is to understand neural networks in \textit{nonlinear classification without positive margins}. Such situations naturally occur in real data---for instance, Figure~\ref{fig:margin_vs_zero_margin}(a) illustrates that digits 4 and 9 in MNIST \citep{mnist_dataset} overlap heavily, leaving no clear separation between classes. Yet despite its prevalence, the zero-margin regime remains poorly understood, since the lack of separation hinders alignment with discriminative features.

To investigate this setting in a controlled way, we study the \textit{canonical XOR classification problem with Gaussian inputs}, where labels are generated by
\[
f^*(x) = -\mathrm{sgn}(x_1 x_2), \qquad x=(x_1,x_2,\dots)\in \R^d,
\]
see Figure~\ref{fig:margin_vs_zero_margin}(b). This minimal toy problem captures the essential difficulties: the decision boundary is nonlinear, feature learning is required, and a non-negligible fraction of samples lie arbitrarily close to the boundary. Unlike the Boolean case, where inputs are discrete and bounded, Gaussian inputs are continuous and unbounded, so data both concentrate near the boundary and can take arbitrarily large values. Prior work by \citet{glasgow2023sgd} showed that vanilla SGD can solve the Boolean XOR problem, but their analysis relies critically on the presence of a positive margin. In the Gaussian case these arguments break down, motivating our \emph{block-dynamic viewpoint}, which departs from per-neuron margin arguments and explains how neural networks can still learn in the zero-margin regime.

\begin{figure}[t]
    \centering
\begin{subfigure}{0.47\linewidth}
    \includegraphics[width=\textwidth]{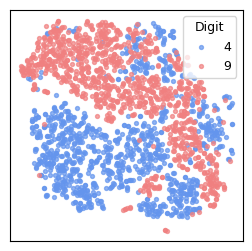}
    \caption{MNIST: 4 vs 9.}
\end{subfigure}%
\begin{subfigure}{0.47\linewidth}
    \includegraphics[width=\textwidth]{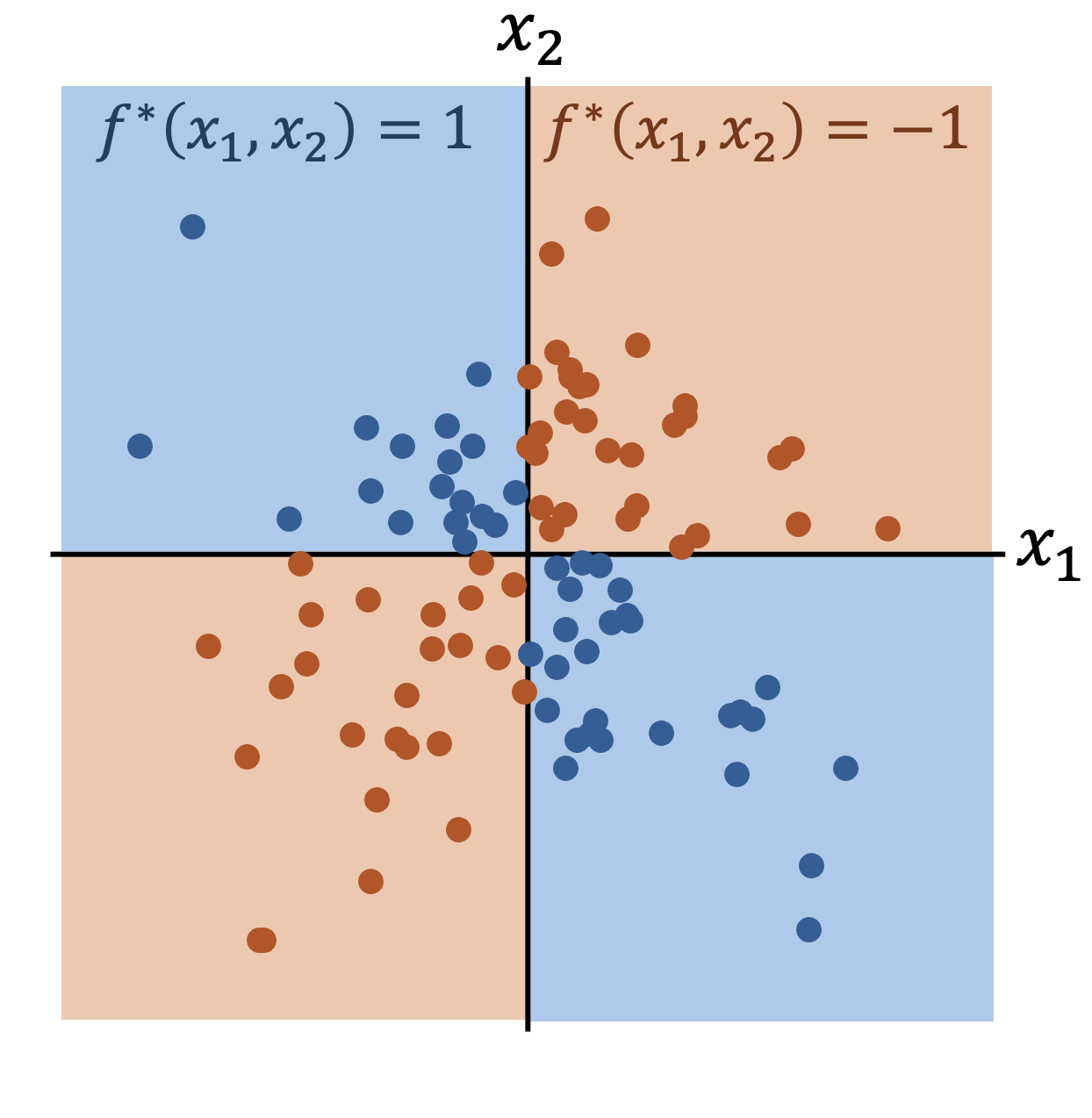}
    \caption{Our Setup.}
\end{subfigure}
\caption{(a) t-SNE plot of MNIST digits 4 and 9, which overlap substantially. (b) Gaussian XOR setup: $f^*(x) = -\mathrm{sgn}(x_1 x_2)$ with zero margin, since the Gaussian distribution crosses the coordinate axes. \label{fig:margin_vs_zero_margin}}

\end{figure}

Building on the work of \citet{glasgow2023sgd}, we develop a \textit{block-dynamic framework} for analyzing SGD in the \emph{zero-margin} regime, and prove that a two-layer neural network trained by vanilla SGD can learn the XOR function from Gaussian inputs.  
Our main contributions are:  

\begin{itemize}[leftmargin=*]
\item \textbf{Block dynamics.} We show that neurons rapidly self-organize into four coherent blocks aligned with the signal directions. A mean-field–style argument establishes that block masses remain comparable, enabling a low-dimensional block-level description of the network.
\item \textbf{Average-case analysis.} In the absence of a margin, many samples lie arbitrarily close to the decision boundary, invalidating worst-case arguments. We introduce an \emph{average margin} statistic and prove that it governs the accuracy of the classifier.
\item \textbf{Experiments.} Simulations confirm block formation, validate the average-case dynamics, and illustrate the importance of nonlinear separation in the zero-margin setting.
\end{itemize}

\subsection{Related Work}

\paragraph{Feature learning.} 
Two-layer neural networks have been shown to adapt to low-dimensional latent structure in regression and multi-index models \cite{AbbeAM23,bietti2022learning,mousavi-hosseini2023neural,bietti2023learning,barak2022hidden,oko24a,dandi_reuse,bruna2025surveyalgorithmsmultiindexmodels}, often achieving significantly lower sample complexity than kernel methods. In contrast, kernel perspectives such as the Neural Tangent Kernel (NTK) \cite{jacotNTK,Chizat2018OnLT} do not capture adaptation to latent structure \cite{montanari19}, while mean-field analyses \cite{Mei2019MeanfieldTO} typically require very wide networks and long training horizons \cite{suzuki2023feature}, leading to suboptimal sample complexity \cite{mahankali2023ntk}. Beyond these approximation regimes, most theoretical works analyze layer-wise training, where the first layer is optimized to recover signal directions before fitting the link function. However, such procedures are not faithful to practice and may even fail to generalize \cite{bietti2023learning}. This motivates the study of vanilla SGD with simultaneous training, as in \citet{glasgow2023sgd}, who analyzed XOR with Boolean inputs, and \citet{Berthier2024-pf}, who studied the single-index model in the mean-field regime.

\paragraph{Classification with neural networks.} 
The dynamics of neural networks in classification have been extensively studied under linearly separable models with a positive margin \citep{Wei2018RegularizationMG,brutzkus2018sgd,telgarsky20,lyu2021gradient,telgarsky2023feature}, often linked to the theory of benign overfitting \cite{osti_10344164,kornowski2023from,zhu24,pmlr-v235-wang24cn}. Benign overfitting has since been established across diverse architectures, including transformers \citep{benignViT24}, graph convolutional networks \citep{huang2025quantifying}, and convolutional networks \cite{pmlr-v202-kou23a,benignCNN}. Beyond linear separability, \citet{xu2023benign} analyzed an XOR Gaussian mixture model, but required that cluster separation scales with dimension. Other works motivated by feature learning, such as \citet{shi2022a}, investigated dictionary learning with discrete latent structure, while \citet{shi2023provable} provided guarantees in more general settings without explicit margin assumptions. However, these analyses assume bounded inputs and layer-wise training, in contrast to our focus on simultaneous training with unbounded Gaussian inputs.

\paragraph{XOR problem and parity learning.} Parity functions of the form $\prod_{i\in S}x_i$, where $x_i\in \lbrace -1,1\rbrace $, represent a class of functions that are generally challenging for gradient-based algorithms to learn without specific assumptions on the input distributions \cite{shalev17}. However, under additional assumptions on the input data, a two-layer neural network can effectively learn these functions, as demonstrated by \cite{barak2022hidden, AbbeAM23,kou2024matching,glasgow2023sgd,abbe2025learning}.

\subsection{Notations} 
We use the notation $a_n \lesssim b_n$ (or $a_n \gtrsim b_n$) for sequences $(a_n)_{n \geq 1}$ and $(b_n)_{n \geq 1}$ if there exists a constant $C > 0$ such that $a_n \leq C b_n$ (or $a_n \geq C b_n$) for all $n$. If the inequalities hold only for sufficiently large $n$, we write $a_n = O(b_n)$ (or $a_n = \Omega(b_n)$). We use $\norm{\cdot}$ and $\langle \cdot, \cdot \rangle$ to denote the Euclidean norm and scalar product, respectively.
The sign function is denoted by $\sgn(\cdot)$. The $(d-1)$-dimensional sphere of radius $\theta$ is denoted by $\mS^{d-1}(\theta)$.The canonical basis of $\R^d$ is denoted by $e_1, \ldots, e_d$.

\section{Statistical Framework}\label{sec:framework}

\paragraph{Data generation.} We model each observation  $(x,y)\in \R^d\times \lbrace \pm 1\rbrace$ as independently generated from the following process: the input $x$ is drawn from an isotropic Gaussian distribution,
\begin{align}
    x  = (x_1,x_2,...,x_d)^\top \sim \calN(0,I_d), \notag
\end{align}
and the label $y$, given $x$, is determined by the XOR function $f^*$
\begin{align}
    y= f^*(x) = -\sgn(x_1x_2). \notag
\end{align}
For simplicity, we assume that the function \( f^*(\cdot) \) depends only on the first two components of \( x \), specifically \( x_1 \) and \( x_2 \), which correspond to the canonical basis vectors \( e_1 \) and \( e_2 \). This assumption can be relaxed, as the algorithm in our analysis is rotationally invariant.

\paragraph{Model.}We utilize a two-layer neural network with the ReLU activation function, denoted by $\sigma(\cdot)$, to learn the XOR function $f^*(\cdot)$. This network consists of $m$ neurons, each characterized by a pair of weights $(w,a)$ where $w\in \R^d$ and $a\in \R$. Let $\calN$ represent the set of these $m$ neurons. The function represented by the neural network is defined as:
\[ f_\rho(x) := \frac{1}{m}\sum_{(w,a) \in \calN}a\sigma(w^\top x)= \expec_\rho [a\sigma(w^\top x)],\] where $\rho $ is the empirical distribution of the neurons in $\calN$.

\paragraph{Initialization.} To initialize each neuron $(w,a) \in \calN$, we sample initial weights $(w^{(0)},a^{(0)})$ as follows. 
The weights of the first layer are initialized uniformly on the sphere of radius $\theta$, which will be specified later. Specifically, we set $w^{(0)} \sim \text{Unif}(\mS^{d-1}(\theta))$ for some scale parameter $\theta>0$. The weights of the second layer are initialized as $a^{(0)}=\theta \varepsilon$ where $\varepsilon$ is an i.i.d. Rademacher variable, i.e. $\mathbb{P}(\varepsilon = 1) = \mathbb{P}(\varepsilon = -1) = 1/2$. 

\paragraph{Network training.} To train the network $f_\rho(\cdot)$, we employ Stochastic Gradient Descent (SGD) over $T$ iterations and a learning rate $\eta > 0$.
Each iteration $t$ uses an independent batch of size $\batch$, denoted by $M_t = \{(t-1)\batch +1, (t-1)\batch + 2,...., t \batch \}$.
The empirical risk for a batch $M \subset \mathbb{N}$ is defined as \[ \hat{L}_{\rho}= \frac{1}{\batch}\sum_{j\in M_t}\ell_{\rho}(x^{(j)}),\]
where $x^{(j)}$ is an independent sample from $ \calN(0,I_d)$ and $\ell_{\rho}(x)$ is the logistic loss: \[ \ell_{\rho}(x)=2\log \left( 1+\exp(-yf_{\rho}(x))\right).\] 
The population loss is $L_{\rho}=\expec_{x} [\ell_{\rho}(x)]$, and the  derivative $\ell'_{\rho}(x)$  of the logistic loss with respect to $f_\rho(x)$ is: \[\ell'_{\rho}(x) = -\frac{2y\exp(-yf_\rho(x))}{1+\exp(-yf_\rho(x))}.\]
The gradient updates of a neuron $(w,a) \in \calN$ are calculated as follows\footnote{Since the ReLU function is not differentiable at zero, we define $\sigma'(0)=0$ for convenience.  }
\begin{align*}
    \nabla_w \hat{L}_{\rho} &= \frac{1}{\batch}\sum_{j\in M_t}\ell'_{\rho}(x^{(j)})a\sigma'(w^\top x^{(j)})x^{(j)}\\ \nabla_a \hat{L}_{\rho} &= \frac{1}{\batch}\sum_{j\in M_t}\ell'_{\rho}(x^{(j)})\sigma(w^\top x^{(j)}).
\end{align*}
At each step, the weights are updated as follows
\[
\begin{cases}
    w^{(t+1)}&=w^{(t)}-\eta \nabla_{w^{(t)}} \hat{L}_{\rho^{(t)}}(w^{(t)}) \\
    a^{(t+1)}&= a^{(t)}-\eta  \nabla_{a^{(t)}} \hat{L}_{\rho^{(t)}}(a^{(t)})
\end{cases},
\]
starting from $(w^{(0)},a^{(0)})$.
After $t$ update steps, the set of neurons is $\calN^{(t)} := \{(w^{(t)},a^{(t)})\}$, and $\rho^{(t)} $ represents the empirical distribution of these neurons. 
For simplicity, we omit the iteration index $t$ when the context makes it clear. Since each iteration uses an independent batch, there are no stochastic dependencies between weights and inputs.

\section{Main Result}\label{sec:result}  

We show that a two-layer neural network, with both layers trained simultaneously by vanilla SGD, can learn the XOR function from Gaussian inputs in the zero-margin regime, where many samples lie arbitrarily close to the decision boundary.

\begin{theorem}\label{thm:main}
Let $\theta=(\log d)^{-C}$ for a sufficiently large constant $C>0$.  
Consider the network of Section~\ref{sec:framework}, trained with mini-batches
of size $\batch\geq d/\theta$, step size $\eta\asymp \theta$, and width
polynomial in $d$. Then, with high probability, there exists a stopping time $T\asymp(\log d)^{C+1}$ such that the expected loss on a fresh input satisfies
\[
    \expec_x \bigl[\ell_{\rho^{(T)}}(x)\bigr]
    \;=\; O\!\left(\tfrac{1}{\sqrt{\log \log d}}\right).
\]
\end{theorem}

\begin{remark}[Sample complexity]
The required sample complexity of our algorithm is $O(d\,\mathrm{polylog}(d))$, 
which is near-optimal up to logarithmic factors, consistent with CSQ lower bounds~\cite{AbbeAM23}.
\end{remark}

 The convergence rate of $\sqrt{\log \log d}^{-1}$ 
provided by Theorem~\ref{thm:main}, is slower than in the Boolean-input setting. The following remarks clarify the source of this slowdown and the interpretation of our bound.

\begin{remark}[Comparison with Boolean XOR]
In the Boolean-input setting of \citet{glasgow2023sgd}, the presence of a positive margin enables faster convergence. In contrast, for Gaussian inputs the absence of a margin makes the problem intrinsically harder: the expected loss is dominated by points lying arbitrarily close to the decision boundary.
\end{remark}

Beyond the loss guarantee of Theorem~\ref{thm:main}, our analysis also yields a classification guarantee: points that are sufficiently far from the decision boundary are classified correctly with overwhelming
probability. Toward this end, let us introduce for all $\epsilon \in (0,1)$ the set \[  \calC_\epsilon =\lbrace (r, \theta )\in \R^+\times [0, 2\pi):r( \abs{\sin \theta }\wedge \abs{\cos \theta}) \leq \epsilon  \rbrace.\]
\begin{corollary}
\label{cor:classification} Let us fix a constant $\epsilon \in (0,1)$. Then, under the assumptions of Theorem~\ref{thm:main}, if $x=z+\xi$-- where $z$ denotes the projection of $x$ onto the $(e_1,e_2)$-plane -- is an input generated independently such that $z\notin \calC_\epsilon$, there are constants $c, C>1$ such that with probability at least $1 - \exp\!\bigl(-C (\log d)^c\bigr)$, the trained network $f_{\rho^{(T)}}$ will correctly classify $x$.
\end{corollary}

This behavior is also clearly visible in Figure~\ref{fig:xp1}, where misclassifications occur only for points lying extremely close to the decision boundary.

A key novelty of our analysis is to track block-level dynamics rather than individual neurons. Neurons rapidly self-organize into four balanced blocks, whose growth is driven by an average margin statistic. These structural insights are essential for explaining how SGD succeeds in the zero-margin regime.

\section{Proof Outline}
Before outlining the proof strategy, we introduce the additional notation needed for the analysis. 

\paragraph{Additional notation.} 
For $w\in\R^d$, write $w=w_{1:2}+w_\perp$ with $w_{1:2}\in\mathrm{span}(e_1,e_2)$ and $w_\perp\perp\mathrm{span}(e_1,e_2)$.  
For $x\sim\calN(0,I_d)$, decompose $x=z+\xi$ with $z\sim\calN(0,I_2)$ and $\xi\sim\calN(0,I_{d-2})$ independent.  
Define the XOR directions $\mu_1=(1,-1)^\top/\sqrt{2}$ and $\mu_2=(1,1)^\top/\sqrt{2}$.  
Neurons $(w,a)$ will align with $\{\pm\mu_1,\pm\mu_2\}$ depending on the sign of $a$.  
Accordingly, we decompose $w_{1:2}$ into a \emph{signal} component $\wsig$ and an \emph{orthogonal} component $\wopp$:  
\[
\small
\begin{array}{ll}
\wsig =
\begin{cases}
\mu_1^\top w\, \mu_1 & a \ge 0,\\
\mu_2^\top w\, \mu_2 & a < 0,
\end{cases}
&
\wopp =
\begin{cases}
\mu_2^\top w\, \mu_2 & a \ge 0,\\
\mu_1^\top w\, \mu_1 & a < 0.
\end{cases}
\end{array}
\]

\paragraph{Roadmap.} Our analysis proceeds in two phases, followed by a synthesis step.  

\textbf{Phase I (individual dynamics).}
When the network outputs are small, we linearize the loss, approximating
$\ell_{\rho^{(t)}}(x)\approx \ell_0(x):=\log 2 - y\,f_{\rho^{(t)}}(x)$,
under which neurons evolve nearly independently.  
In the early stage (\emph{Phase~Ia}), all neurons grow at comparable rates.
Once the signal has amplified (\emph{Phase~Ib}), growth rates begin to diverge.
At this point, it is more natural to group neurons into four blocks aligned
with $\pm\mu_1,\pm\mu_2$, determined jointly by the sign of $a$ and the
initial correlation with these directions.
As illustrated in Figure~\ref{fig:phase1}, neurons align with these
cluster directions while approximately balanced across quadrants, a property we call \emph{pre-balanced block dynamics}.
This aggregated description will be crucial for Phase~II.

\begin{figure}[htbp]
    \centering
    \includegraphics[width=0.6\linewidth]{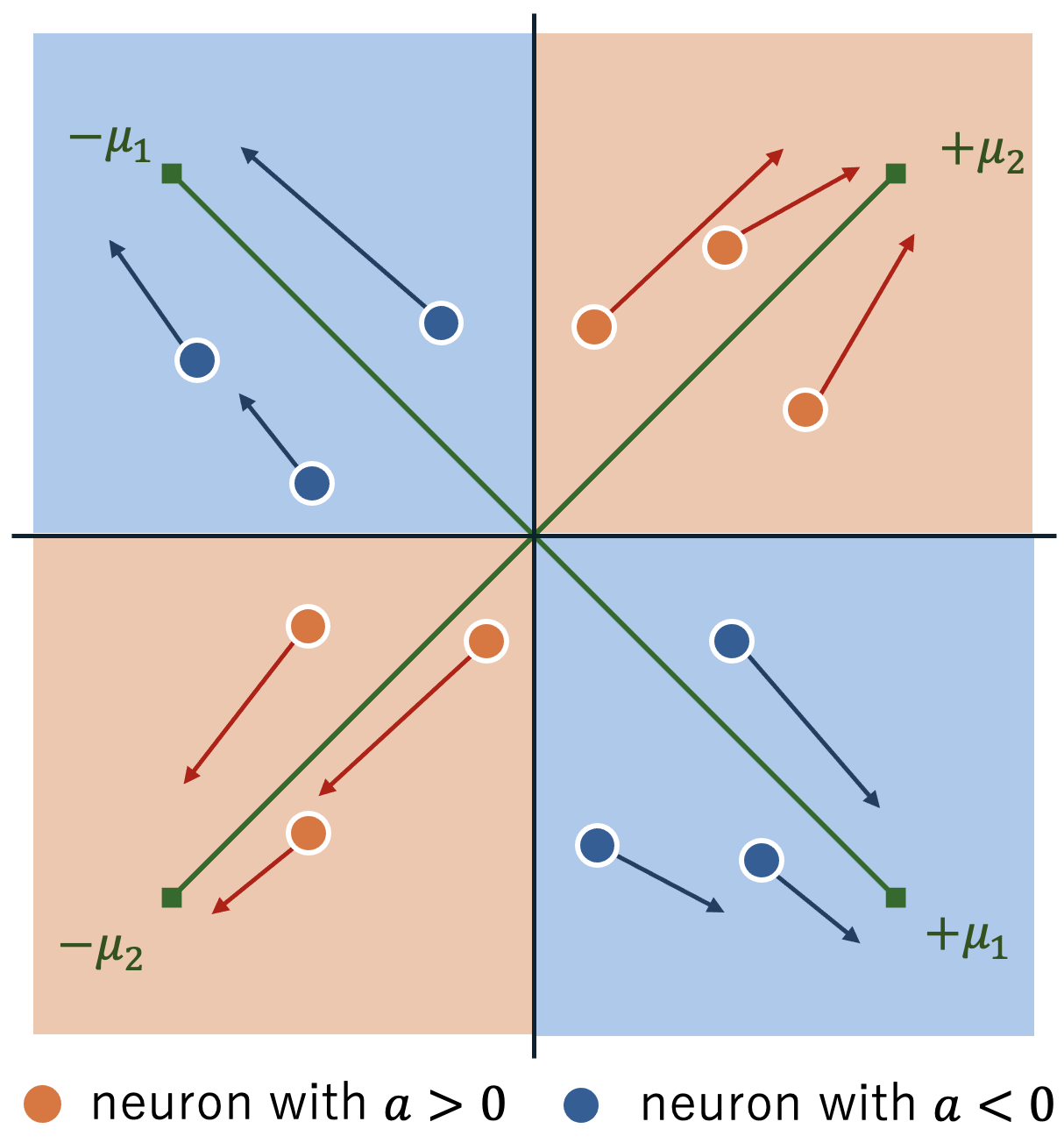}
    \caption{Neuron dynamics in Phase I. Circles show projections of weights onto 
    $\mathrm{span}(e_1,e_2)$; colors indicate the sign of $a$. Neurons align with 
    $\pm\mu_1,\pm\mu_2$ (green), and blocks remain balanced across quadrants.}
    \label{fig:phase1}
\end{figure}

\textbf{Phase II (block dynamics).}
As outputs grow, the linearized model ceases to be accurate, and neuron interactions become significant.
We analyze the network at the block level: block masses continue to grow
in a balanced way. For tractability, we use an oracle approximation
(replacing the full gradient $\ell'_{\rho}(x)$ with its two-dimensional projection $\ell'_{\rho}(z)$, and enforcing perfectly
balanced blocks), which lets us evaluate the average margin and derive
a multiplicative growth law for the aggregate block mass.

\textbf{Synthesis.} Finally, we leverage the characterization of the network obtained in
Phase~II to conclude the proof of Theorem~\ref{thm:main}. Figure~\ref{fig:phase_overview} summarizes the main steps of the proof and their relative importance in the network dynamics.

\subsection{Phase I: Early individual dynamics}\label{sec:analysis:phase1}
In this phase, the network output is small, and the loss
can be well approximated by its first-order expansion
\(
  \ell_{\rho^{(t)}}(x) \approx \ell_0(x) := \log 2 - y f_{\rho^{(t)}}(x).
\)
Under this approximation, neurons evolve independently, driven by the
population gradient.

At initialization we have w.h.p.\ $\|\wpe\|\approx\theta$, while
$\|\wsig\|,\|\wopp\|\lesssim\theta\log d/d$, hence
$\|\wsig\|\ll\|\wpe\|$. In this regime, the population gradient satisfies
\(
 -\nabla_{w_{1:2}}L_0 \approx 
   \tfrac{\sqrt{2}|a|}{\pi^{3/2}\|\wpe\|}(\wsig-\wopp),
\)
so the signal is reinforced while the orthogonal component is damped.
Controlling stochastic fluctuations by concentration, we obtain:

\begin{lemma}[Phase Ia dynamics]\label{lem:phaseIa}
For any neuron $(w,a)\in\calN$ with
$\|\wsig^{(0)}\|\gtrsim (d\log d)^{-1/2}$ and all $t\le T_a$,
\[
 \|\wsig^{(t+1)}\|
   =(1+\eta\tfrac{\sqrt{2}}{\pi^{3/2}}(1+o(1)))
     \|\wsig^{(t)}\|,
\]
while $\|\wpe^{(t)}\|\approx\theta$ and
\[
 \|\wopp^{(t)}\|\le (1+\eta\theta)\,
   \max\{\|\wopp^{(0)}\|,\,(d\log d)^{-1/2}\}.
\]
\end{lemma}

Thus, in Phase~Ia, most neurons experience multiplicative signal growth,
while the orthogonal component remains bounded, ensuring alignment toward the signal directions.

\subsubsection{Phase Ib: Heterogeneous growth}
After sufficiently many iterations, the signal $\|\wsig\|$ becomes comparable
to $\|\wpe\|$, so the approximation $\|\wsig\|\ll\|\wpe\|$ used in Phase~Ia
no longer holds. In this regime, we exploit instead that
$\|\wopp\|\ll\|\wsig\|$, which yields
\[
 (1+\eta c_1)\|\wsig^{(t)}\|
   \;\le\; \|\wsig^{(t+1)}\|
   \;\le\; (1+\eta c_2)\|\wsig^{(t)}\|,
\]
for some constants $0<c_1<c_2$. Thus, the signal continues to grow,
though at heterogeneous rates across neurons.

\begin{lemma}[Signal-based individual dynamics]\label{lem:phase1}
Let $\zeta=(\log d)^{-c}$ for some constant $1<c<C$.  
Under the assumptions of Theorem~\ref{thm:main}, there exists
$T_1\asymp \log d/\eta$ such that for all $t\le T_1$:
\begin{enumerate}
\setlength{\parskip}{0cm}\setlength{\itemsep}{0cm}
\item \emph{(Weak noise)} 
$\expec_{\rho^{(t)}}\|\,\wpe+\wopp\,\|^2 \le 4\theta^2$.
\item \emph{(Large signal)}  
There exists a constant $C'>0$ such that for all neurons with $\mu^\top w^{(0)}>\theta/\sqrt{d}$ for some
$\mu\in\{\pm\mu_1,\pm\mu_2\}$,
\[
   \theta\zeta^{-1}\;\le\;\|\wsig^{(T_1)}\|
   \;\le\;C'\theta\zeta^{-1}\log d.
\]
\end{enumerate}
\end{lemma}

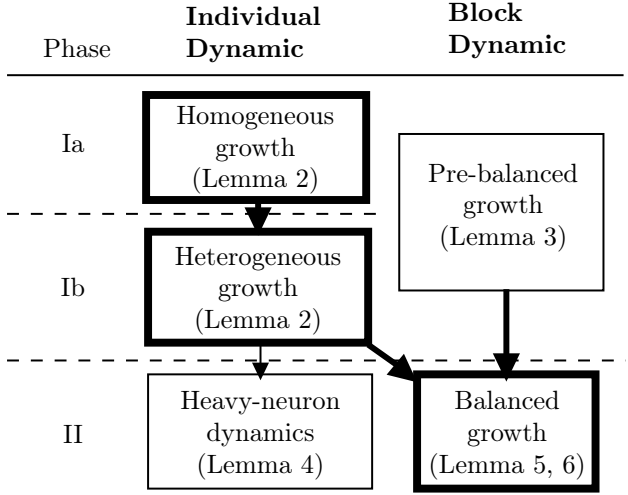
\begin{figure}[t]
    \centering

\tikzset{every picture/.style={line width=0.75pt}} %

\begin{tikzpicture}[x=0.75pt,y=0.75pt,yscale=-1,xscale=1]

\draw    (99.87,302) -- (410.81,302) ;
\draw  [dash pattern={on 4.5pt off 4.5pt}]  (99.87,372) -- (286.87,372) ;
\draw  [dash pattern={on 4.5pt off 4.5pt}]  (99.87,446) -- (412,446) ;
\draw  [line width=3]  (170,312) -- (281,312) -- (281,365) -- (170,365) -- cycle ;
\draw  [line width=3]  (171,381) -- (282,381) -- (282,437) -- (171,437) -- cycle ;
\draw  [line width=0.75]  (171,453) -- (284,453) -- (284,510.67) -- (171,510.67) -- cycle ;
\draw  [line width=0.75]  (297.67,331.67) -- (400.67,331.67) -- (400.67,410.67) -- (297.67,410.67) -- cycle ;
\draw  [line width=3]  (306.67,453) -- (396.67,453) -- (396.67,510.67) -- (306.67,510.67) -- cycle ;
\draw [line width=2.25]    (282,438) -- (303.43,453.19) ;
\draw [shift={(306.67,457)}, rotate = 219.62] [fill={rgb, 255:red, 0; green, 0; blue, 0 }  ][line width=0.08]  [draw opacity=0] (14.29,-6.86) -- (0,0) -- (14.29,6.86) -- cycle    ;
\draw [line width=2.25]    (226.5,365) -- (226.5,378) ;
\draw [shift={(226.5,383)}, rotate = 270] [fill={rgb, 255:red, 0; green, 0; blue, 0 }  ][line width=0.08]  [draw opacity=0] (14.29,-6.86) -- (0,0) -- (14.29,6.86) -- cycle    ;
\draw [line width=0.75]    (227.7,438.67) -- (227.7,453) ;
\draw [shift={(227.7,456)}, rotate = 270] [fill={rgb, 255:red, 0; green, 0; blue, 0 }  ][line width=0.08]  [draw opacity=0] (8.93,-4.29) -- (0,0) -- (8.93,4.29) -- cycle    ;
\draw [line width=2.25]    (351.7,410) -- (351.7,451.67) ;
\draw [shift={(351.7,456.67)}, rotate = 270] [fill={rgb, 255:red, 0; green, 0; blue, 0 }  ][line width=0.08]  [draw opacity=0] (14.29,-6.86) -- (0,0) -- (14.29,6.86) -- cycle    ;

\draw (224.18,266) node [anchor=north] [inner sep=0.75pt]   [align=left] {\textbf{Individual}\\\textbf{Dynamic}};
\draw (352.18,264) node [anchor=north] [inner sep=0.75pt]   [align=left] {\textbf{Block}\\\textbf{Dynamic}};
\draw (133.18,330) node [anchor=north] [inner sep=0.75pt]   [align=left] {Ia};
\draw (133.18,403) node [anchor=north] [inner sep=0.75pt]   [align=left] {Ib};
\draw (133.18,477) node [anchor=north] [inner sep=0.75pt]   [align=left] {II};
\draw (225.18,316) node [anchor=north] [inner sep=0.75pt]   [align=left] {\begin{minipage}[lt]{71.9pt}\setlength\topsep{0pt}
\begin{center}
Homogeneous \\growth\\(Lemma \ref{lem:phase1})
\end{center}

\end{minipage}};
\draw (226.5,387) node [anchor=north] [inner sep=0.75pt]   [align=left] {\begin{minipage}[lt]{72.47pt}\setlength\topsep{0pt}
\begin{center}
Heterogeneous\\growth\\(Lemma \ref{lem:phase1})
\end{center}

\end{minipage}};
\draw (227.7,460) node [anchor=north] [inner sep=0.75pt]   [align=left] {\begin{minipage}[lt]{66.79pt}\setlength\topsep{0pt}
\begin{center}
Heavy-neuron\\dynamics\\(Lemma \ref{lem:inductive2})
\end{center}

\end{minipage}};
\draw (352.18,460) node [anchor=north] [inner sep=0.75pt]   [align=left] {\begin{minipage}[lt]{60.04pt}\setlength\topsep{0pt}
\begin{center}
Balanced\\growth\\(Lemma \ref{lem:pop2_grad_main},  \ref{lem:inductive3})
\end{center}

\end{minipage}};
\draw (350.18,345) node [anchor=north] [inner sep=0.75pt]   [align=left] {\begin{minipage}[lt]{63.4pt}\setlength\topsep{0pt}
\begin{center}
Pre-balanced\\growth\\(Lemma \ref{lem:ineq_U})
\end{center}

\end{minipage}};
\draw (135.18,282) node [anchor=north] [inner sep=0.75pt]   [align=left] {Phase};

\end{tikzpicture}

    \caption{Overview of phases and neuron dynamics. Thicker-bordered boxes emphasize the key dynamics that capture how the network evolves in each phase. \label{fig:phase_overview}}
\end{figure}

In short, by the end of Phase~I, each neuron develops a strong signal
aligned with $\mu_1$ or $\mu_2$, while the noise components
$\wpe+\wopp$ remain controlled.

\subsubsection{Blocks are nearly balanced}

When neuron growth becomes heterogeneous in Phase~Ib, it is crucial to reason at the block level.  
We group neurons into four \emph{blocks}, aligned with $\pm\mu_1$ and $\pm\mu_2$, 
and demonstrate that their masses remain nearly balanced despite stochastic fluctuations.

\begin{definition}[Neuron blocks]
Let $\calN^{(0)}$ be the set of neurons at initialization.  
Define
\[
\begin{aligned}
\calN_1^{\pm} &= \{(w,a)\in\calN^{(0)} : a>0,\ \pm w^\top\mu_1>\init\},\\
\calN_2^{\pm} &= \{(w,a)\in\calN^{(0)} : a<0,\ \pm w^\top\mu_2>\init\}.
\end{aligned}
\]
For $i=1,2$, the block mass at time $t$ is
\[
N_i^{\pm,(t)}=\frac{1}{|\calN_i^{\pm}|}\sum_{(w,a)\in\calN_i^{\pm}}
   \|a^{(t)}\wsig^{(t)}\|,
\]
and the average mass is
\[
N^{(t)}=\tfrac{1}{4}\sum_{i=1,2}\bigl(N_i^{+,(t)}+N_i^{-,(t)}\bigr).
\]
\end{definition}

To quantify balance, we define the \emph{unbalance level}
\[
U^{(t)}=\max_{i,j\in\{1,2\}, \pm}
   \left|\tfrac{N_i^{\pm,(t)}}{N_j^{\pm,(t)}}-1\right|.
\]
Using the correlation-loss approximation $L_\rho\approx L_0$,  
we compare the true weight sequence $(w^{(t)},a^{(t)})$ with a surrogate
sequence $(\tilde{w}^{(t)},\tilde{a}^{(t)})$ that evolves under $L_0$.
The key property is that these surrogate updates preserve independence,
and block laws differ only by a rotation.  
Hence, block averages remain close, up to sampling noise from SGD.

\begin{lemma}[Pre-balanced blocks]\label{lem:ineq_U}
Under the assumptions of Theorem~\ref{thm:main}, w.h.p.\ for all $t\le T_1$,
\[
U^{(t)} \;\le\; \log^{-c_U} d,
\]
where $c_U>0$ can be made arbitrarily large by increasing the batch size.
\end{lemma}
\emph{Proof idea.} To control $U^{(t)}$, we introduce a \emph{surrogate sequence} $(\tilde w^{(t)},\tilde a^{(t)})$ 
initialized as $(w^{(0)},a^{(0)})$ but updated by population gradients $\nabla L_0$ where $L_0 = \expec_x(\ell_0(x))$.  
Block averages can then be approximated via conditional expectations:
\[
\frac{1}{|\calN_i^{\pm}|}\sum_{(w,a)\in\calN_i^{\pm}}
   \|\tilde w^{(t)}_{\text{sig}}\|\,|\tilde a^{(t)}|
   \;\approx\;
   \E\!\left[\|\tilde w^{(t)}_{\text{sig}}\|\,|\tilde a^{(t)}|\,\middle|\,\calN_i^{\pm}\right].
\]
A key symmetry (Lemma~\ref{lem:block_rot}) shows that these conditional laws differ only by a rotation.  
Since $\|\tilde w^{(t)}\||\tilde a^{(t)}|$ is rotation-invariant, the conditional expectations coincide across blocks.  
Concentration of the empirical averages around these expectations then yields near-balance.

Thus, by the end of Phase~I, neurons have developed strong signals
(Lemma~\ref{lem:phase1}) and block masses remain nearly balanced
(Lemma~\ref{lem:ineq_U}), preparing the ground for the block-dynamic
analysis in Phase~II.

\subsection{Phase II : Blocks dynamics}\label{sec:analysis:phase2}

In Phase~I, neurons evolved nearly independently under the correlation-loss approximation.  
In Phase~II, this approximation no longer holds: interactions between neurons matter, and the signal component dominates.  
Our analysis, therefore, shifts to the block level. 

\paragraph{Signal-heavy networks.}  
Some neurons fail to develop large signals, but their contribution to $f_{\rho^{(t)}}$ remains negligible.  
We thus focus on \emph{heavy neurons}, where the signal $\wsig$ dominates the noise terms $\wpe,\wopp$.  

\begin{definition}[Signal-heavy network]
    Let $H>1$ be a constant and $\zeta'=o(1)$. A network is said to be $(\zeta', H)$-signal-heavy if there is a set $\calS \subset \calN$ of heavy neurons that satisfy the following conditions:\begin{enumerate}
    \setlength{\parskip}{0cm}
  \setlength{\itemsep}{0cm}
        \item \rm{(Signal heavy neuron):} For all neurons in $ \calS, \norm{\wpe}+\norm{\wopp}\leq \zeta' \norm{\wsig}$ holds. 
        \item \rm{(Non-heavy neurons mass is negligible):} $\expec_\rho \indic_{\lbrace (w,a)\not \in \calS \rbrace}\norm{w}^2\leq \zeta' N^{(t)}$ holds.
        \item \rm{(Layer weights balance):}  $\expec_\rho \norm{w}^2\leq \expec_\rho a^2+\zeta' H$ and $|a|\leq \norm{w}$ hold for all neurons.
    \end{enumerate} 
\end{definition}

Intuitively, heavy neurons carry the signal, while non-heavy ones are negligible, see Figure~\ref{fig:phase2} for an illustration.  
The following lemma formalizes that this property is stable over training (see Appendix~\ref{app:proof_inductive_lem}).  

\begin{lemma}[Stability of signal-heavy network]\label{lem:inductive2}
If the network is $(\zeta',H)$-signal-heavy at time $t$, then after one gradient step, with probability at least $1-d^{-\Omega(1)}$, it remains $(\zeta'(1+O(\eta\zeta')),H)$-signal-heavy.
\end{lemma}

\paragraph{Oracle approximation.}  
To analyze block dynamics, we approximate population gradients in two steps:  

\emph{Step 1}: Since the signal component is dominant, we approximate \( w^\top x \) by \( \wsig^\top z \), allowing us to use \( \ell'_{\rho^{(t)}}(x) \approx \ell'_{\rho^{(t)}}(z) \). To formalize this approximation, we define a \textit{clean gradient} \( \ncl \) as
\[
\begin{aligned}
\ncl_w L_{\rho} &:= a_w \,\mathbb{E}_{x}\,\ell'_{\rho}(z)\,\sigma'(w^\top x)\,x,\\
\ncl_a L_{\rho} &:= \mathbb{E}_{x}\,\ell'_{\rho}(z)\,\sigma(w^\top x).
\end{aligned}
\]

\emph{Step 2}: We introduce an \textit{oracle network}, which approximates the model \( {f}_{\rho^{(t)}} \) by the following model across each direction \( \pm \mu_1 \) and \( \pm \mu_2 \):
\[
    f_{\rho^{(t)}, \text{id}} (z) = N^{(t)} \sum_{i=1}^2 (-1)^{i+1} \left( \sigma(\mu_i^\top z) + \sigma(-\mu_i^\top z) \right).
\]
Using the approximate balance of the neuron blocks established in Phase I (see Lemma \ref{lem:ineq_U}), we obtain the approximation \( \ell'_{\rho}(z) \approx \ell'_{\rho, \text{id}}(z) \). For details, see Section \ref{sec:clean_gradient_oracle}. Also see Figure~\ref{fig:phase2} for an illustration.

We also define the notion of an \textit{average margin} for the oracle network as
\[
    g_{\mu}^{(t)} = \mathbb{E}_z \, \ell'_{\rho^{(t)}, \text{id}}(z) \, \sigma(\mu^\top z).
\]
We will sometimes forget the dependence on time $t$ to simplify the notation.
Using this notion, we describe the dynamics of the oracle network in the following lemma. The symmetry of this quantity within the oracle model will be leveraged in the lemma to show that the gradients evolve at the same rate in each direction. A full version of the following lemmas is presented in Section \ref{app:phase2:pop_grad}.

\begin{lemma}[Block Dynamics via the Oracle Approximation] \label{lem:pop2_grad_main}
The followings hold:
\begin{enumerate}
    \setlength{\parskip}{0cm}
  \setlength{\itemsep}{0cm}
    \item \rm{(oracle alignment):}  $(w,a)\in \calS$ such that $\wsig^\top \mu >0$ for some $\mu \in \lbrace \pm \mu_1, \pm \mu_2 \rbrace$ we have\[
\begin{aligned}
\mu^\top \ncl_w L_{\rho,\mathrm{id}} &= -|a|g_\mu(1 \pm o(1)),\\
y\ncl_a L_{\rho,\mathrm{id}} &= -(1 \pm o(1))\|\wsig\| g_\mu.
\end{aligned}
\]
 
    \item \rm{(average margin asymptotics):} There exist constants $0<c_1< c_2$ such that we have for all $t\leq T_b$
    \begin{align*}        
    c_1(1\wedge (N^{(t)})^{-3})\leq g_{\mu^{(t)}} \leq c_2 (1\wedge (N^{(t)})^{-3} ).
    \end{align*}
\end{enumerate}
\end{lemma}
The first result shows that gradients align with the signal and scale with the average margin $g_\mu$. The second bounds $g_\mu$: initially constant, 
it decays as $N^{(t)}$ grows.

\paragraph{Inductive block growth.}  
Finally, we show that block mass grows multiplicatively at rate $g_\mu$:

\begin{lemma}[Inductive Block Growth]\label{lem:inductive3}
If the network is $(\zeta',H)$-signal-heavy at time $t$, then with probability at least $1-d^{-\Omega(1)}$,
\[
N^{(t+1)}=(1+2\eta g_\mu^{(t)})(1+o(1))\,N^{(t)}.
\]
\end{lemma}

\noindent Combined with the pre-balance of Phase~I, this ensures sustained, symmetric block growth and drives the loss toward zero.

\begin{figure}[htbp]
    \centering
    \includegraphics[width=0.95\linewidth]{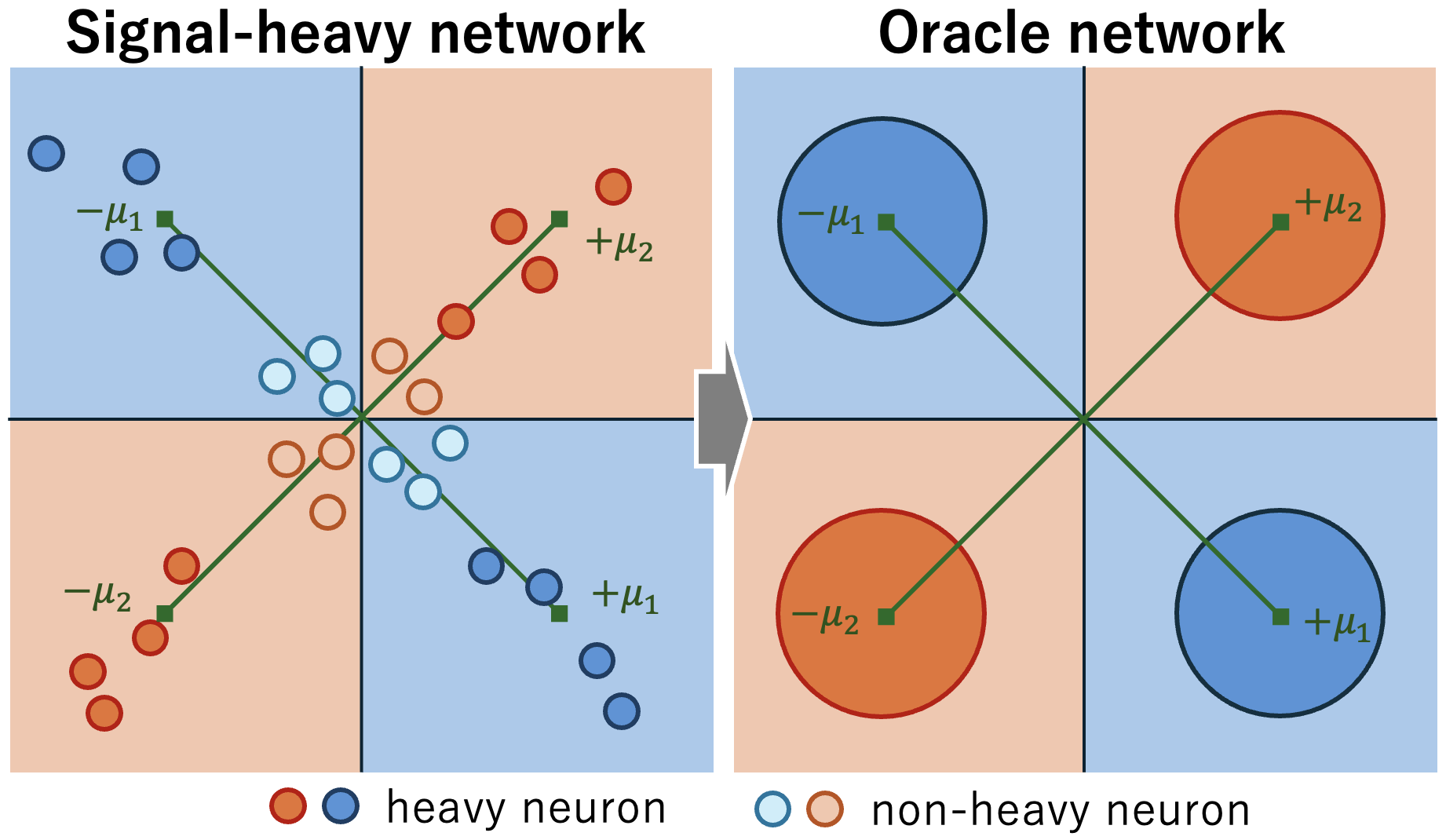}
    \caption{Neuron dynamics in Phase II.}
    \label{fig:phase2}
\end{figure}

\subsection{Conclusion: Proof of Theorem~\ref{thm:main}}

To conclude the proof, we connect the block mass \(N^{(T)}\) with the expected risk 
\(\mathbb{E}_x \ell_{\rho^{(T)}}(x)\). 
The argument proceeds in three steps:  

\begin{enumerate}[label=(\roman*)]
  \setlength{\parskip}{0cm}
  \setlength{\itemsep}{0cm}
\item \textbf{Growth of block mass.}  
From the Phase~II analysis, the network is signal-heavy and the block mass grows at a rate $g_\mu^{(t)}\gtrsim (N^{(t)})^{-3}$. By a contradiction argument, we show that within \(T = O(\eta^{-1} (\log \log d)^4)\) steps, the block mass reaches order \(\log \log d\).  

\item \textbf{Accuracy away from the boundary.}  
We split the input space into points close to the decision boundary, where misclassification is unavoidable, 
and points sufficiently far.  
Using a polar decomposition of the Gaussian input, we prove that the near-boundary region has probability mass 
\(O((\log \log d)^{-1/2})\).  
For points outside this set, the oracle network output has the correct sign with high confidence.  

\item \textbf{Approximation error.}  
Finally, we control the difference between the actual network and the idealized block model.  
By signal-heaviness and the balanced evolution of blocks, this error is negligible compared to the dominant term above.  
\end{enumerate}

\begin{remark}  A key novelty of our proof is the analysis of block dynamics.  
The growth rate of each block is governed by the average margin $g_\mu^{(t)}$.  
Although $g_\mu^{(t)}$ converges to zero, its rate of decay can be related to the evolution of $N^{(t)}$, 
which in turn forces $N^{(t)}$ to diverge to infinity. 
\end{remark}

\section{Numerical experiments}\label{sec:xp}

\begin{figure*}[ht!]
\centering

\begin{subfigure}{0.25\linewidth}
    \includegraphics[width=\textwidth]{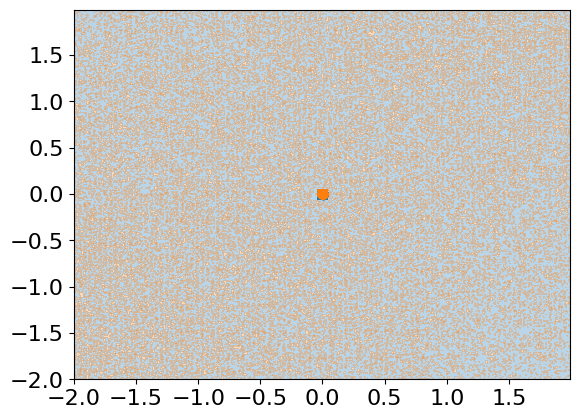}
    \caption{$T=8000$}
\end{subfigure}%
\begin{subfigure}{0.25\linewidth}
    \includegraphics[width=\textwidth]{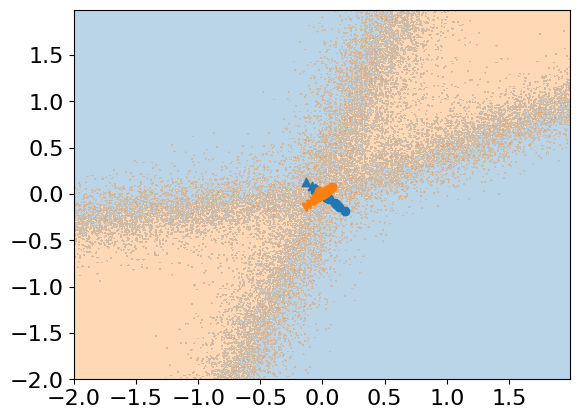}
    \caption{$T=12000$}
\end{subfigure}%
\begin{subfigure}{0.25\linewidth}
    \includegraphics[width=\textwidth]{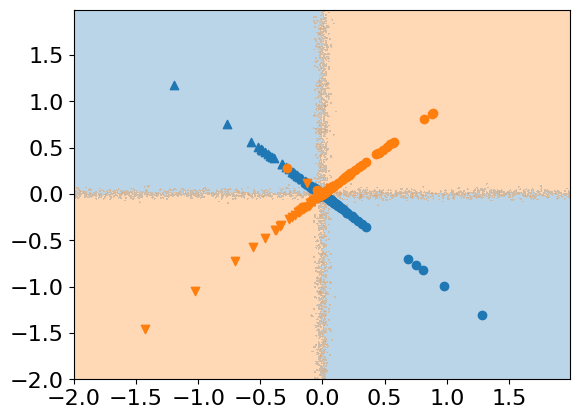}
    \caption{$T=30000$}
\end{subfigure}%
\begin{subfigure}{0.25\linewidth}
    \includegraphics[width=\textwidth]{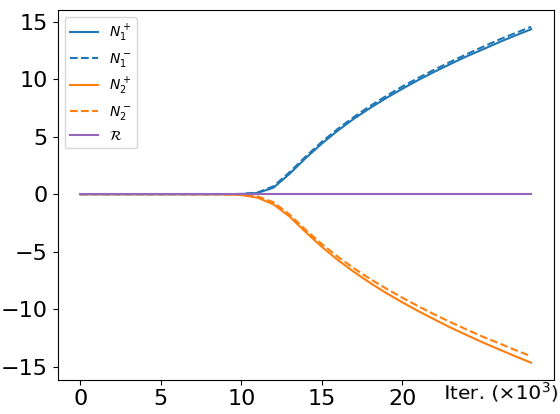}
    \caption{Block growth.}
    \label{fig:xp1weights}
\end{subfigure}

\caption{Simulated neuron dynamics. (a–c) Weight projections (dots; colors indicate assigned labels) and decision boundaries. The weights progressively align with the directions $\pm\mu_i$. (d) The block mass $N_i^\pm$ vs time (purple: residual mass $\calR$). Block masses remain approximately balanced. }
\label{fig:xp1}
\end{figure*}

We conduct numerical experiments to validate our theoretical results, and also examine settings beyond our formal assumptions.

\subsection{Decision Boundary During Training.} 
We track the evolution of weights and the decision boundary for $d=600$, $m=400$, $M=82000$, $\theta=\eta=0.01$.  
Figure~\ref{fig:xp1} shows the projection of weights $w^{(t)}$ on
$\mathrm{span}(e_1,e_2)$ at different training stages (for readability, we keep only neurons with
$\|\wsig^{(0)}\|\ge 2\theta\log d/\sqrt{d}$).  
The background indicates the predicted clusters, and neurons progressively align
with one of the four directions $\pm\mu_1,\pm\mu_2$, sharpening the decision
boundary.  

Figure~\ref{fig:xp1weights} confirms that block masses grow at comparable rates,
consistent with our block-dynamic analysis. The purple curve represents the
residual mass $\calR=\sum_{(w,a)\in\calN}|a|\|\wpe\|$, which remains small.

\subsection{Additional Experiments}
We complement our main experiments with several variants (details in Appendix~\ref{app:xp}). 
Overall, our findings show that training dynamics are robust across different input distributions 
and settings, but can be sensitive to noise and overtraining:

\begin{itemize}[leftmargin=*]
  \setlength{\parskip}{0cm}
  \setlength{\itemsep}{0cm}
    \item \textbf{Input distribution.} Replacing Gaussian inputs with uniform or Gaussian XOR inputs 
    still leads to successful training, provided the distribution is symmetric with respect to the four directions $\pm\mu_1,\pm\mu_2$.
    \item \textbf{Noise sensitivity.} Flipping only $5\%$ of the labels already degrades test performance significantly, 
    indicating a vulnerability of SGD to mislabeled data. However, the decision boundary learned by the network is only slightly biased along the $x$- and $y$-axes.
    \item \textbf{Anisotropy.} With anisotropic covariance, the network learns the decision boundary, 
    but weight evolution becomes disordered compared to the isotropic case. 
    \item \textbf{Nonlinear boundaries.} For highly nonlinear decision functions (e.g., sinusoidal boundaries), 
    the network learns only a linear approximation, and training quickly plateaus. 
\end{itemize}

Taken together, these experiments confirm the key ingredients of our analysis: 
neurons self-organize into balanced blocks, block growth at a similar rate, and the geometry of the zero-margin setting imposes inherent limitations on convergence.

\section{Discussion}\label{sec:concl}

We analyzed the training dynamics of a two-layer network on Gaussian XOR, a canonical \emph{zero-margin} classification problem. 
Our key finding is that although individual neurons evolve heterogeneously, they rapidly self-organize into four balanced blocks whose dynamics reduce to a low-dimensional symmetric system. 
This block-level description yields convergence guarantees without relying on margin assumptions: generalization is governed by \emph{average-case dynamics} near the decision boundary rather than worst-case margins. 
We propose this \emph{block-dynamics} perspective as a new analytical tool for studying feature learning in classification settings where standard separability arguments break down.

Our analysis is limited to the XOR boundary with isotropic Gaussian inputs, chosen for tractability. 
Extending block-dynamic arguments to richer decision boundaries and more general input distributions is a natural direction for future work. 
More broadly, block-level analysis offers a promising framework for understanding the training dynamics of neural networks in challenging classification regimes.

\newpage
\bibliography{references}

@inproceedings{
brutzkus2018sgd,
title={{SGD} Learns Over-parameterized Networks that Provably Generalize on Linearly Separable Data},
author={Alon Brutzkus and Amir Globerson and Eran Malach and Shai Shalev-Shwartz},
booktitle={International Conference on Learning Representations},
year={2018},
url={https://openreview.net/forum?id=rJ33wwxRb},
}

@inproceedings{telgarsky20,
author = {Ji, Ziwei and Telgarsky, Matus},
title = {Directional convergence and alignment in deep learning},
year = {2020},
booktitle = {Proceedings of the 34th International Conference on Neural Information Processing Systems},
articleno = {1441},
numpages = {11},
location = {Vancouver, BC, Canada},
series = {NIPS '20}
}

@InProceedings{pmlr-v235-wang24cn,
  title = 	 {Benign Overfitting in Adversarial Training of Neural Networks},
  author =       {Wang, Yunjuan and Zhang, Kaibo and Arora, Raman},
  booktitle = 	 {Proceedings of the 41st International Conference on Machine Learning},
  pages = 	 {52171--52232},
  year = 	 {2024},
  volume = 	 {235},
  series = 	 {Proceedings of Machine Learning Research},
  month = 	 {21--27 Jul},
  publisher =    {PMLR},
}

@inproceedings{zhu24,
author = {Zhu, Zhenyu and Liu, Fanghui and Chrysos, Grigorios G and Locatello, Francesco and Cevher, Volkan},
title = {Benign overfitting in deep neural networks under lazy training},
year = {2023},
publisher = {JMLR.org},
booktitle = {Proceedings of the 40th International Conference on Machine Learning},
articleno = {1816},
numpages = {24},
location = {Honolulu, Hawaii, USA},
series = {ICML'23}
}

@inproceedings{
kornowski2023from,
title={From Tempered to Benign Overfitting in Re{LU} Neural Networks},
author={Guy Kornowski and Gilad Yehudai and Ohad Shamir},
booktitle={Thirty-seventh Conference on Neural Information Processing Systems},
year={2023}
}

@inproceedings{dandi_reuse,
author = {Dandi, Yatin and Troiani, Emanuele and Arnaboldi, Luca and Pesce, Luca and Zdeborova, Lenka and Krzakala, Florent},
title = {The benefits of reusing batches for gradient descent in two-layer networks: breaking the curse of information and leap exponents},
year = {2024},
publisher = {JMLR.org},
booktitle = {Proceedings of the 41st International Conference on Machine Learning},
articleno = {397},
numpages = {26},
location = {Vienna, Austria},
series = {ICML'24}
}

@inproceedings{
lyu2021gradient,
title={Gradient Descent on Two-layer Nets: Margin Maximization and Simplicity Bias},
author={Kaifeng Lyu and Zhiyuan Li and Runzhe Wang and Sanjeev Arora},
booktitle={Advances in Neural Information Processing Systems},
editor={A. Beygelzimer and Y. Dauphin and P. Liang and J. Wortman Vaughan},
year={2021},
url={https://openreview.net/forum?id=Aa5oPXc_1IV}
}

@inproceedings{
abbe2025learning,
title={Learning High-Degree Parities: The Crucial Role of the Initialization},
author={Emmanuel Abbe and Elisabetta Cornacchia and Jan H{\k{a}}z{\l}a and Donald Kougang-Yombi},
booktitle={The Thirteenth International Conference on Learning Representations},
year={2025}
}

@inproceedings{benignViT24,
 author = {Jiang, Jiarui and Huang, Wei and Zhang, Miao and Suzuki, Taiji and Nie, Liqiang},
 booktitle = {Advances in Neural Information Processing Systems},
 editor = {A. Globerson and L. Mackey and D. Belgrave and A. Fan and U. Paquet and J. Tomczak and C. Zhang},
 pages = {135464--135625},
 publisher = {Curran Associates, Inc.},
 title = {Unveil Benign Overfitting for Transformer in Vision: Training Dynamics, Convergence, and Generalization},
 volume = {37},
 year = {2024}
}

@inproceedings{
shi2023provable,
title={Provable Guarantees for Neural Networks via Gradient Feature Learning},
author={Zhenmei Shi and Junyi Wei and Yingyu Liang},
booktitle={Thirty-seventh Conference on Neural Information Processing Systems},
year={2023},
url={https://openreview.net/forum?id=5F04bU79eK}
}

@misc{bruna2025surveyalgorithmsmultiindexmodels,
      title={Survey on Algorithms for multi-index models}, 
      author={Joan Bruna and Daniel Hsu},
      year={2025},
      eprint={2504.05426},
      archivePrefix={arXiv},
      primaryClass={stat.ML},
      url={https://arxiv.org/abs/2504.05426}, 
}

@inproceedings{
huang2025quantifying,
title={Quantifying the Optimization and Generalization Advantages of Graph Neural Networks Over Multilayer Perceptrons},
author={Wei Huang and Yuan Cao and Haonan Wang and Xin Cao and Taiji Suzuki},
booktitle={The 28th International Conference on Artificial Intelligence and Statistics},
year={2025}
}

@inproceedings{benignCNN,
author = {Cao, Yuan and Chen, Zixiang and Belkin, Mikhail and Gu, Quanquan},
title = {Benign overfitting in two-layer convolutional neural networks},
year = {2022},
booktitle = {Proceedings of the 36th International Conference on Neural Information Processing Systems},
articleno = {1830},
numpages = {14},
location = {New Orleans, LA, USA},
series = {NIPS '22}
}

@article{osti_10344164,
place = {Country unknown/Code not available}, title = {Benign Overfitting without Linearity: Neural Network Classifiers Trained by Gradient Descent for Noisy Linear Data}, abstractNote = {}, journal = {Proceedings of the 35th Conference on Learning Theory (COLT2022)}, author = {Frei, Spencer and Chatterji, Niladri and Bartlett, Peter L.}, year="2022"}

@inproceedings{
shi2022a,
title={A Theoretical Analysis on Feature Learning in Neural Networks: Emergence from Inputs and Advantage over Fixed Features},
author={Zhenmei Shi and Junyi Wei and Yingyu Liang},
booktitle={International Conference on Learning Representations},
year={2022}
}

@InProceedings{pmlr-v202-kou23a,
  title = 	 {Benign Overfitting in Two-layer {R}e{LU} Convolutional Neural Networks},
  author =       {Kou, Yiwen and Chen, Zixiang and Chen, Yuanzhou and Gu, Quanquan},
  booktitle = 	 {Proceedings of the 40th International Conference on Machine Learning},
  pages = 	 {17615--17659},
  year = 	 {2023},
  editor = 	 {Krause, Andreas and Brunskill, Emma and Cho, Kyunghyun and Engelhardt, Barbara and Sabato, Sivan and Scarlett, Jonathan},
  volume = 	 {202},
  series = 	 {Proceedings of Machine Learning Research},
  month = 	 {23--29 Jul},
  publisher =    {PMLR}
}

@book{Vershynin_2018, place={Cambridge}, series={Cambridge Series in Statistical and Probabilistic Mathematics}, title={High-Dimensional Probability: An Introduction with Applications in Data Science}, publisher={Cambridge University Press}, author={Vershynin, Roman}, year={2018}, collection={Cambridge Series in Statistical and Probabilistic Mathematics}}

@inproceedings{suzuki2023feature,
    title={Feature learning via mean-field Langevin dynamics: classifying sparse parities and beyond},
    author={Taiji Suzuki and Denny Wu and Kazusato Oko and Atsushi Nitanda},
    booktitle={Thirty-seventh Conference on Neural Information Processing Systems},
    year={2023}
}

@inproceedings{
mahankali2023ntk,
title={Beyond {NTK} with Vanilla Gradient Descent: A Mean-Field Analysis of Neural Networks with Polynomial Width, Samples, and Time},
author={Arvind Venkat Mahankali and Jeff Z. HaoChen and Kefan Dong and Margalit Glasgow and Tengyu Ma},
booktitle={Thirty-seventh Conference on Neural Information Processing Systems},
year={2023}
}

@inproceedings{Mei2019MeanfieldTO,
  title={Mean-field theory of two-layers neural networks: dimension-free bounds and kernel limit},
  author={Song Mei and Theodor Misiakiewicz and Andrea Montanari},
  booktitle={Annual Conference Computational Learning Theory},
  year={2019}
}

@inproceedings{jacotNTK,
 author = {Jacot, Arthur and Gabriel, Franck and Hongler, Clement},
 booktitle = {Advances in Neural Information Processing Systems},
 pages = {},
 title = {Neural Tangent Kernel: Convergence and Generalization in Neural Networks},
 volume = {31},
 year = {2018}
}

@inproceedings{Chizat2018OnLT,
  title={On Lazy Training in Differentiable Programming},
  author={L{\'e}na{\"i}c Chizat and Edouard Oyallon and Francis R. Bach},
  booktitle={Neural Information Processing Systems},
  year={2018}
}

@article{montanari19,
author = {Ghorbani, Behrooz and Mei, Song and Misiakiewicz, Theodor and Montanari, Andrea},
year = {2021},
month = {04},
pages = {},
title = {Linearized two-layers neural networks in high dimension},
volume = {49},
journal = {The Annals of Statistics}
}

@article{bietti2023learning,
  title={On learning Gaussian multi-index models with gradient flow part I: General properties and two-timescale learning},
  author={Bietti, Alberto and Bruna, Joan and Pillaud-Vivien, Loucas},
  journal={Communications on Pure and Applied Mathematics},
  year={2025},
  publisher={Wiley Online Library}
}

@inproceedings{bietti2022learning,
    title={Learning single-index models with shallow neural networks},
    author={Alberto Bietti and Joan Bruna and Clayton Sanford and Min Jae Song},
    booktitle={Advances in Neural Information Processing Systems},
    year={2022}
}

@inproceedings{AbbeAM23,
  author       = {Emmanuel Abbe and
                  Enric Boix Adser{\`{a}} and
                  Theodor Misiakiewicz},
  title        = {{SGD} learning on neural networks: leap complexity and saddle-to-saddle
                  dynamics},
  booktitle    = {The Thirty Sixth Annual Conference on Learning Theory, {COLT} }, 
  volume       = {195},
  pages        = {2552--2623},
  year         = {2023}
}

@InProceedings{oko24a,
  title = 	 {Learning sum of diverse features: computational hardness and efficient gradient-based training for ridge combinations},
  author =       {Oko, Kazusato and Song, Yujin and Suzuki, Taiji and Wu, Denny},
  booktitle = 	 {Proceedings of Thirty Seventh Conference on Learning Theory},
  pages = 	 {4009--4081},
  year = 	 {2024},
  volume = 	 {247},
  series = 	 {Proceedings of Machine Learning Research},
  month = 	 {30 Jun--03 Jul},
  publisher =    {PMLR}
}

@inproceedings{
    mousavi-hosseini2023neural,
    title={Neural Networks Efficiently Learn Low-Dimensional Representations with {SGD}},
    author={Alireza Mousavi-Hosseini and Sejun Park and Manuela Girotti and Ioannis Mitliagkas and Murat A Erdogdu},
    booktitle={The Eleventh International Conference on Learning Representations },
    year={2023}
}

@inproceedings{glasgow2023sgd,
      title={SGD Finds then Tunes Features in Two-Layer Neural Networks with near-Optimal Sample Complexity: A Case Study in the XOR problem}, 
      author={Margalit Glasgow},
    booktitle={International Conference on Learning Representations.},
      year={2024}
}

@ARTICLE{Berthier2024-pf,
  title     = "Learning time-scales in two-layers neural networks",
  author    = "Berthier, Rapha{\"e}l and Montanari, Andrea and Zhou, Kangjie",
  journal   = "Found. Comput. Math.",
  publisher = "Springer Science and Business Media LLC",
  month     =  aug,
  year      =  2024
}

@inproceedings{
kou2024matching,
title={Matching the Statistical Query Lower Bound for $k$-Sparse Parity Problems with Sign Stochastic Gradient Descent},
author={Yiwen Kou and Zixiang Chen and Quanquan Gu and Sham M. Kakade},
booktitle={The Thirty-eighth Annual Conference on Neural Information Processing Systems},
year={2024}
}

@inproceedings{
xu2023benign,
title={Benign Overfitting and Grokking in Re{LU} Networks for {XOR} Cluster Data},
author={Zhiwei Xu and Yutong Wang and Spencer Frei and Gal Vardi and Wei Hu},
booktitle={The Twelfth International Conference on Learning Representations},
year={2024}
}

@inproceedings{
barak2022hidden,
title={Hidden Progress in Deep Learning: {SGD} Learns Parities Near the Computational Limit},
author={Boaz Barak and Benjamin L. Edelman and Surbhi Goel and Sham M. Kakade and Eran Malach and Cyril Zhang},
booktitle={Advances in Neural Information Processing Systems},
year={2022}
}

@inproceedings{shalev17,
  title={Failures of gradient-based deep learning},
  author={Shalev-Shwartz, Shai and Shamir, Ohad and Shammah, Shaked},
  booktitle={International Conference on Machine Learning},
  pages={3067--3075},
  year={2017},
  organization={PMLR}
}

@inproceedings{Wei2018RegularizationMG,
  title={Regularization Matters: Generalization and Optimization of Neural Nets v.s. their Induced Kernel},
  author={Colin Wei and J. Lee and Qiang Liu and Tengyu Ma},
  booktitle={Neural Information Processing Systems},
  year={2018}
}

@ARTICLE{mnist_dataset,
  author={Deng, Li},
  journal={IEEE Signal Processing Magazine}, 
  title={The MNIST Database of Handwritten Digit Images for Machine Learning Research [Best of the Web]}, 
  year={2012},
  volume={29},
  number={6},
  pages={141-142}}

@inproceedings{
telgarsky2023feature,
title={Feature selection and low test error in shallow low-rotation Re{LU} networks},
author={Matus Telgarsky},
booktitle={The Eleventh International Conference on Learning Representations },
year={2023}
}

\section*{Checklist}
The checklist follows the references. For each question, choose your answer from the three possible options: Yes, No, Not Applicable.  You are encouraged to include a justification to your answer, either by referencing the appropriate section of your paper or providing a brief inline description (1-2 sentences). 
Please do not modify the questions.  Note that the Checklist section does not count towards the page limit. Not including the checklist in the first submission won't result in desk rejection, although in such case we will ask you to upload it during the author response period and include it in camera ready (if accepted).

\begin{enumerate}

  \item For all models and algorithms presented, check if you include:
  \begin{enumerate}
    \item A clear description of the mathematical setting, assumptions, algorithm, and/or model. [Yes]
    \item An analysis of the properties and complexity (time, space, sample size) of any algorithm. [Yes]
    \item (Optional) Anonymized source code, with specification of all dependencies, including external libraries. [Yes]
  \end{enumerate}

  \item For any theoretical claim, check if you include:
  \begin{enumerate}
    \item Statements of the full set of assumptions of all theoretical results. [Yes]
    \item Complete proofs of all theoretical results. [Yes]
    \item Clear explanations of any assumptions. [Yes]     
  \end{enumerate}

  \item For all figures and tables that present empirical results, check if you include:
  \begin{enumerate}
    \item The code, data, and instructions needed to reproduce the main experimental results (either in the supplemental material or as a URL). [Yes]
    \item All the training details (e.g., data splits, hyperparameters, how they were chosen). [Yes]
    \item A clear definition of the specific measure or statistics and error bars (e.g., with respect to the random seed after running experiments multiple times). [Yes]
    \item A description of the computing infrastructure used. (e.g., type of GPUs, internal cluster, or cloud provider). [Yes]
  \end{enumerate}

  \item If you are using existing assets (e.g., code, data, models) or curating/releasing new assets, check if you include:
  \begin{enumerate}
    \item Citations of the creator If your work uses existing assets. [Not Applicable]
    \item The license information of the assets, if applicable. [Not Applicable]
    \item New assets either in the supplemental material or as a URL, if applicable. [Not Applicable]
    \item Information about consent from data providers/curators. [Not Applicable]
    \item Discussion of sensible content if applicable, e.g., personally identifiable information or offensive content. [Not Applicable]
  \end{enumerate}

  \item If you used crowdsourcing or conducted research with human subjects, check if you include:
  \begin{enumerate}
    \item The full text of instructions given to participants and screenshots. [Not Applicable]
    \item Descriptions of potential participant risks, with links to Institutional Review Board (IRB) approvals if applicable. [Not Applicable]
    \item The estimated hourly wage paid to participants and the total amount spent on participant compensation. [Not Applicable]
  \end{enumerate}

\end{enumerate}

\clearpage
\appendix
\thispagestyle{empty}

\onecolumn
\aistatstitle{Supplementary Materials}

\addtocontents{toc}{\protect\setcounter{tocdepth}{2}} %

{\renewcommand{\contentsname}{}%
 \tableofcontents}

\section{Preliminary Lemmas}\label{sec:app:preliminary}
We will first introduce some technical results that will be used repeatedly in the proof of Lemma \ref{lem:phase1}.

\subsection{Concentration of Empirical Gradients}
The following lemma bounds the difference between the population gradient and its empirical counterpart. It extends Lemma B.12 in \cite{glasgow2023sgd} to the Gaussian setting.

\begin{lemma}[Empirical concentration of the gradients]\label{lem:conc_grad}
Assume that the input data are as described in Section~\ref{sec:framework}, and that the neural network width satisfies \( m \leq d^c \) for some constant \( c > 0 \). Let \(\batch\) denote the batch size, i.e., \(|M_t|\). Recall that \( L = \mathbb{E}_x [\ell_\rho(x)] \) and \(\hat{L} = \frac{1}{|\batch|} \sum_{j \in \batch} \ell_\rho(x^{(j)})\) for some loss function \(\ell_\rho\).

Then, for any loss function \(\ell_\rho\) that is differentiable and \(2\)-Lipschitz, there exists a constant \( C > 0 \) such that with probability at least \( 1 - d^{-\Omega(1)} \), for all neurons \((w, a) \in \mathcal{N}\), we have
\begin{enumerate}
  \setlength{\parskip}{0cm}
  \setlength{\itemsep}{0cm}
    \item \(\|\nabla_w L - \nabla_w \hat{L}\|^2 \leq C^2 a^2 \frac{d \log d}{\batch}\),
    \item \(\|\nabla_{w_i} L - \nabla_{w_i} \hat{L}\|^2 \leq C^2 a^2 \frac{\log d}{\batch}\), where \( w_i = \langle w, e_i \rangle \),
    \item \(\left| \nabla_a L - \nabla_a \hat{L} \right|^2 \leq C^2 \|w\|^2 \frac{\log d}{\batch}\).
\end{enumerate}
\end{lemma}

\begin{proof}
To simplify the notation, we write \(\ell\) instead of \(\ell_\rho\).
By the Lipschitz assumption, we have \(\|\ell'\|_\infty \leq 2\). This implies, by the Gaussianity of \(x \sim \mathcal{N}(0, I_d)\), that for all \(u \in \mathbb{S}^{d-1}\),
\[
\mathbb{P}\left( \left| \ell'(x) \sigma'(w^\top x) \langle x, u \rangle \right| > t \right) \leq \mathbb{P}\left( 2 \left| \langle x, u \rangle \right| > t \right) \leq 2e^{-t^2/8},
\]
since for any random variables \(X, Y\) and any \(t \in \mathbb{R}\), the inequality \(Y \geq X\) implies \(\{X \geq t\} \subset \{Y \geq t\}\).
As a consequence, with the choice \(u = e_i\), we obtain that
\[
a^{-1} \langle \nabla_w L - \nabla_w \hat{L}, e_i \rangle = \sum_{j \in M_t} \left( \ell'\big(x^{(j)}\big) \sigma'\big(w^\top x^{(j)}\big) x^{(j)}_i - \mathbb{E}_x \left[ \ell'(x) \sigma'(w^\top x) x_i \right] \right)
\]
is sub-Gaussian with parameter \(2\sqrt{2}\batch\), as a sum of \(\batch\) independent centered \(2\sqrt{2}\)-sub-Gaussian random variables.

Hence, we have
\[
\mathbb{P}\left( \left| a^{-1} \langle \nabla_w L - \nabla_w \hat{L}, e_i \rangle \right| > t \right) \leq e^{-t^2/(8\batch)}.
\]
By choosing \(t = C \sqrt{\batch \log d} \) for a constant \(C > 0\) large enough, and taking a union bound over \(i = 1, \dots, d\), we obtain that with probability at least \(1 - d^{-\Omega(1)}\),
\[
\|\nabla_{w_i} L - \nabla_{w_i} \hat{L}\|^2 = \langle \nabla_w L - \nabla_w \hat{L}, e_i \rangle^2 \leq C^2 a^2 \frac{\log d}{\batch},
\]
and
\[
\|\nabla_w L - \nabla_w \hat{L}\|^2 = \sum_{i=1}^d \langle \nabla_w L - \nabla_w \hat{L}, e_i \rangle^2 \leq C^2 a^2 \frac{d \log d}{\batch}.
\]

A similar argument leads to the third statement of the lemma by noticing that \(\ell'(x) \sigma(w^\top x) = \ell'(x) \sigma'(w^\top x) w^\top x\) holds, and that \(w^\top x\) is Gaussian with variance \(\|w\|^2\). We conclude by taking a union bound over all the \(m\) neurons.
\end{proof}

\begin{remark}
Since the batch size is chosen such that \(|\batch| \asymp d \log^\beta d\) for some  \(\beta > 1\), the constant \(C\) appearing in Lemma~\ref{lem:conc_grad} can be absorbed into the asymptotic notation and thus ignored.
\end{remark}

\subsection{Basic Properties of the Population Loss Gradients}
We describe several basic properties of the population loss gradients.
To obtain the results, Lemma~B.13 in \cite{glasgow2023sgd} can be easily adapted to the Gaussian setting. For completeness, we detail the proof.

\begin{lemma}\label{lemma:layer_balance}
Assume that the neural network width satisfies \( m \leq d^c \) for some constant \( c > 0 \), the learning rate satisfies \( \eta < \frac{\sqrt{\pi}}{4\sqrt{2}} \), and the batch size satisfies \(\batch = |M_t| \geq d \log^2 d\). Then, with probability at least \( 1 - d^{-\Omega(1)} \), for all neurons \((w,a) \in \mathcal{N}\), we have
\begin{enumerate}
  \setlength{\parskip}{0cm}
  \setlength{\itemsep}{0cm}
    \item \(\|\nabla_w L_\rho\| \leq 2 \sqrt{\frac{2}{\pi}} |a|\),
    \item \(\left| \nabla_a L_\rho \right| \leq (2 + o(1)) \sqrt{\frac{2}{\pi}} \|w\|\),
    \item if \(|a^{(t)}| \leq \|w^{(t)}\|\), then \(|a^{(t+1)}| \leq \|w^{(t+1)}\|\),
    \item \(\|w^{(t+1)}\|^2 - |a^{(t+1)}|^2 \leq \left( \frac{16}{\pi} + o(1) \right) \eta^2 |a^{(t)}|^2 + \|w^{(t)}\|^2 - |a^{(t)}|^2\).
\end{enumerate}
\end{lemma}

\begin{proof}
We prove each statement in turn.

For the first statement, we have
\begin{align*}
    \frac{1}{|a|}\|\nabla_w L_\rho\| &= \frac{1}{|a|} \sup_{v \in \mathbb{S}^{d-1}} \langle v, \nabla_w L_\rho \rangle \\
    &= \sup_{v \in \mathbb{S}^{d-1}} \mathbb{E}_x \left[ \ell'_{\rho}(x) \sigma'(w^\top x) x^\top v \right] \\
    &\leq \sup_{v \in \mathbb{S}^{d-1}} \mathbb{E}_x \left[ |\ell'_{\rho}(x)| \, |x^\top v| \right] \\
    &\leq \sup_{v \in \mathbb{S}^{d-1}} 2 \mathbb{E}_x \left[ |x^\top v| \right] \\
    &= 2 \sqrt{\frac{2}{\pi}},
\end{align*}
where we used that \( |\ell'_{\rho}(x)| \leq 2 \) and that \( \mathbb{E}[|x^\top v|] = \sqrt{2/\pi} \) for \( x \sim \mathcal{N}(0, I_d) \).

For the second statement, similarly, we obtain
\begin{align*}
    |\nabla_a L_\rho| &= \left| \mathbb{E}_x \left[ \ell'_{\rho}(x) \sigma(w^\top x) \right] \right| \\
    &\leq \mathbb{E}_x \left[ |\ell'_{\rho}(x)| \, |\sigma(w^\top x)| \right] \\
    &\leq 2 \mathbb{E}_x \left[ |w^\top x| \right] \\
    &= 2 \sqrt{\frac{2}{\pi}} \|w\|.
\end{align*}
Here we used that for ReLU, \(\sigma(w^\top x) = (w^\top x)_+\) and \(\sigma(w^\top x) \leq w^\top x\) when \(w^\top x \geq 0\).

Now we show the third statement.  
Since \(\sigma\) is ReLU, we have \(\sigma'(w^\top x) w^\top x = \sigma(w^\top x)\) almost everywhere.  
Combined with the chain rule, this leads to the identity:
\begin{equation}\label{eq:id_aw}
    (w^{(t)})^\top \nabla_{w^{(t)}} \hat{L}_{\rho^{(t)}} = a^{(t)} \nabla_{a^{(t)}} \hat{L}_{\rho^{(t)}}.
\end{equation}

Expanding the update equations, we have
\begin{align*}
    (a^{(t+1)})^2 &= \left( a^{(t)} - \eta \nabla_{a^{(t)}} \hat{L}_{\rho^{(t)}} \right)^2 \\
    &= (a^{(t)})^2 - 2\eta a^{(t)} \nabla_{a^{(t)}} \hat{L}_{\rho^{(t)}} + \eta^2 \left( \nabla_{a^{(t)}} \hat{L}_{\rho^{(t)}} \right)^2,
\end{align*}
and
\begin{align*}
    \|w^{(t+1)}\|^2 &= \|w^{(t)} - \eta \nabla_{w^{(t)}} \hat{L}_{\rho^{(t)}}\|^2 \\
    &= \|w^{(t)}\|^2 - 2\eta (w^{(t)})^\top \nabla_{w^{(t)}} \hat{L}_{\rho^{(t)}} + \eta^2 \|\nabla_{w^{(t)}} \hat{L}_{\rho^{(t)}}\|^2.
\end{align*}

Using \eqref{eq:id_aw}, we obtain
\begin{align*}
    (a^{(t+1)})^2 - \|w^{(t+1)}\|^2 
    &= (a^{(t)})^2 - \|w^{(t)}\|^2 + \eta^2 \left( (\nabla_{a^{(t)}} \hat{L}_{\rho^{(t)}})^2 - \|\nabla_{w^{(t)}} \hat{L}_{\rho^{(t)}}\|^2 \right) \\
    &\leq (a^{(t)})^2 - \|w^{(t)}\|^2 + \eta^2 \left( (\nabla_{a^{(t)}} \hat{L}_{\rho^{(t)}})^2 - \frac{1}{\|w^{(t)}\|^2} \left( (w^{(t)})^\top \nabla_{w^{(t)}} \hat{L}_{\rho^{(t)}} \right)^2 \right) \\
    &= (a^{(t)})^2 - \|w^{(t)}\|^2 + \frac{\eta^2 (\nabla_{a^{(t)}} \hat{L}_{\rho^{(t)}})^2}{\|w^{(t)}\|^2} \left( \|w^{(t)}\|^2 - (a^{(t)})^2 \right) \\
    &= \left( (a^{(t)})^2 - \|w^{(t)}\|^2 \right) \left( 1 - \frac{\eta^2 (\nabla_{a^{(t)}} \hat{L}_{\rho^{(t)}})^2}{\|w^{(t)}\|^2} \right).
\end{align*}

By the choice of the batch size and Lemma~\ref{lem:conc_grad}, we have with probability at least \(1 - d^{-\Omega(1)}\),
\begin{align*}
    |\nabla_a \hat{L}_\rho| &\leq |\nabla_a L_\rho| + |\nabla_a L_\rho - \nabla_a \hat{L}_\rho| \\
    &\leq 2 \sqrt{\frac{2}{\pi}} \|w\| + C \|w\| \sqrt{\frac{d \log^2 d}{\batch}} \\
    &\leq (2 + o(1)) \sqrt{\frac{2}{\pi}} \|w\|,
\end{align*}
where the last inequality uses that \(\batch \geq d \log^2 d\).

Thus, if \(|a^{(t)}| \leq \|w^{(t)}\|\) initially, the same inequality holds at step \(t+1\).

Finally, for the fourth statement, by again using \eqref{eq:id_aw}, we have
\begin{align*}
    \|w^{(t+1)}\|^2 - (a^{(t+1)})^2 - \left( \|w^{(t)}\|^2 - (a^{(t)})^2 \right)
    &= \eta^2 \left( \|\nabla_w \hat{L}_{\rho^{(t)}}\|^2 - (\nabla_a \hat{L}_{\rho^{(t)}})^2 \right) \\
    &\leq \eta^2 \|\nabla_w \hat{L}_{\rho^{(t)}}\|^2 \\
    &\leq 2\eta^2 \left( \|\nabla_w L_{\rho^{(t)}}\|^2 + \|\nabla_w \hat{L}_{\rho^{(t)}} - \nabla_w L_{\rho^{(t)}}\|^2 \right) \\
    &\leq \left( \frac{16}{\pi} + o(1) \right) \eta^2 a^2,
\end{align*}
with probability at least \(1 - d^{-\Omega(1)}\), which completes the proof.
\end{proof}

\subsection{Other Technical Lemmas}

In this section, we collect standard technical estimates for Gaussian random variables that will be used throughout the proofs.

\begin{lemma}\label{lem:gauss_approx}
Let \(X \sim \mathcal{N}(0,1)\). For all \(\epsilon \in (0,1)\), we have
\[
\sqrt{\frac{2}{\pi}} e^{-\epsilon^2/2} \left( \epsilon + \frac{\epsilon^3}{3} \right) \leq \mathbb{P}(|X| \leq \epsilon) \leq \sqrt{\frac{2}{\pi}} \epsilon.
\]
In particular, when \(\epsilon = o(1)\), we have
\[
\mathbb{P}(|X| \leq \epsilon) = \sqrt{\frac{2}{\pi}} \epsilon + O(\epsilon^3).
\]
\end{lemma}

\begin{proof}
The upper bound can be directly obtained by bounding \(e^{-t^2/2}\) by \(1\) over \((- \epsilon, \epsilon)\).

For the lower bound, consider the function
\[
\Omega(x) = e^{x^2/2} \int_0^x e^{-t^2/2} \, \mathrm{d}t.
\]
It satisfies the differential equation \(\Omega'(x) = 1 + x \Omega(x)\) with initial conditions \(\Omega(0) = 0\) and \(\Omega'(0) = 1\). In particular, it implies that the coefficients \(a_n\) of its Taylor expansion around \(0\) satisfy the relation \(a_{n+2} = \frac{a_n}{n+2}\).

Consequently,
\begin{align*}
\frac{1}{\sqrt{2\pi}} \int_0^x e^{-t^2/2} \, \mathrm{d}t
&= \frac{1}{\sqrt{2\pi}} e^{-x^2/2} \Omega(x) \\
&= \frac{1}{\sqrt{2\pi}} e^{-x^2/2} \sum_{k=0}^\infty \frac{x^{2k+1}}{(2k+1)!!}.
\end{align*}
The result follows by substituting \(x = \epsilon\).
\end{proof}

We recall the classical bounds for the upper tail of a Gaussian random variable:

\begin{lemma}\label{lem:gauss_approx_tail}[Mill's ratio bounds]
Let \(X \sim \mathcal{N}(0,1)\). For all \(t > 0\), we have
\[
\frac{t}{1 + t^2} \frac{1}{\sqrt{2\pi}} e^{-t^2/2} \leq \mathbb{P}(X \geq t) \leq \frac{1}{t \sqrt{2\pi}} e^{-t^2/2}.
\]
\end{lemma}
\begin{proof}
The upper bound follows from an integration by parts, using that
\[
\mathbb{P}(X \geq t) = \int_t^{+\infty} \frac{1}{\sqrt{2\pi}} e^{-s^2/2} \, \mathrm{d}s \leq \frac{1}{t} \frac{1}{\sqrt{2\pi}} e^{-t^2/2}.
\]
The lower bound follows from a refined integration by parts argument and is a standard variant of Mill's ratio; see, e.g., \cite[Proposition 2.1.2]{Vershynin_2018}.
\end{proof}

Finally, we recall a useful bound for the Laplace transform of the folded Gaussian distribution:

\begin{lemma}\label{lem:folded_gauss}
Let \(X \sim \mathcal{N}(0,1)\). For all \(t > 0\), we have
\[
\mathbb{E}\left(e^{-t|X|}\right) \leq \frac{\sqrt{2}}{\sqrt{\pi} t}.
\]
\end{lemma}

\begin{proof}
It follows from the identity
\[
\mathbb{E}\left(e^{-t|X|}\right) = 2 e^{t^2/2} \mathbb{P}(X \geq t),
\]
combined with the classical upper bound on \(\mathbb{P}(X \geq t)\) from Lemma~\ref{lem:gauss_approx_tail}.
\end{proof}

\begin{lemma}\label{lem:sq_gauss}
Let \( X \sim \mathcal{N}(0,1) \) be a standard normal random variable. Then, for all \( t > 0 \), we have
\[
\mathbb{E}\left(e^{-t^2 X^2}\right) = \frac{1}{\sqrt{2(t^2 + 1/2)}}.
\]

\end{lemma}

\begin{proof}
By definition of the expectation under the standard normal distribution, we have
\[
\mathbb{E}\left(e^{-t^2 X^2}\right) = \int_{-\infty}^{\infty} e^{-t^2 x^2} \cdot \frac{1}{\sqrt{2\pi}} e^{-x^2/2} \, dx.
\]
Combining the exponential terms gives
\[
\mathbb{E}\left(e^{-t^2 X^2}\right) = \frac{1}{\sqrt{2\pi}} \int_{-\infty}^{\infty} e^{-x^2(t^2 + 1/2)} \, dx.
\]
This is a standard Gaussian integral of the form
\[
\int_{-\infty}^{\infty} e^{-a x^2} \, dx = \sqrt{\frac{\pi}{a}} \quad \text{for } a > 0.
\]
Applying this with \( a = t^2 + 1/2 \), we get
\[
\mathbb{E}\left(e^{-t^2 X^2}\right) = \frac{1}{\sqrt{2\pi}} \cdot \sqrt{\frac{\pi}{t^2 + 1/2}} = \frac{1}{\sqrt{2(t^2 + 1/2)}}.
\]
For the asymptotic behavior as \( t \to \infty \), note that
\[
\frac{1}{\sqrt{2(t^2 + 1/2)}} \sim \frac{1}{\sqrt{2} \cdot t}.
\]
\end{proof}

\section{Analysis of Phase I}\label{sec:app:phase1} At the beginning of Phase I, \(\|\wsig\|\) is much smaller than \(\|\wpe\|\), but after several iterations, the signal magnitude exceeds the global noise level. As our computations show, the SGD dynamics depend on the relative magnitude between \(\|\wsig\|\) and \(\|\wpe\|\), leading us to subdivide Phase I into two subphases: Phase Ia, where \(\|\wsig\| \ll \|\wpe\|\), and Phase Ib, where \(\|\wsig\| \gtrsim \|\wpe\|\).

\subsection{Evaluation of \(\nabla L_0\) during Phase Ia}

At initialization, typical neurons satisfy \(\| \wsig^{(0)} \| \approx d^{-1/2} \theta \ll \| \wpe^{(0)} \| \approx \theta\). As long as \(\| \wsig \| = o(\| \wpe \|)\) holds, the following lemma characterizes the behavior of the population gradients.
\begin{lemma}[\(L_0\) Population Gradients]\label{lemma:pop1}
For any neuron \((w,a) \in \mathcal{N}\), we have
\begin{enumerate}
  \setlength{\parskip}{0cm}
  \setlength{\itemsep}{0cm}
    \item \label{Psig} 
    \[
    -\wsig^\top \nabla_{w} L_0
    = \frac{2}{\pi\sqrt{2\pi}}\; |a|\; \frac{ \| \wsig \|^2}{ \| \wpe \|}
      + O\!\left( |a| \frac{ \| w_{1:2} \|^3 \| \wsig \| }{ \| \wpe \|^3 } \right),
    \]
    \item \label{Popp}
    \[
    -\wopp^\top \nabla_{w} L_0
    = -\frac{2}{\pi\sqrt{2\pi}}\; |a|\; \frac{ \| \wopp \|^2}{ \| \wpe \|}
      + O\!\left( |a| \frac{ \| w_{1:2} \|^3 \| \wopp \| }{ \| \wpe \|^3 } \right),
    \]
    \item \label{Pwpe}
    \[
    \operatorname{sgn}\left( -\wpe^\top \nabla_{w} L_0 \right) = 
    \begin{cases}
    1 & \text{if } \|\wopp\| > \|\wsig\|, \\
    -1 & \text{if } \|\wopp\| \leq \|\wsig\|,
    \end{cases}
    \]
    and
    \[
    \left| \wpe^\top \nabla_{w} L_0 \right| \leq 0.5\, |a| \, \| \wpe \|.
    \]
    \item For \(i \geq 3\), we have
    \[
    -\,w_i \, \nabla_{w_i} L_0
    = \frac{a}{2}\, \mathbb{E}_x \!\left[ y(z)\, \mathbf{1}\!\left( |w^\top z + w^\top \xii| \leq |w_i \xi_i| \right)\, |w_i \xi_i| \right].
    \]
    Furthermore, for neurons such that \(\| \wpe \|_\infty \ll \| \wpe \|\) and \(\| w_{1:2} \| = o(1)\), we have
    \[
    \left| w_i\, \nabla_{w_i} L_0 \right|
    \leq \frac{4 |a w_i|}{\| \wpe \|} \frac{ \| w_{1:2} \|^2 + |w_i| \| w_{1:2} \| }{ \| \wpe \|^2 }
    = o\!\left( \frac{ |a w_i| }{ \| \wpe \| } \right).
    \]
\end{enumerate}
\end{lemma}

\begin{proof}
We analyze each statement separately. Throughout the proof, write \(x=z+\xi\), where \(z\) is the projection of \(x\) onto \(\mathrm{span}\{e_1,e_2\}\) and \(\xi\) is the orthogonal noise.

\paragraph{Proof of the first and second statements.}
Using the symmetrization trick and ReLU’s \(\sigma'(u)=\mathbf 1_{\{u>0\}}\),
\begin{align}
    -\nabla_{w_{1:2}} L_0 
    &= \frac{a}{2}\, \mathbb{E}_z \mathbb{E}_\xi \!\left[ y \left( \sigma'(w^\top \xi + w^\top z) - \sigma'(w^\top \xi - w^\top z) \right) z \right] \nonumber \\
    &= \frac{a}{2}\, \mathbb{E}_z \!\left[ y \, \sgn(w^\top z)\, z \; \mathbb{P}_\xi\!\left( |w^\top \xi| \leq |w^\top z| \right) \right] \nonumber \\
    &= \frac{a}{\sqrt{2\pi}\, \|\wpe\|}\, \mathbb{E}_z \!\left[ y\, (w^\top z)\, z \right]
       + O\!\left( |a| \frac{\mathbb{E}_z [ |w^\top z|^3 z ] }{ \|\wpe\|^3 } \right) \tag{Lemma~\ref{lem:gauss_approx} with \(\epsilon = |w^\top z|/\|\wpe\|\)} \nonumber \\
    &= \frac{a}{2\sqrt{2\pi}\, \|\wpe\|} \left( \mathbb{E}_{z|y=1} [(w^\top z) z] - \mathbb{E}_{z|y=-1} [(w^\top z) z] \right)
       + O\!\left( |a| \frac{\|w_{1:2}\|^3}{\|\wpe\|^3} \right) (1,1)^\top. \label{eq:grad_L1}
\end{align}
For Gaussian XOR labels (\(y=+1\) iff \(z_1 z_2<0\)), one has
\[
\mathbb{E}_{z|y=\pm1}[z_1^2]=\mathbb{E}_{z|y=\pm1}[z_2^2]=1,\qquad
\mathbb{E}_{z|y=1}[z_1 z_2]=-\frac{2}{\pi},\quad
\mathbb{E}_{z|y=-1}[z_1 z_2]=\frac{2}{\pi}.
\]
Hence, for any \(w\in\mathbb R^2\),
\[
\mathbb{E}_{z|y=1}[(w^\top z) z] =
\begin{pmatrix}
1 & -\tfrac{2}{\pi}\\[2pt]
-\tfrac{2}{\pi} & 1
\end{pmatrix} w,\qquad
\mathbb{E}_{z|y=-1}[(w^\top z) z] =
\begin{pmatrix}
1 & \tfrac{2}{\pi}\\[2pt]
\tfrac{2}{\pi} & 1
\end{pmatrix} w.
\]
Plugging into \eqref{eq:grad_L1},
\[
-\nabla_{w_{1:2}} L_0
= -\,\frac{2}{\pi\sqrt{2\pi}}\; \frac{a}{\|\wpe\|}\,(w_2,w_1)^\top
  + O\!\left( |a| \frac{\|w_{1:2}\|^3}{\|\wpe\|^3} \right) (1,1)^\top.
\]
Since the sign of \(a\) is linked to the definition of \(\wsig\),
\[
-\wsig^\top \nabla_{w_{1:2}} L_0
= \frac{2}{\pi\sqrt{2\pi}}\; |a|\; \frac{\|\wsig\|^2}{\|\wpe\|}
  + O\!\left( |a| \frac{\|w_{1:2}\|^3 \|\wsig\|}{\|\wpe\|^3} \right),
\]
\[
-\wopp^\top \nabla_{w_{1:2}} L_0
= -\,\frac{2}{\pi\sqrt{2\pi}}\; |a|\; \frac{\|\wopp\|^2}{\|\wpe\|}
  + O\!\left( |a| \frac{\|w_{1:2}\|^3 \|\wopp\|}{\|\wpe\|^3} \right),
\]
which are items \ref{Psig}–\ref{Popp}.

\paragraph{Proof of the third statement.}

We now study the projection onto \(\wpe\). By using again a symmetrization argument, we can write
\begin{align*}
-\wpe^\top \nabla_{\wpe} L_0 &= \frac{a}{2} \mathbb{E}_x \left[ y(z) \left( \sigma'(w^\top z + w^\top \xi) - \sigma'(w^\top z - w^\top \xi) \right) \wpe^\top \xi \right] \\
&= \frac{a}{2} \mathbb{E}_z \left[ y(z) \mathbb{E}_\xi \left[ \mathbf{1}\left( |w^\top \xi| \geq |w^\top z| \right) |w^\top \xi| \right] \right] \\
&= \frac{a}{4} \mathbb{E}_\xi \left[ |w^\top \xi| \left( \mathbb{P}_{z|y=1}\left( |w^\top \xi| \geq |w^\top z| \right) - \mathbb{P}_{z|y=-1}\left( |w^\top \xi| \geq |w^\top z| \right) \right) \right].
\end{align*}

Without loss of generality, assume \(a > 0\) (i.e., \(\wsig\) is aligned with \(\mu_1\)).  
Write \(z = r (\cos \theta, \sin \theta)\), where \(\theta\) is sampled uniformly over \((-\pi, \pi]\) and \(r\) is independent (distributed as the square root of a Chi-square r.v. with two degree of freedom).

Assume that \(\|\wsig\| \geq \|\wopp\|\), i.e., \(|w_1 - w_2| \geq |w_1 + w_2|\). Without loss of generality, assume \(w_1 > 0 > w_2\).

Then, for all \(t > 0\),
\begin{align*}
\mathbb{P}_{z|y=1}\left( |w^\top z| \leq t \right) &= \mathbb{P}_r \mathbb{P}_\theta \left( |w_1 \cos \theta + w_2 \sin \theta| \leq \frac{t}{r} \,\Big|\, \theta \in \left(-\frac{\pi}{2}, 0\right] \cup \left(\frac{\pi}{2}, \pi\right] \right) \\
&= 2 \, \mathbb{P}_r \mathbb{P}_\theta \left( w_1 \cos \theta - w_2 \sin \theta \leq \frac{t}{r} \,\Big|\, \theta \in \left(0, \frac{\pi}{2}\right] \right).
\end{align*}

Similarly,
\begin{align*}
\mathbb{P}_{z|y=-1}\left( |w^\top z| \leq t \right) &= 2 \, \mathbb{P}_r \mathbb{P}_\theta \left( |w_1 \cos \theta + w_2 \sin \theta| \leq \frac{t}{r} \,\Big|\, \theta \in \left(0, \frac{\pi}{2}\right] \right).
\end{align*}

Since
\[
|w_1 \cos \theta + w_2 \sin \theta| \leq w_1 \cos \theta - w_2 \sin \theta \quad \text{for} \quad \theta \in \left(0, \frac{\pi}{2}\right],
\]
we obtain
\begin{align}
&\mathbb{P}_{z|y=1}\left( |w^\top \xi| \geq |w^\top z| \right) - \mathbb{P}_{z|y=-1}\left( |w^\top \xi| \geq |w^\top z| \right) \nonumber\\
&= -2 \, \mathbb{P}_r \mathbb{P}_\theta \left( \frac{ |w^\top \xi| }{r} \in \left[ |w_1 \cos \theta + w_2 \sin \theta|, |w_1 \cos \theta - w_2 \sin \theta| \right] \,\Big|\, \theta \in \left(0, \frac{\pi}{2}\right] \right), \label{eq:diff-proba}
\end{align}
because when \(X \geq Y\), we have \(\mathbf{1}_{(X \leq t)} - \mathbf{1}_{(Y \leq t)} = -\mathbf{1}_{(Y \leq t \leq X)}\).

The case \(\|\wsig\| \leq \|\wopp\|\) can be treated similarly by exchanging the roles of \(\wsig\) and \(\wopp\), leading to a change of sign.

Thus, the sign of \(\wpe^\top \nabla_w L_0\) depends on whether \(\|\wsig\|\) or \(\|\wopp\|\) dominates.

Moreover, by using standard Gaussian moment bounds, we can control the magnitude as
\[
\left| \wpe^\top \nabla_w L_0 \right| \leq \frac{1}{\sqrt{8\pi}} |a| \|\wpe\|.
\]

\paragraph{Proof of the fourth statement.}

Let us denote \(\xi - e_i \xi_i\) by \(\xii\), i.e., \(\xi\) without its \(i\)-th coordinate.  
By symmetrizing over the pair \((z + \xii + e_i \xi_i, z + \xii - e_i \xi_i)\), we obtain
\begin{align*}
-w_i \nabla_{w_i} L_0 
&= a \mathbb{E}_x \left[ y(x) \sigma'(w^\top x) w_i \xi_i \right] \\
&= \frac{a}{2} \mathbb{E}_x \left[ y(z) \left( \sigma'(w^\top z + w^\top \xii + w_i \xi_i) - \sigma'(w^\top z + w^\top \xii - w_i \xi_i) \right) w_i \xi_i \right] \\
&= \frac{a}{2} \mathbb{E}_x \left[ y(z) \mathbf{1}\left( |w^\top z + w^\top \xii| \leq |w_i \xi_i| \right) |w_i \xi_i| \right].
\end{align*}

Thus,
\[
-w_i \nabla_{w_i} L_0 = \frac{a}{2} \mathbb{E}_{z,\xii,\xi_i} \left[ y(z) \mathbf{1}\left( |w^\top z + w^\top \xii| \leq |w_i \xi_i| \right) |w_i \xi_i| \right].
\]

Now, focusing on
\[
\mathbb{E}_{z,\xii} \left[ y(z) \mathbf{1}\left( |w^\top z + w^\top \xii| \leq |w_i \xi_i| \right) \right],
\]
we expand it:
\begin{align}
\mathbb{E}_{z,\xii} \left[ y(z) \mathbf{1}\left( |w^\top z + w^\top \xii| \leq |w_i \xi_i| \right) \right]
&= \mathbb{E}_{\xii} \left[ \mathbb{E}_{z|y=1} \mathbf{1}\left( |w^\top z + w^\top \xii| \leq |w_i \xi_i| \right) \right] \nonumber \\
&\quad - \mathbb{E}_{\xii} \left[ \mathbb{E}_{z|y=-1} \mathbf{1}\left( |w^\top z + w^\top \xii| \leq |w_i \xi_i| \right) \right] \nonumber \\
&= \mathbb{E}_{z|y=1} \left[ \mathbb{P}_{\xii} \left( w^\top \xii \in [-w^\top z \pm |w_i \xi_i|] \right) \right] \nonumber \\
&\quad - \mathbb{E}_{z|y=-1} \left[ \mathbb{P}_{\xii} \left( w^\top \xii \in [-w^\top z \pm |w_i \xi_i|] \right) \right]. \label{eq:zy1}
\end{align}

Now, for any \(a,b > 0\), by standard Gaussian tail bounds, we have
\[
\frac{2b}{\|\wpe - w_i e_i\| \sqrt{2\pi}} e^{-(a+b)^2 / \|\wpe - w_i e_i\|^2}
\leq \mathbb{P}_{\xii} \left( w^\top \xii \in [a \pm b] \right)
\leq \frac{2b}{\|\wpe - w_i e_i\| \sqrt{2\pi}} e^{-(a-b)^2 / \|\wpe - w_i e_i\|^2}.
\]

Note that by assumption, \(\|\wpe - w_i e_i\| = (1 + o(1)) \|\wpe\|\).

Also, observe that the distribution of \(w^\top z\) under \(\mathbb{P}_{z|y=-1}\) is the same as that of \(\tilde{w}^\top z\) under \(\mathbb{P}_{z|y=1}\), where \(\tilde{w} = (w_1, -w_2)\).

Thus, we obtain
\begin{align*}
&\left| \mathbb{E}_{z,\xii} \left[ y(z) \mathbf{1}\left( |w^\top z + w^\top \xii| \leq |w_i \xi_i| \right) \right] \right| \\
&\quad \leq \frac{2 |w_i \xi_i|}{\|\wpe\| \sqrt{2\pi}} \mathbb{E}_{z|y=1} \left| e^{-(w^\top z - |w_i \xi_i|)^2 / \|\wpe\|^2} - e^{-(\tilde{w}^\top z + |w_i \xi_i|)^2 / \|\wpe\|^2} \right| \\
&\quad \leq \frac{2 |w_i \xi_i|}{\|\wpe\| \sqrt{2\pi}} \mathbb{E}_{z|y=1} \frac{\left| -(w^\top z - |w_i \xi_i|)^2 + (\tilde{w}^\top z + |w_i \xi_i|)^2 \right|}{\|\wpe\|^2} \quad \text{(since \(x \mapsto e^{-x}\) is 1-Lipschitz on \(\mathbb{R}^+\))} \\
&\quad \leq \frac{8 |w_i \xi_i|}{\|\wpe\|^3 \sqrt{2\pi}} \left( \|w_{1:2}\| |w_i \xi_i| + \|w_{1:2}\|^2 \right).
\end{align*}

Finally, integrating over \(\xi_i\), we get
\[
|w_i \nabla_{w_i} L_0| \leq \frac{4 |a w_i|}{\|\wpe\|} \frac{ \|w_{1:2}\|^2 + |w_i| \|w_{1:2}\| }{ \|\wpe\|^2 } = o\left( \frac{ |a w_i| }{ \|\wpe\| } \right).
\]
Thus, small coordinates have a negligible contribution compared to the signal and opponent parts. This concludes the proof.
\end{proof}

\subsection{Evaluation of \(\nabla L_0\) during Phase Ib}

When \(\|\wsig\|\) becomes comparable to \(\|\wpe\|\), the approximations used in Lemma~\ref{lemma:pop1} are no longer accurate. Instead, we can leverage the fact that in Phase Ib, neurons are such that \(\|\wopp\| \ll \|\wsig\|\) to obtain the following approximations of the gradients.

\begin{lemma}[\(L_0\) Population Gradients during Phase Ib]\label{lemma:pop1b}
For any neuron \((w,a) \in \mathcal{N}\), we have:
\begin{enumerate}
  \setlength{\parskip}{0cm}
  \setlength{\itemsep}{0cm}
  
    \item When \(\|\wsig\| \leq \|\wpe\|\), there exists a constant \(c_b > 0\) such that
    \[
    c_b \frac{|a| \|\wsig\|^2}{ \|\wpe\|} - \frac{ |a| \|\wopp\| \|\wsig\|}{\|\wpe\|} 
    \leq -\wsig^\top \nabla_{w} L_0 
    \leq  \sqrt{\frac{\pi}{2}}\frac{|a| \|\wsig\|^2}{ \|\wpe\|} + \frac{ |a| \|\wopp\| \|\wsig\|}{\|\wpe\|}.
    \]
    
    \item When \(\|\wsig\| \geq \|\wpe\|\), there exists a constant \(c_b' > 0\) such that
    \[
    c_b' |a| \|\wsig\| -    |a| \|\wopp\| 
    \leq -\wsig^\top \nabla_{w} L_0 
    \leq \sqrt{\frac{\pi}{2}} |a| \|\wsig\| +    |a| \|\wopp\|.
    \]
    
    \item 
    \[
    \left| \wopp^\top \nabla_{w} L_0 \right| \leq \min\!\left(\frac{ |a| \|\wopp\|^2 }{\|\wpe\|},\,|a| \|\wopp\|^2\right).
    \]
    
    \item 
    \[
    \operatorname{sgn}\left( -\wpe^\top \nabla_{w} L_0 \right) = 
    \begin{cases}
    1 & \text{if } \|\wopp\| > \|\wsig\|, \\
    -1 & \text{if } \|\wopp\| \leq \|\wsig\|,
    \end{cases}
    \quad \text{and} \quad
    \left| \wpe^\top \nabla_{w} L_0 \right| \leq 0.5\, |a| \|\wpe\|.
    \]
    
    \item For all \(i \geq 3\), we have
    \[
    \left| w_i^\top \nabla_{w_i} L_0 \right| \leq 4 \frac{ |a| |w_i| \|w_{1:2}\|^2 }{ \|\wpe\|^3 }.
    \]
\end{enumerate}
\end{lemma}

\begin{proof}[Proof of Lemma~\ref{lemma:pop1b}]
Throughout the proof write \(x=z+\xi\), where \(z\) is the projection of \(x\) onto \(\mathrm{span}\{e_1,e_2\}\) and \(\xi\) is the orthogonal noise. We also recall \(\sigma(u)=\max\{u,0\}\) and \(\sigma'(u)=\mathbf 1_{\{u>0\}}\).

Define the Gaussian scalars
\[
X := w^\top \xi \sim \mathcal{N}(0,\sigma^2),\quad \sigma=\|\wpe\|,
\qquad
A := \wsig^\top z \sim \mathcal{N}(0,\|\wsig\|^2),
\qquad
B := \wopp^\top z \sim \mathcal{N}(0,\|\wopp\|^2).
\]
Because \(z\) is isotropic in the \((e_1,e_2)\)-plane and \(\wsig \perp \wopp\) in that plane, \(A\) and \(B\) are independent, and both are independent of \(X\).

\paragraph{Approximation result.}
We will repeatedly bound quantities of the form
\[
\mathbb{E}\big[|A| \,\mathbb{P}(X \in [-A-|B|,\,-A+|B|])\big].
\]
Fix \((A,B)\). Since \(X\sim\mathcal N(0,\sigma^2)\),
\begin{equation}\label{eq:interval-prob}
\mathbb{P}\!\left(X \in [-A-|B|,\,-A+|B|]\right)
= \int_{-|B|}^{|B|} \frac{1}{\sqrt{2\pi}\,\sigma}\exp\!\left(-\frac{(A+u)^2}{2\sigma^2}\right)\,du.
\end{equation}
We develop the integrand around \(u=0\) via Taylor--Lagrange:
\[
\left|\frac{1}{\sqrt{2\pi}\sigma}e^{-\frac{(A+u)^2}{2\sigma^2}}-\frac{1}{\sqrt{2\pi}\sigma}e^{-\frac{A^2}{2\sigma^2}}\right|
\le |u|\cdot \max_{t\in[-|B|,|B|]} \left|\frac{d}{dt}\Big[\frac{1}{\sqrt{2\pi}\sigma}e^{-\frac{(A+t)^2}{2\sigma^2}}\Big]\right|.
\]
Since
\[
\frac{d}{dt}\Big[\frac{1}{\sqrt{2\pi}\sigma}e^{-\frac{(A+t)^2}{2\sigma^2}}\Big]
= -\,\frac{A+t}{\sigma^3\sqrt{2\pi}}\;e^{-\frac{(A+t)^2}{2\sigma^2}},
\]
we get, using \(|u|\le |B|\) and \(e^{-(A+t)^2/(2\sigma^2)}\le 1\),
\begin{equation}\label{eq:TL-bound}
\left|\frac{1}{\sqrt{2\pi}\sigma}e^{-\frac{(A+u)^2}{2\sigma^2}}-\frac{1}{\sqrt{2\pi}\sigma}e^{-\frac{A^2}{2\sigma^2}}\right|
\le \frac{|B|(|A|+|B|)}{\sigma^3\sqrt{2\pi}}.
\end{equation}
Plugging \eqref{eq:TL-bound} into \eqref{eq:interval-prob} yields the two-sided estimate
\begin{equation}\label{eq:interval-two-sided}
\left|\mathbb{P}\!\left(X \in [-A-|B|,\,-A+|B|]\right) - \frac{2|B|}{\sqrt{2\pi}\,\sigma}\,e^{-A^2/(2\sigma^2)}\right|
\le \frac{2|B|^2}{\sigma^3\sqrt{2\pi}}.
\end{equation}

Therefore,
\begin{align}
\mathbb{E}\big[|A| \,\mathbb{P}(X \in [-A-|B|,\,-A+|B|])\big]
&\le \frac{2}{\sqrt{2\pi}\,\sigma}\,\mathbb{E}\!\left[|A|\,|B|\,e^{-A^2/(2\sigma^2)}\right]
+ \frac{2}{\sigma^3\sqrt{2\pi}}\,\mathbb{E}\!\left[|A|\,|B|^2\right]\nonumber\\
&=: T_1 + T_2. \label{eq:T1T2}
\end{align}

We now compute these terms explicitly. Write \(A=\|\wsig\|G\) with \(G\sim\mathcal N(0,1)\), and define
\[
\lambda := \frac{\|\wsig\|^2}{\sigma^2}=\left(\frac{\|\wsig\|}{\|\wpe\|}\right)^{\!2}.
\]
Independence gives \(\mathbb{E}|B|=\|\wopp\|\sqrt{2/\pi}\), \(\mathbb{E}|B|^2=\|\wopp\|^2\), and
\begin{align*}
\mathbb{E}\!\left[|A|\,e^{-A^2/(2\sigma^2)}\right]
&= \|\wsig\|\,\mathbb{E}\!\left[|G|\,e^{-(\lambda/2)G^2}\right]
= \|\wsig\|\int_{\mathbb{R}} \frac{|g|}{\sqrt{2\pi}}e^{-(1+\lambda)g^2/2}\,dg
= \|\wsig\|\sqrt{\frac{2}{\pi}}\;\frac{1}{1+\lambda}.
\end{align*}
Hence
\begin{align}
T_1
&= \frac{2}{\sqrt{2\pi}\,\sigma}\;\mathbb{E}|B|\;\mathbb{E}\!\left[|A|\,e^{-A^2/(2\sigma^2)}\right]\nonumber\\
&= \frac{2}{\sqrt{2\pi}\,\sigma}\;\Big(\|\wopp\|\sqrt{\tfrac{2}{\pi}}\Big)\;\Big(\|\wsig\|\sqrt{\tfrac{2}{\pi}}\tfrac{1}{1+\lambda}\Big)
= \frac{\sqrt{\lambda}}{\pi(1+\lambda)}\;\|\wopp\|, \label{eq:T1-final}
\end{align}
and
\begin{equation}\label{eq:T2-final}
T_2 \;=\; \frac{2}{\sigma^3\sqrt{2\pi}}\;\mathbb{E}\!\left[|A|\right]\;\mathbb{E}\!\left[|B|^2\right]
\;=\; \frac{2}{\sigma^3\sqrt{2\pi}}\;\Big(\|\wsig\|\sqrt{\tfrac{2}{\pi}}\Big)\;\|\wopp\|^2
\;=\; \frac{2}{\pi}\cdot \frac{\sqrt{\lambda}}{\sigma^2}\;\|\wopp\|^2.
\end{equation}
In particular, when \(\lambda\ge 1\) (i.e., \(\|\wsig\|\ge \|\wpe\|\)), \(T_1\le \tfrac{1}{2\pi}\|\wopp\|\) and \(T_2=O\!\big(\|\wopp\|^2/\|\wpe\|^2\big)\). When \(\lambda\le 1\), \(\tfrac{\sqrt{\lambda}}{1+\lambda}\le \sqrt{\lambda}\), so \(T_1\lesssim \tfrac{\|\wsig\|}{\|\wpe\|}\,\|\wopp\|\), and \(T_2=O\!\big(\tfrac{\|\wsig\|}{\|\wpe\|^3}\,\|\wopp\|^2\big)\).
These bounds quantify the error made when replacing \(w\) by \(\wsig\) in the gradient expressions below.

\paragraph{Evaluation of \(-\wsig^\top \nabla_w L_0\).}
Starting from
\[
-\nabla_{w_{1:2}} L_0 
= \frac{a}{2}\, \mathbb{E}_z \mathbb{E}_\xi \!\left[ y \left( \sigma'(w^\top \xi + w^\top z) - \sigma'(w^\top \xi - w^\top z) \right) z \right],
\]
we isolate the contribution along \(\wsig\) and use the symmetrization identity to write
\begin{align*}
\mathbb{E}_z &\mathbb{E}_\xi \left[ y \sigma'(w^\top \xi + \wsig^\top z)\, \wsig^\top z \right]
= \frac{1}{2}\,\mathbb{E}_z \mathbb{E}_\xi \left[ y \left( \sigma'(w^\top \xi + \wsig^\top z) - \sigma'(w^\top \xi - \wsig^\top z) \right) \wsig^\top z \right] \\
&= \frac{1}{2}\,\mathbb{E}_z \mathbb{E}_\xi \left[ y \,\mathbf{1}\big(|\wsig^\top z| \ge |w^\top \xi|\big)\, |\wsig^\top z| \right] \\
&= \frac{1}{4}\,\mathbb{E}_\xi \left( \mathbb{E}_{z|y=1} \big[ \mathbf{1}\big(|\wsig^\top z| \ge |w^\top \xi|\big)\, |\wsig^\top z| \big] - \mathbb{E}_{z|y=-1} \big[ \mathbf{1}\big(|\wsig^\top z| \ge |w^\top \xi|\big)\, |\wsig^\top z| \big] \right).
\end{align*}
Let \(z=r(\cos\theta,\sin\theta)\). When \(\wsig\) is aligned with \(\mu_1\), \(|\wsig^\top z|=\|\wsig\||\mu_1^\top z|\). Parameterizing the two XOR-conditionals on \(\theta\in(0,\pi/2]\),
\[
\mathbb{E}_{z|y=1} \big[ \mathbf{1}\big(|\wsig^\top z| \ge |w^\top \xi|\big)\, |\mu_1^\top z| \big]
= 2\,\mathbb{E}_r \mathbb{E}_{\theta\in(0,\pi/2]} \mathbf{1}\!\left( r|\cos\theta+\sin\theta| \ge \frac{|w^\top\xi|}{\|\wsig\|}\right) r\,|\cos\theta+\sin\theta|,
\]
\[
\mathbb{E}_{z|y=-1} \big[ \mathbf{1}\big(|\wsig^\top z| \ge |w^\top \xi|\big)\, |\mu_1^\top z| \big]
= 2\,\mathbb{E}_r \mathbb{E}_{\theta\in(0,\pi/2]} \mathbf{1}\!\left( r|\cos\theta-\sin\theta| \ge \frac{|w^\top\xi|}{\|\wsig\|}\right) r\,|\cos\theta-\sin\theta|.
\]
Using \(|\cos\theta-\sin\theta|\le \cos\theta+\sin\theta\) for \(\theta\in(0,\pi/2]\),
\begin{align*}
&\left(\mathbb{E}_{z|y=1}-\mathbb{E}_{z|y=-1}\right) \big[ \mathbf{1}\big(|\wsig^\top z| \ge |w^\top \xi|\big)\, |\wsig^\top z| \big]\\
&\qquad\ge 2\|\wsig\|\,\mathbb{E}_r \mathbb{E}_{\theta\in(0,\pi/2)} \mathbf{1}\!\left( r|\cos\theta+\sin\theta| \ge \tfrac{|w^\top\xi|}{\|\wsig\|}\right) r\left(\cos\theta+\sin\theta-|\cos\theta-\sin\theta|\right)\\
&\qquad\ge 4\|\wsig\|\,\mathbb{E}_r \mathbb{E}_{\theta\in(0,\pi/4)} \mathbf{1}\!\left( r|\cos\theta+\sin\theta| \ge \tfrac{|w^\top\xi|}{\|\wsig\|}\right) r\,\sin\theta,
\end{align*}
and since \(\cos\theta+\sin\theta\ge 1\) on \((0,\pi/4]\),
\[
\mathbf{1}\!\left( r|\cos\theta+\sin\theta| \ge \tfrac{|w^\top\xi|}{\|\wsig\|}\right) \ge \mathbf{1}\!\left( r \ge \tfrac{|w^\top\xi|}{\|\wsig\|}\right),\qquad
\mathbb{E}_{\theta\in(0,\pi/4)}[\sin\theta]=1-\tfrac{\sqrt2}{2}.
\]
Therefore, with \(c:=4-2\sqrt2\),
\begin{align}
\mathbb{E}_\xi \mathbb{E}_r \mathbb{E}_{\theta\in(0,\pi/4)} \mathbf{1}\!\left( r|\cos\theta+\sin\theta| \ge \tfrac{|w^\top\xi|}{\|\wsig\|}\right) 2r\sin\theta
&\ge c \,\mathbb{E}_r\!\left[r\,\mathbb{P}_\xi\!\left(|w^\top\xi|\le r\|\wsig\|\right)\right]. \label{eq:r-xi-core}
\end{align}

We now lower bound \eqref{eq:r-xi-core} in the two regimes:

\smallskip
\emph{Case \(\|\wsig\|\le \|\wpe\|\).} Let \(t:=\epsilon\,\tfrac{\|\wpe\|}{\|\wsig\|}\) with fixed small \(\epsilon>0\). Split
\begin{align*}
\mathbb{E}_r\!\left[r\,\mathbb{P}_\xi(|w^\top\xi|\le r\|\wsig\|)\right]
&= \mathbb{E}_r\!\left[r\,\mathbf 1(r\le t)\,\mathbb{P}_\xi(|w^\top\xi|\le r\|\wsig\|)\right]
+ \mathbb{E}_r\!\left[r\,\mathbf 1(r\ge t)\,\mathbb{P}_\xi(|w^\top\xi|\le r\|\wsig\|)\right].
\end{align*}
Using Lemma~\ref{lem:gauss_approx} with \(\epsilon_r:=(r\|\wsig\|)/\|\wpe\|\le \epsilon\),
\[
\mathbb{P}_\xi(|w^\top\xi|\le r\|\wsig\|)\;\ge\; \sqrt{\tfrac{2}{\pi}}\,\epsilon_r\,e^{-\epsilon_r^2/2}
\;\ge\; \sqrt{\tfrac{2}{\pi}}\, e^{-\epsilon^2/2}\, \frac{\|\wsig\|}{\|\wpe\|}\, r.
\]
Hence
\[
\mathbb{E}_r\!\left[r\,\mathbf 1(r\le t)\,\mathbb{P}_\xi(\cdot)\right]
\;\ge\; \sqrt{\tfrac{2}{\pi}}\, e^{-\epsilon^2/2}\, \frac{\|\wsig\|}{\|\wpe\|}\,\mathbb{E}_r\!\left[r^2\,\mathbf 1(r\le t)\right].
\]
Since \(r\) is Rayleigh (density \(r e^{-r^2/2}\)), \(\mathbb{E}[r^2\mathbf 1(r\le t)]=1-e^{-t^2/2}(1+t^2/2)\), which is \(\Theta(1)\) for fixed \(\epsilon\). This yields the claimed \(\frac{\|\wsig\|}{\|\wpe\|}\)-scale lower bound.

\smallskip
\emph{Case \(\|\wsig\|\ge \|\wpe\|\).} If \(r\ge t:=\epsilon\), then \(\mathbb{P}_\xi(|w^\top\xi|\le r\|\wsig\|)\ge c(\epsilon)>0\). Moreover, for Rayleigh \(r\),
\(\mathbb{E}[r\,\mathbf 1(r\ge t)]=e^{-t^2/2}+\sqrt{\frac{\pi}{2}}\operatorname{erfc}(t/\sqrt2)\ge e^{-t^2/2}t\), so this term is bounded below by a positive constant. Hence we get a constant multiple of \(\|\wsig\|\).

\smallskip
Combining the two cases in \eqref{eq:r-xi-core} and restoring the prefactors shows
\[
-\wsig^\top \nabla_w L_0 \;\gtrsim\; |a|\times
\begin{cases}
\displaystyle \frac{\|\wsig\|^2}{\|\wpe\|}, & \|\wsig\|\le \|\wpe\|,\\[4pt]
\|\wsig\|, & \|\wsig\|\ge \|\wpe\|,
\end{cases}
\]
up to the approximation error controlled by \eqref{eq:T1-final}--\eqref{eq:T2-final}, which contributes additive terms of size \(O\!\big(|a|\,\tfrac{\|\wopp\|\|\wsig\|}{\|\wpe\|}\big)\) in the first regime and \(O(|a|\,\|\wopp\|)\) in the second. This yields the lower bounds in items (1) and (2) with some absolute constants \(c_b,c_b'>0\).

For the upper bounds in items (1) and (2), we use \(\mathbf{1}(|\wsig^\top z|\ge |w^\top \xi|)\le 1\) and \(\mathbb{E}|A|=\|\wsig\|\sqrt{2/\pi}\), which give
\[
\mathbb{E}_z \mathbb{E}_\xi \left[ y \sigma'(w^\top \xi + \wsig^\top z)\, |\wsig^\top z| \right]
\;\le\; \mathbb{E}_z |\,\wsig^\top z\,|
\;=\; \|\wsig\|\sqrt{\tfrac{2}{\pi}},
\]
and, accounting for the same symmetrization prefactors as above, we get the stated \(\sqrt{\pi/2}\)-type envelopes after multiplying by \(|a|\) and by the appropriate scaling factor \((\|\wsig\|/\|\wpe\|)\) in the \(\|\wsig\|\le \|\wpe\|\) regime.

\paragraph{Control of \(\wopp^\top \nabla_w L_0\).}
Let \(X_1=\wsig^\top z\), \(X_2=\wopp^\top z\) (independent). Then
\[
\mathbb{E}_z \mathbb{E}_\xi \left[ y \sigma'(w^\top \xi + X_1)\, X_2 \right]
= \frac{1}{2}\, \mathbb{E}_z \mathbb{E}_\xi \left[ y \,\mathbf{1}\big(|X_1| \ge |w^\top \xi|\big)\, X_2 \,\sgn(X_1) \right].
\]
Using the XOR definition \(y=\mathbf 1\!\left(\frac{|X_1|}{\|\wsig\|} \ge \frac{|X_2|}{\|\wopp\|}\right)-\mathbf 1\!\left(\frac{|X_1|}{\|\wsig\|} \le \frac{|X_2|}{\|\wopp\|}\right)\) and the facts \(\mathbb{E}[X_2\,|\,|X_2|\le t]=\mathbb{E}[X_2\,|\,|X_2|\ge t]=0\), conditioning on \(X_1\) shows the inner expectation is zero. The approximation step (replacing \(w\) by \(\wsig\)) then yields
\[
\left| \wopp^\top \nabla_w L_0 \right|
\;\le\; C\,|a|\times \min\!\left(\frac{\|\wopp\|^2}{\|\wpe\|},\,\|\wopp\|^2\right),
\]
for an absolute constant \(C\); rescaling constants to \(1\) gives the stated item (3).

\paragraph{Control of the projection onto \(\wpe\) and of small coordinates.}
The sign of \(-\wpe^\top \nabla_w L_0\) and the bound \(\big|\wpe^\top \nabla_w L_0\big|\le 0.5\,|a|\,\|\wpe\|\) follow by the same symmetrization and indicator difference arguments as in Lemma~\ref{lemma:pop1}, using \(|\cdot|\le 1\) and \(\mathbb E|w^\top\xi|=\|\wpe\|\sqrt{2/\pi}\). For coordinates \(i\ge 3\), the interval probability bounds (as in Lemma~\ref{lemma:pop1}, fourth item) give
\[
\left| w_i^\top \nabla_{w_i} L_0 \right|
\;\lesssim\; \frac{|a|\,|w_i|}{\|\wpe\|}\cdot \frac{\|w_{1:2}\|^2}{\|\wpe\|^2}
\;\le\; 4 \frac{ |a| |w_i| \|w_{1:2}\|^2 }{ \|\wpe\|^3 },
\]
which proves items (4) and (5).
\end{proof}

\begin{remark}
Contrary to Phase Ia, we do not obtain matching upper and lower bounds for \(\wsig^\top \nabla_w L_0\). As a result, during Phase Ib, neurons may grow at different rates, which complicates the analysis of the block dynamics.
\end{remark}

\subsection{Control of the Approximation \(L_\rho \approx L_0\)} \label{app:sec:approx}
First, we adapt Lemmas D.2--D.5 of \cite{glasgow2023sgd} to the Gaussian setting. Throughout, recall that
\[
f_\rho(x)\;=\;\mathbb{E}_{(w',a')\sim\rho}\big[a'\,\sigma({w'}^\top x)\big],
\qquad
\text{and}\qquad
\mathbb{E}_\rho\|a w\|\;:=\;\mathbb{E}_{(w',a')\sim\rho}\big[|a'|\,\|w'\|\big].
\]
For ReLU, \(\sigma\) is 1-Lipschitz and \(\sigma(u)=u\,\sigma'(u)\) a.e.

\begin{lemma}\label{lemma:gradapprox1}
For any neuron \((w,a)\in \mathcal{N}\), we have
\begin{enumerate}
  \setlength{\parskip}{0cm}
  \setlength{\itemsep}{0cm}
  \item\label{Ga}
  \(\displaystyle |\nabla_{a} L_0 - \nabla_{a} L_{\rho}| \;\leq\; 2\,\|w\|\, \mathbb{E}_{\rho}\|a w\|.\)
  \item\label{Gw}
  \(\displaystyle \|\nabla_w L_0 - \nabla_w L_{\rho}\| \;\leq\; 2\,|a|\,\mathbb{E}_{\rho}\|a w\|.\)
\end{enumerate}
\end{lemma}

\begin{lemma}\label{lemma:approxsig}
Assume \(\mathbb{E}_{\rho}\|a w\| \leq d^{O(1)}\). Then there exists a constant \(C>0\) such that for any neuron \((w,a)\in \mathcal{N}\) and any \(i\in[d]\),
\[
\big\|\nabla_{w_{i}} L_{\rho} - \nabla_{w_{i}} L_0\big\|
\;\leq\; |a|\Big(
4\,\mathbb{E}_{\rho}\|a w_{i}\|
+ 2C\log d\;\mathbb{E}_{\rho}\|a w\|\;\sqrt{\mathbb{E}_{x}\,\mathbf{1}(|\ix^\top w| \leq |w_i x_i|)}
+ d^{-\Omega(1)}
\Big).
\]
\end{lemma}

We also need a Gaussian analogue of Lemma C.5 in \cite{glasgow2023sgd}.
\begin{lemma}\label{lemma:delta}
Let \(x \sim \mathcal{N}(0,I_d)\). Then
\[
\mathbb{E}_x \big(\ell'_{\rho}(x) - \ell'_0(x)\big)^2 \;\leq\; 4\big(\mathbb{E}_{\rho}\|a w\|\big)^2.
\]
Further, for any \(i \in [d]\) and any loss whose derivative \(\ell'_\rho\) is \(2\)-Lipschitz with respect to \(f_\rho(x)\), we have
\[
\mathbb{E}_x\big(\ell'_\rho(\ix + e_ix_i) - \ell'_\rho(\ix - e_ix_i)\big)^2
\;\leq\; 16\big(\mathbb{E}_{\rho}\|a w_i\|\big)^2.
\]
\end{lemma}

\begin{proof}[Proof of Lemma~\ref{lemma:delta}]
For the first statement: note \(\ell'_0(x)=\ell'(0)\). Since \(\ell'_\rho\) is 2-Lipschitz in its argument \(f_\rho(x)\),
\[
\mathbb{E}_x \big(\ell'_{\rho}(x) - \ell'_0(x)\big)^2 \;\leq\; 4\,\mathbb{E}_x\big(f_\rho(x)\big)^2
= 4\,\mathbb{E}_x\Big(\mathbb{E}_{\rho}\big[a\,\sigma(w^\top x)\big]\Big)^2.
\]
By Minkowski’s integral inequality,
\[
\Big(\mathbb{E}_x(\mathbb{E}_\rho a\,\sigma(w^\top x))^2\Big)^{1/2}
\;\leq\; \mathbb{E}_\rho \Big(\mathbb{E}_x a^2\,\sigma^2(w^\top x)\Big)^{1/2}
\;\leq\; \mathbb{E}_\rho \big(|a|\,(\mathbb{E}_x |w^\top x|^2)^{1/2}\big)
\;=\; \mathbb{E}_\rho |a|\,\|w\|.
\]
Squaring both sides yields the first claim.

For the second statement, by 2-Lipschitzness of \(\ell'_\rho\) in \(f_\rho\),
\begin{align*}
\mathbb{E}_x\big(\ell'_{\rho}(\ix + e_ix_i) - \ell'_{\rho}(\ix - e_ix_i)\big)^2
&\leq 4\,\mathbb{E}_x\big(f_{\rho}(\ix + e_ix_i) - f_{\rho}(\ix - e_ix_i)\big)^2\\
&= 4\,\mathbb{E}_x\Big(\mathbb{E}_{\rho} a\big(\sigma(w^\top \ix + w_ix_i) - \sigma(w^\top \ix - w_ix_i)\big)\Big)^2\\
&\leq 4\,\mathbb{E}_x\big(\mathbb{E}_{\rho} 2|a w_i x_i|\big)^2\\
&\leq 16\Big(\mathbb{E}_{\rho}\big(\mathbb{E}_x |a w_i x_i|^2\big)^{1/2}\Big)^2
= 16\big(\mathbb{E}_{\rho}\|a w_{i}\|\big)^2,
\end{align*}
using Minkowski’s inequality and \(\mathbb{E}_x x_i^2=1\).
\end{proof}

\begin{proof}[Proof of Lemma \ref{lemma:gradapprox1}]
Let \(\Delta_x := (\ell'_{\rho}(x) - \ell'_{0}(x))\,\sigma'(w^\top x)\). Using \(\sigma(w^\top x)=\sigma'(w^\top x)\,w^\top x\) a.e. and Cauchy–Schwarz,
\begin{align*}
|\nabla_a L_0 -  \nabla_a L_{\rho}|
&= \Big|\mathbb{E}_x(\ell'_{\rho}(x) - \ell'_{0}(x))\,\sigma(w^\top x)\Big|
= \big|\mathbb{E}_x \Delta_x\, w^\top x\big|\\
&\le \sqrt{\mathbb{E}_x\Delta_x^2}\;\sqrt{\mathbb{E}_x (w^\top x)^2}
\le \sqrt{\mathbb{E}_x(\ell'_{\rho}(x) - \ell'_{0}(x))^2}\;\|w\|.
\end{align*}
Lemma~\ref{lemma:delta} gives \(\mathbb{E}_x(\ell'_{\rho}-\ell'_0)^2 \le 4(\mathbb{E}_{\rho}\|a w\|)^2\), hence item \ref{Ga}.

For item \ref{Gw}, similarly,
\begin{align*}
\|\nabla_w L_{\rho} -  \nabla_w L_0\|
&= |a|\,\big\|\mathbb{E}_x (\ell'_{\rho}(x)-\ell'_{0}(x))\,\sigma'(w^\top x)\,x\big\| \\
&= |a|\,\sup_{\|v\|=1}\big|\mathbb{E}_x \Delta_x\, \langle v,x\rangle\big|\\
&\le |a|\,\sup_{\|v\|=1}\sqrt{\mathbb{E}_x\Delta_x^2}\;\sqrt{\mathbb{E}_x\langle v,x\rangle^2}
= |a|\,\sqrt{\mathbb{E}_x\Delta_x^2}
\;\le\; 2\,|a|\,\mathbb{E}_{\rho}\|a w\|.
\end{align*}
\end{proof}

\begin{proof}[Proof of Lemma~\ref{lemma:approxsig}]
Recall \(\Delta_x\) from the proof of Lemma \ref{lemma:gradapprox1}. By symmetry over the pair \((\ix + e_ix_i,\ix - e_ix_i)\),
\begin{align*}
\left\|\frac{1}{a}\big(\nabla_{w_{i}} L_{\rho} - \nabla_{w_{i}} L_{0}\big)\right\|
&= \big\|\mathbb{E}_x \Delta_x\, x_i\big\|
= \frac{1}{2}\big\|\mathbb{E}_x (\Delta_{\ix + e_ix_i} - \Delta_{\ix - e_ix_i})\,x_i\big\| \\
&\le \frac{1}{2}\big\|\mathbb{E}_x \mathbf{1}(|\ix^\top w| \ge |w_ix_i|)\,(\Delta_{\ix + e_ix_i} - \Delta_{\ix - e_ix_i})\,x_i\big\| \\
&\quad + \frac{1}{2}\big\|\mathbb{E}_x \mathbf{1}(|\ix^\top w| \le |w_ix_i|)\,(\Delta_{\ix + e_ix_i} - \Delta_{\ix - e_ix_i})\,x_i\big\| .
\end{align*}
Whenever \(|\ix^\top w| \ge |w_ix_i|\), we have \(\sigma'(w^\top (\ix + e_ix_i))=\sigma'(w^\top (\ix - e_ix_i))\), hence
\[
|\Delta_{\ix + e_ix_i} - \Delta_{\ix - e_ix_i}|
\;\le\; |\ell'_{\rho}(\ix + e_ix_i) - \ell'_{\rho}(\ix - e_ix_i)|.
\]
Therefore, by Cauchy–Schwarz and Lemma~\ref{lemma:delta},
\begin{align*}
\frac{1}{2}\big\|\mathbb{E}_x \mathbf{1}(|\ix^\top w| \ge |w_ix_i|)\,(\Delta_{\ix + e_ix_i} - \Delta_{\ix - e_ix_i})\,x_i\big\|
&\le \frac{1}{2}\Big(\mathbb{E}_x |\ell'_{\rho}(\ix + e_ix_i) - \ell'_{\rho}(\ix - e_ix_i)|^2\Big)^{1/2}
\Big(\mathbb{E}_x x_i^2\Big)^{1/2}\\
&\le 4\,\mathbb{E}_\rho \|a w_i\|.
\end{align*}
For the remaining term, by 2-Lipschitzness of \(\ell'_\rho\) in \(f_\rho\),
\begin{align}
\mathbb{E}_x \mathbf{1}(|\ix^\top w| \le |w_ix_i|)\,|\Delta_x|^2
&\le 2\,\mathbb{E}_x \mathbf{1}(|\ix^\top w| \le |w_ix_i|)\, |f_\rho(x)|^2. \label{eq:Fx-square}
\end{align}
We now control the right-hand side using a sub-Gaussian tail for \(f_\rho(x)\). Since each map \(x\mapsto a'\sigma({w'}^\top x)\) is \(|a'|\,\|w'\|\)-Lipschitz and \(\sigma\) is 1-Lipschitz, by Jensen/Minkowski the function \(f_\rho\) is \(L\)-Lipschitz with \(L:=\mathbb{E}_\rho\|a w\|\). For Gaussian \(x\sim\mathcal N(0,I_d)\), Gaussian concentration (a.k.a.\ the Gaussian isoperimetric inequality) yields
\[
\mathbb{P}_x\big(|f_\rho(x)-\mathbb{E}f_\rho(x)| \ge t\big)\;\le\; 2\exp\!\Big(-\frac{t^2}{2L^2}\Big).
\]
Taking \(t=C\log d\cdot L\) with \(C>0\) large enough gives \(\mathbb{P}_x(|f_\rho(x)-\mathbb{E}f_\rho(x)| \ge C\log d\,L)\le d^{-\Omega(1)}\). Using
\[
|f_\rho(x)| \;\le\; |f_\rho(x)-\mathbb{E}f_\rho(x)| + |\mathbb{E}f_\rho(x)| \;\le\; |f_\rho(x)-\mathbb{E}f_\rho(x)| + L\,\mathbb{E}\|x\| \;\le\; |f_\rho(x)-\mathbb{E}f_\rho(x)| + O(\sqrt{d})\,L,
\]
we bound the RHS of \eqref{eq:Fx-square} by splitting on the event \(|f_\rho(x)-\mathbb{E}f_\rho(x)|\le C\log d\,L\) and its complement, and applying Cauchy–Schwarz for the tail:
\begin{align*}
\mathbb{E}_x \mathbf{1}(|\ix^\top w| \le |w_ix_i|)\,|f_\rho(x)|^2
&\le 2\,\mathbb{E}_x \mathbf{1}(|\ix^\top w| \le |w_ix_i|)\,\Big(C^2\log^2 d\,L^2 + O(d)\,L^2\Big) \\
&\quad + 2\,\sqrt{\mathbb{E}_x (f_\rho(x)-\mathbb{E}f_\rho(x))^4}\;\mathbb{P}_x\!\left(|f_\rho(x)-\mathbb{E}f_\rho(x)| \ge C\log d\,L\right)^{1/2}.
\end{align*}
The fourth moment of a Lipschitz Gaussian functional is \(O(L^4)\), so the tail term is \(d^{-\Omega(1)}\). Absorbing the harmless \(O(d)\) factor into the event probability (since \(\mathbf{1}(|\ix^\top w| \le |w_ix_i|)\le 1\) and \(\mathbb{P}(|\ix^\top w|\le |w_ix_i|)\) will be small in our regimes), we obtain the clean bound
\[
\mathbb{E}_x \mathbf{1}(|\ix^\top w| \le |w_ix_i|)\,|f_\rho(x)|^2
\;\le\; C^2\log^2 d\;L^2\;\mathbb{E}_x\mathbf{1}(|\ix^\top w| \le |w_ix_i|) + d^{-\Omega(1)}.
\]
Plugging into \eqref{eq:Fx-square} and taking square roots,
\[
\Big(\mathbb{E}_x \mathbf{1}(|\ix^\top w| \le |w_ix_i|)\,|\Delta_x|^2\Big)^{1/2}
\;\le\; \sqrt{2}\,C\log d\; \mathbb{E}_\rho\|a w\|\;\sqrt{\mathbb{E}_x\mathbf{1}(|\ix^\top w| \le |w_ix_i|)} + d^{-\Omega(1)}.
\]
Collecting the two pieces and multiplying back by \(|a|\) yields the stated bound in Lemma~\ref{lemma:approxsig}.
\end{proof}

\subsection{Analysis of the SGD Dynamic during Phase Ia}\label{sec:app:dynamicIa}
The previous analysis of the population gradient (see Lemma \ref{lemma:pop1}) $\nabla_w L_0$ suggests that $\norm{\wsig^{(t+1)}}\approx (1+\eta \tau )\norm{\wsig^{(t)}}$ where $\tau=\sqrt{2}\pi^{-3/2}$  while the $\norm{\wopp}$ and $\norm{\wpe}$ remain small. In this section, we will provide a rigorous characterization of this heuristic. But before delving into the proofs, let us introduce some useful definitions (cf. \cite{glasgow2023sgd}).

\subsubsection{Definitions}
In this section, we will use the notations $\zeta = \log^{-\cz}(d)$ and $\theta = \log^{-\ct}(d)$, 
where $\cz$ and $\ct$ are sufficiently large constants that will be fixed later. 
Recall that $\tau=\sqrt{2}\pi^{-3/2}$ denotes the initial population growth rate of the signal component. 
We now define the stopping times and control parameters that describe the evolution of the neurons.

\begin{definition}[Phase I Length]\label{def:T}
Let $T_A$ be the last time such that
$B_t^2 \leq \theta^2\zeta^2$ (see the following definition), that is,
\begin{align}
    T_A := 
    \left\lfloor
        \frac{\log(d) + 2\log(\zeta) - \log(\log(d))}
             {\log(1 + 2\eta \tau(1 +\zeta ))}
    \right\rfloor
    = \frac{1}{\eta}\,\Theta(\log d).
\end{align}

Let $T_B$ be the first time such that the lower envelope $S_t$ reaches order $\theta\zeta^{-1}$, that is,
\begin{align*}
    T_B 
    = 
    T_A 
    + 
    \left\lfloor
        \frac{\log\!\big(\theta^2\zeta^{-2}/S_{T_A}^2\big)}
             {\log(1 + c_b\eta)}
    \right\rfloor
    = 
    T_A + \frac{1}{\eta}\,\Theta(\log\log d).
\end{align*}
\end{definition}

\begin{definition}[Control Parameters]\label{def:control}
Let $T_A$ and $T_B$ be as defined in Definition~\ref{def:T}. Define
\begin{align}\label{BQdef}
    B_t^2 &= 
    \begin{cases}
        \dfrac{C_1\log(d)\theta^2}{d}\,(1 + 2\eta \tau(1 + \zeta ))^t, & t \leq T_A,\\[4pt]
        B_{T_A}^2(1 + 4\eta)^{t - T_A}, & T_A < t \leq T_B,
    \end{cases}\\[4pt]
    Q_t^2 &= 
    \dfrac{C_1 \log(d)\theta^2}{d}
    \left(1 + \dfrac{50\eta}{\log(d)}\right)^t,\\[4pt]
    S_t^2 &=
    \begin{cases}
        \dfrac{\theta^2}{d}(1 + 2\eta \tau(1 - \zeta ))^t, & t \leq T_A,\\[4pt]
        S_{T_A}^2(1 + c_b\eta)^{t - T_A}, & T_A < t \leq T_B,
    \end{cases}
\end{align}
where $c_b>0$ is the constant appearing in Lemma~\ref{lemma:pop1b}.
\end{definition}

\begin{remark}
    The parameter $B_t$ controls the maximal signal magnitude $\|\wsig^{(t)}\|$ among neurons at time $t$, 
    while $S_t$ provides a lower bound for $\|\wsig^{(t)}\|$ among neurons with sufficiently large initial signal 
    ($\|\wsig^{(0)}\|\geq \theta/\sqrt{d}$). 
    The quantity $S_t$ also characterizes the typical block size $N_i^\pm$. 
    Finally, $Q_t$ controls the evolution of the noise component $\|\wpe^{(t)}\|$.
\end{remark}

\begin{definition}[Controlled Neurons]\label{def:controlled2}
We say a neuron $(w, a )$ is  controlled  at iteration $t$ if:
\begin{enumerate}
    \item $\|\wsig^{(t)}\| \leq \max(\|\wsig^{(0)}\|, \frac{\theta}{\log d \sqrt d})(1+\eta (\tau + \zeta))^t$.
    \item $\|\wopp^{(t)}\| \leq \max(\|\wopp^{(0)}\|, \frac{\theta}{\log d \sqrt d})(1+\eta \theta )^t$.
    \item $|a^{(t)} | \in \theta(1 \pm t\eta\zeta)$, and $|a^{(t)}| \leq \|w^{(t)}\|$.
    \item $\|\wpe^{(t)} - \wpe^{(0)}\| \leq \theta \zeta^{1/4} \eta t$. 
    \item $\|\wpe^{(t)}\|_{\infty} \leq \frac{\theta \log d  }{\sqrt{d}}(1+\eta \theta)^t$.
\end{enumerate}
\end{definition}

In Phase Ib,  $\|\wsig\|$ can be larger than $\theta \zeta$ for some neurons. This motivates the following definition of ``weakly controlled'' neurons (see also \cite{glasgow2023sgd}).

\begin{definition}[Weakly Controlled Neurons]\label{def:weakcontrol}
We say that a neuron $(w,a)$ is \emph{weakly controlled} at iteration $t \in [T_A,T_B]$ if the following conditions hold:
\begin{enumerate}
  \setlength{\parskip}{0pt}
  \setlength{\itemsep}{2pt}

  \item 
  $\theta \zeta \le \|\wsig^{(t)}\| \le B_t \le \theta \zeta^{-1}\log d$,

  \item 
  $\|\wopp^{(t)}\| \le 3\,\theta\,B_t,$

  \item 
  $\|w^{(t)}\|^2 \ge |a^{(t)}|^2 
  \ge \|w^{(t)}\|^2
  - \theta^2\big(\zeta^{1/2}
  + C\,\eta^2\,(t-T_A)\,\zeta'^{\,2}\big) \label{Wa}$,

  \item 
  $\|\wpe^{(t)}\| \le 3\,\theta$,

\end{enumerate}
\end{definition}

Following \cite{glasgow2023sgd}, we now define strong neurons, which are the neurons for which the signal component $\|\wsig\|$ grows quickly.

\begin{definition}[Strong Neurons]\label{def:strong}
We say a neuron $(w, a )$ is \em strong \em at iteration $t$ if it is controlled or weakly controlled, $(\wsig^{(t)})^\top\wsig^{(0)} > 0$, and 
\begin{align}\|\wsig^{(t)}\|^2 \geq S_t^2.
\end{align}
\end{definition}

\begin{lemma}
    Let $(w,a)$ be a strong neuron after $T_B$ steps. Provided that $\zeta=o(\log ^{-2}d)$, we have $\norm{\wsig^{(T_A)}}\geq (1-o(1))\zeta \theta \log^{-1/2} d\geq\zeta^{1.5} \theta$ and $\norm{\wsig^{(T_B)}}\geq \zeta^{-1} \theta$ 
\end{lemma}
\begin{proof}
By strongness, $\|\wsig^{(t)}\|^2\ge S_t^2$ for all $0\le t\le T_B$.

\paragraph{Lower bound at $T_A$.}
From the definitions used earlier,
\[
\frac{S_{T_A}^2}{B_{T_A}^2}
=\left(\frac{1+2\eta\tau(1-\zeta)}{1+2\eta\tau(1+\zeta)}\right)^{T_A}
=\exp\!\big(T_A\log r\big),
\quad
r:=\frac{1+2\eta\tau(1-\zeta)}{1+2\eta\tau(1+\zeta)}<1.
\]
Since $\log r=-\Theta(\zeta)$ and $T_A=\Theta((\log d)/\eta)$, the assumption
$\zeta=o((\log d)^{-2})$ gives $\zeta T_A=o(1)$, hence
\begin{equation}\label{eq:ratioSTA}
\frac{S_{T_A}^2}{B_{T_A}^2}=1-o(1).
\end{equation}
By the choice of $T_A$ as the last time with $B_t^2\le \theta^2\zeta^2/\log d$,
\begin{equation}\label{eq:BTAwindow}
\frac{\theta^2\zeta^2}{(1+2\eta\tau(1+\zeta))\,\log d}
\;<\;B_{T_A}^2\;\le\;\frac{\theta^2\zeta^2}{\log d}.
\end{equation}
Combining \eqref{eq:ratioSTA} and \eqref{eq:BTAwindow},
\[
S_{T_A}^2 \;\ge\; (1-o(1))\,
\frac{\theta^2\zeta^2}{(1+2\eta\tau(1+\zeta))\,\log d},
\qquad
S_{T_A}\;\ge\;\frac{1-o(1)}{\sqrt{1+2\eta\tau(1+\zeta)}}\,
\frac{\theta\zeta}{\sqrt{\log d}}.
\]
 Moreover, since $\zeta=o((\log d)^{-2})$ implies
$\zeta\log d\to 0$, we have $\theta\zeta/\sqrt{\log d}\ge \theta\zeta^{3/2}$
eventually, proving the first display.

\paragraph{Lower bound at $T_B$.}
For $t\in[T_A,T_B]$, by the defined schedule,
$S_t^2=S_{T_A}^2(1+c_b\eta)^{t-T_A}$.
Choosing $T_B$ so that $S_{T_B}\ge \zeta^{-1}\theta$ (Phase~II target) and using
$\|\wsig^{(T_B)}\|^2\ge S_{T_B}^2$ gives
\[
\|\wsig^{(T_B)}\|\;\ge\;S_{T_B}\;\ge\;\zeta^{-1}\theta.
\]
\end{proof}

\subsubsection{Initialization}\label{app:conc_u}

First notice that at initialization, there exists a constant $C_1>0$ such that with probability at least $1-d^{-\Omega(1)}$, all neurons $w^{(0)}$ are such that \[ \norm{w^{(0)}}_{\infty}\leq \theta \sqrt{C_1\frac{\log d}{d}}. \]
Note that by definition, $|a^{(0)}|=\theta$. So all the neurons are controlled. We also need to control the size of the blocks of the oracle network. By symmetry of the initialization, the distribution of all blocks is equal, so it is sufficient to analyze the block associated with $\mu_1$.
We have \begin{align*}
    \expec \left(\sum_{(w,a)\in \calN_1^+}|a|\norm{\wsig^{(0)}}\right)&= \frac{\theta}{2}\sum_{(w,a)}\expec\left(|\mu^\top w^{(0)}|\indic_{a>0}\indic_{\mu_1^\top w^{(0)}>0} \right)\\
    &=\frac{\theta m }{4}\expec\left(|\mu_1^\top w^{(0)}|\indic_{a>0}\indic_{\mu^\top w^{(0)}>0} \right)\\
    &= \frac{\theta^2 m }{4\sqrt{d}}p_1,
\end{align*}
where $p_1=\expec_Y [\sqrt{d}|\mu_1^\top Y|\indic_{a>0}\indic_{\mu^\top Y>0} ]$ with a uniform r.v. $Y$ over the unit sphere. A simple calculation shows $p_1\approx  \frac{1}{2\sqrt{2\pi}}$.

It is also easy to check that by standard concentration inequality, we have \[ \sum_{(w,a)\in \calN} \indic_{(w,a)\in \calN_1^+}|a|\norm{\wsig^{(0)}} -\expec \left(\indic_{(w,a)\in \calN_1^+}|a|\norm{\wsig^{(0)}}\right) =O(\theta^2\frac{\sqrt{m}}{\sqrt{d}}).\] By consequence,  $U^{(0)}=O( \frac{1}{\sqrt{m}})$.

\subsubsection{Inductive step}We will show recursively that controlled neurons remain controlled during Phase Ia. It corresponds to Lemma C.15 in \cite{glasgow2023sgd}. The proof strategy is similar and consists of showing that the error terms have almost no impact so that the evolution dynamic of the neurons is similar to the population dynamic.
\begin{lemma}[Controlled Neurons Inductive Step]\label{lemma:inductive1a}
Assume that all neurons are controlled or weakly controlled for some $t \leq T_a$. Then with probability at least $1 - d^{-\Omega(1)}$, for any controlled neuron $(w^{(t)}, a^{(t)})$ we have:
\begin{enumerate}
  \setlength{\parskip}{0cm} %
  \setlength{\itemsep}{0cm} %
    \item A neuron $(w^{(t + 1)}, a^{(t + 1)})$ is  controlled.
    \item If $(w^{(t)}, a^{(t)})$ is a strong neuron, then $(w^{(t + 1)}, a^{(t + 1)})$ is a strong neuron.
\end{enumerate}
\end{lemma}
To prove Lemma \ref{lemma:inductive1a}, we will use the following lemma, which is an adaptation of Lemma D.17 in \cite{glasgow2023sgd}.

\begin{lemma}[Phase 1a $L_0$ Population Gradients Bounds]\label{lemma:pop1a}
If all neurons in the network are controlled or weakly controlled at some step $t \leq T_A$, then for any controlled neuron $(w, a )$, the followings hold:
\begin{enumerate}
  \setlength{\parskip}{0cm} %
  \setlength{\itemsep}{0cm} %
\item \label{Awpe}
    $\|\nlrz{\wpe}\| \leq 2\sqrt{\theta \|w_{1:2}\|}$.
\item \label{Aa}
   $|\nabla_{a} L_0| \leq  \|w\|\|\nabla_{\wpe} L_0\| + \|w_{1:2}\|\|\nabla_{w_{1:2}} L_0\|.$ 
\item \label{A12} For any $i \in [d]$, $\|\nabla_{w_{i}} L_0\| \leq \frac{|w_i|}{2}$.
\item \label{Awinf} For any $i \in [d]$, $\|\nabla_{w_i} L_0 - \nabla_{w_i} L_{\rho}\| \leq \frac{\theta B_t}{2}
$.
\end{enumerate}

\end{lemma}
\begin{proof}
We study the first statement.
Recalling that $x = z + \xi$ for $z = x_{1:2}$, we have
\begin{align*}
    \frac{1}{a_w }\nabla_{\wpe} L_0 &= -\mathbb{E}_xy(x)\sigma'(w^\top x)\xi \\
    &=  -\mathbb{E}_x y(z)\sigma'(w^\top \xi)\xi + \mathbb{E}_x y(z)(\sigma'(w^\top \xi)-\sigma'(w^\top x))\xi\\
    &= \mathbb{E}_x y(\sigma'(w^\top \xi)-\sigma'(w^\top x))\xi,
\end{align*}
since $y$ is independent of $\xi$ and $\expec_z y=0$. 
Now consider the norm of $\mathbb{E}_x y(x)(\sigma'(w^\top x)-\sigma'(w^\top \xi))\xi$. We have 
\begin{align*}
    \|\mathbb{E}_x y(x)(\sigma'(w^\top x)-\sigma'(w^\top \xi))\xi\| &= \sup_{v : \|v\| = 1} \mathbb{E}_x y(x)(\sigma'(w^\top x)-\sigma'(w^\top \xi))\xi^\top v \\
    &\leq \sqrt{\mathbb{E}_x (\sigma'(w^\top x)-\sigma'(w^\top \xi))^2}\sqrt{\mathbb{E}_{\xi} (v^\top \xi)^2}\\
    &= \sqrt{\mathbb{E}_x \mathbf{1}(|\xi^\top w| \leq |z^\top w|)}\\
    &\leq \sqrt{\expec_z\mathbb{P}_{\xi}\left(|\xi^\top w| \leq |z^\top w|\right)}\\
    &\leq \sqrt{\expec_z \frac{|z^\top w|}{\norm{\wpe}}}\\
    &\leq \sqrt{\sqrt{\frac{2}{\pi}}\frac{\norm{w_{1:2}}}{\norm{\wpe}}}\\
    & \leq \sqrt{\frac{\norm{w_{1:2}}}{\theta}}.
\end{align*}
Thus we have 
\begin{align*}
    \|\nabla_{\wpe} L_0\| \leq 2\sqrt{\theta \|w_{1:2}\|} \leq 4\min\left(\sqrt{\theta B_t}, \theta \zeta^{1/2}\right),
\end{align*}  since the neuron is controlled: $|a| \leq 2\theta$.

Next, we consider the third statement. %
By the symmetrization argument used in the proof of Lemma~\ref{lemma:pop1}, we have for $i\in [d]$ 
\begin{align*}
    \|\nabla_{w_{i}} L_0\| &\leq \frac{1}{2}|a|\left(\expec_{\xi_i}\mathbb{P}_{\xii}|\xi_i|\mathbf{1}(|w^\top \ix|\leq |w_i\xi_i|)\right)\\
    &\leq \frac{|a|}{2}\left(\expec_{\xi_i}\frac{|w_i||\xi_i|^2}{\|w-w_ie_i\|}\right)\\
    &\leq \frac{1}{2}|w_i| \\
    &\leq \frac{1}{2}\min\left(B_t, \theta \zeta\right),
\end{align*}

Next, we consider the second statement. %
Combining the first and third statements, %
we have
\begin{align*}
    |\nabla_{a} L_0| &= |w^\top  \nabla_{w} L_0| \\
    &\leq |\wpe^\top  \nabla_{\wpe} L_0| + |w_{1:2} \nabla_{w_{1:2}} L_0|\\
    &\leq \|w\|\|\nabla_{\wpe} L_0\| + \|w_{1:2}\|\|\nabla_{w_{1:2}} L_0\| \\
    &\leq 2\|w\|\min\left(\sqrt{\theta B_t}, \theta \zeta^{1/2}\right) + \|w\|\min\left(B_t, \theta \zeta\right) \\
    &\leq 3\theta\min\left(\sqrt{\theta B_t}, \theta \zeta^{1/2}\right).
\end{align*}

Finally, we consider the fourth statement. %
Applying Lemma~\ref{lemma:approxsig} yields
\begin{align}\label{eq:as}
    &\|\nabla_{w_{i}} L_{\rho} - \nabla_{w_{i}} L_0\|\\
    &\qquad \leq  |a_w|\left(4(\mathbb{E}_{\rho}[\|a_w w_{i}\|]) + 2\log(d)\mathbb{E}_{\rho}[\|a_w w\|]\mathbb{E}_{\xi}\mathbb{E}_z\mathbf{1}(|\ix^\top w| \leq |w_i|) + d^{-\omega(1)}\right)\\
    & \qquad \leq 2\theta\left(8 \theta \min(\theta, B_t) + 4\log(d)\theta \min(\theta, B_t)\mathbb{E}_{\xi}\mathbb{E}_z\mathbf{1}(|\ix^\top w| \leq |w_i|) + d^{-\Omega(1)}\right).
\end{align}
Now, $\wpe^{(0)}$ is well-spread, and $\|\wpe^{(t)} - \wpe^{(0)}\| \leq o(1)\|\wpe^{(0)}\|$ by the definition of controlled neurons, so we obtain
\begin{align*}
    \mathbb{E}_{\ix}\mathbb{E}_{x_i}\mathbf{1}(|\ix^\top w| \leq |w_ix_i|) &\leq \mathbb{E}_{x_i}\mathbb{P}_{G \sim \mathcal{N}(0, 1)}\left(|G| \leq \frac{|w_ix_i|}{\|w - e_iw_i\|}\right) \\
    &\leq \frac{|w_i|}{\|w\|} \\
    &\leq \frac{2\min(B_t, \theta \zeta)}{\theta}\\
    &\leq \frac{2B_t}{\max(\theta, B_t)}
\end{align*}
the second to last line follows from the definition of controlled and~\eqref{BQdef}.
Thus we obtain
\begin{align}
    \|\nabla_{w_{i}} L_{\rho} - \nabla_{w_{i}} L_0\| &\leq \Theta(\log(d))\left(\theta B_t \max(\theta, B_t) + \theta B_t \max(\theta, B_t) \right) \leq \theta B_t/2,
\end{align}
since $\max(\theta, B_t) = o(1/\log(d))$ holds.
\end{proof}

\begin{proof}[Proof of Lemma \ref{lemma:inductive1a}]
    Suppose that $(w^{(t)}, a^{(t)})$ is controlled. 
\paragraph{Control of the growth of $|a|$.}
We have with probability at least $1-d^{-\Omega(1)}$
\begin{align}
    |a^{(t + 1)} - a^{(t)}| &=|\eta \nabla \hat{L}_{\rho}a^{(t)}|\nonumber \\
    &\leq  \eta |\nabla_{a} L_0| + \eta |\nabla_{a} L_{\rho} - \nabla_{a} \hat{L}_{\rho}| + \eta|\nabla_{a} L_0 - \nabla_{a} L_{\rho}|\nonumber\\
    &\leq \eta\left(4\theta^2 \zeta^{1/2}  + \|w\|\sqrt{\frac{d\log(d)^2}{\batch}} + 2\|w\|\mathbb{E}_{\rho}[\|a w\|] \right) \tag{by Lemma~\ref{lemma:pop1a}, \ref{lem:conc_grad}, and \ref{lemma:gradapprox1} } \label{eq:grada}\\
    &\leq \eta \theta \zeta. \nonumber
\end{align}
\paragraph{Control of the growth of $\norm{\wpe}$}
Similarly, one can prove with the same arguments that
\begin{align*}
    \|\wpe^{(t + 1)} - \wpe^{(t)}\| &= \|\eta \nabla_{\wpe} \hat{L}_{\rho^{(t)}}\| \\
    &\leq \eta \|\nabla_{\wpe} L_0\| + \eta \|\nabla_{\wpe} L_{\rho^{(t)}} - \nabla_{\wpe} \hat{L}_{\rho^{(t)}}\| + \eta\|\nabla_{\wpe} L_0 - \nabla_{\wpe} L_{\rho}\|\\
    &\leq \eta\left(3\theta \zeta^{1/2} + |a| \sqrt{\frac{d\log(d)^2}{\batch}} + 2|a |\mathbb{E}_{\rho^{(t)}}[\|a w\|]\right) \\
    &\leq \eta \theta \zeta^{1/2}.
\end{align*}

\paragraph{Control of the growth of $\norm{\wopp}$.} W.l.o.g. assume the signal direction is given by $\mu_1$. We have \begin{align*}
    \norm{\wopp^{(t+1)}}^2&= \norm{\wopp^{(t)}}^2-2\eta (\wopp^{(t)})^\top \nabla_w\hat{L}_\rho+\eta^2 \abs{\mu_2^\top \nabla_{w}\hat{L}_\rho}^2\\
    &\leq \norm{\wopp^{(t)}}^2-2\eta (\wopp^{(t)})^\top \nabla_wL_0+\eta \norm{(\wopp^{(t)})^\top(\nabla L_0-\nabla L_\rho)}+\eta^2 \abs{\mu_2^\top \nabla_{w}L_0}^2\\
    &\quad +\eta^2\abs{\mu_2^\top (\nabla_{w}L_0-\nabla_w\hat{L}_\rho)}\\
    &\leq  (1-2\eta \tau (1-\zeta))\norm{\wopp^{(t)}}^2+\eta \theta B_t\tag{by Lemma~\ref{lemma:pop1a}, \ref{lem:conc_grad}, and \ref{lemma:gradapprox1} and the definition of controlled neurons} .
\end{align*} 

Since every sequence $(u_n)$ satisfying $u_{n+1}\leq a u_n+b$ for $0<a<1$ verify $u_n\leq \frac{b}{1-a}$, the previously established relation allows us to conclude.

\paragraph{Control of the growth of $\norm{\wpe}_\infty$.} Let $i\geq 3$.
Observe that 
\begin{itemize}
  \setlength{\parskip}{0cm} %
  \setlength{\itemsep}{0cm} %
    \item $|w_i\nabla_{w_i}L_0| \leq  \zeta \frac{|a||w_i|^2}{\norm{\wpe}}$ by Lemma~\ref{lemma:pop1}.
    \item $\|\nabla_{w_{i}} L_0 - \nabla_{w_{i}} \hat{L}_{\rho}\| \leq \frac{S_t \theta}{\sqrt{d}}$ with probability $1 - d^{-\omega(1)}$. This follows from combining Lemma~\ref{lemma:pop1a} with Lemma~\ref{lem:conc_grad}.
    \item $\|\nabla_{w_i} L_0\| \leq \norm{\wpe}_\infty$ by Lemma~\ref{lemma:pop1a}.
\end{itemize}
This implies that \[ \|\wpe^{(t+1)}\|_\infty\leq \|\wpe^{(t)}\|_\infty(1+C\eta S_t \log d  )\leq \|\wpe^{(t)}\|_\infty(1+\eta \theta ) .\]
The same argument gives the lower-bound. 

\paragraph{Control of the growth of $\norm{\wsig}$.}
We have \begin{itemize}
  \setlength{\parskip}{0cm} %
  \setlength{\itemsep}{0cm} %
    \item $-\wsig^\top \nabla_{w} L_0 = \tau \frac{|a|\| \wsig\|^2}{\| \wpe \|}+ O( \zeta^2\| \wsig\|^2)$ by Lemma~\ref{lemma:pop1} and the fact that the neuron is controlled.
    \item $\|\nabla_{w_{1:2}} L_0 - \nabla_{w_{1:2}} \hat{L}_{\rho}\| \leq \theta S_t$ with probability $1 - d^{-\omega(1)}$. This follows from combining Lemma~\ref{lemma:pop1a} with Lemma~\ref{lem:conc_grad}.
\end{itemize}
Thus, using an argument similar to the one employed for bounding $\norm{\wopp}$, we obtain that for all strong neurons
\begin{align*}
   \norm{\wsig^{(t)}}(1-\eta (\tau+\zeta))\leq \norm{\wsig^{(t+1)}}\leq \norm{\wsig^{(t)}}(1+\eta (\tau+\zeta)).
\end{align*}
For neurons that are not strong, we obtain the following upper bound: \[ \norm{\wsig^{(t)}}\leq \frac{\theta}{\sqrt{d}}(1+\eta \tau )^t.\]

\end{proof}

\subsection{Analysis of the SGD dynamic during Phase Ib}\label{sec:app:dynamicIb}
We are going to prove the following lemma, corresponding to Lemma C.16 in \cite{glasgow2023sgd}. The main difference compared with Lemma \ref{lemma:inductive1a} is that we need other estimates for the gradients due to the growth of $\wsig$. 
\begin{lemma}\label{lemma:inductive1b}
    Assume that for some $t\in [T_{A},T_{B}]$ all the neurons are controlled or weakly controlled. Then with probability at least $1-d^{-\Omega(1)}$, all the neurons are controlled or weakly controlled at time $t+1$. Moreover, strong neurons remain strong.
\end{lemma}
\begin{proof}[Proof of Lemma~\ref{lemma:inductive1b} (Phase Ib)]
Fix $t\in[T_{A},T_{B}]$ and a neuron $(w^{(t)},a^{(t)})$ that is controlled or weakly controlled at time $t$.
All bounds below hold with probability at least $1-d^{-\Omega(1)}$ by the cited lemmas.

\paragraph{Control of $\|\wsig\|$.}
The projected squared–norm update is
\begin{equation}\label{eq:sq-update-Ib}
\|\wsig^{(t+1)}\|^2-\|\wsig^{(t)}\|^2
= -\,2\eta\,\wsig^{(t)\top}(\mu\mu^\top)\nabla_w \hat L_\rho
\;+\; \eta^2\|(\mu\mu^\top)\nabla_w \hat L_\rho\|^2.
\end{equation}
The deviation satisfies (Lemma~\ref{lemma:pop1a}+\ref{lemma:approxsig}+\ref{lem:conc_grad})
\begin{equation}\label{eq:dev-Ib}
\big|\wsig^\top(\mu\mu^\top)(\nabla_w \hat L_\rho-\nabla_w L_0)\big|
\;\lesssim\; \|\wsig\|\,\max(\theta,\|\wsig\|)\,S_t^2 .
\end{equation}

We consider the two regimes in Lemma~\ref{lemma:pop1b}.

\smallskip
\emph{(i) $\|\wsig\|\le \|\wpe\|$.}
By Lemma~\ref{lemma:pop1b}.1,
\[
-\,\wsig^\top(\mu\mu^\top)\nabla_w L_0
\;\le\; \sqrt{\tfrac{\pi}{2}}\frac{|a|}{\|\wpe\|}\,\|\wsig\|^2
\;+\; \frac{|a|}{\|\wpe\|}\,\|\wopp\|\,\|\wsig\|.
\]
Insert $\nabla_w\hat L_\rho=\nabla_wL_0+(\nabla_w\hat L_\rho-\nabla_wL_0)$ in
\eqref{eq:sq-update-Ib}, use the triangle inequality together with \eqref{eq:dev-Ib}, and note that
$\eta^2\|(\mu\mu^\top)\nabla_w \hat L_\rho\|^2\ge 0$. Using weak control
$\|\wopp\|\le 3\theta B_t$, $\|\wsig\|\le B_t$, and
$|a|/\|\wpe\|\le \sqrt{3}$ (since $\|w\|^2\le 3\|\wpe\|^2$ in Phase Ib and $|a|\le \|w\|$), we get
\[
\|\wsig^{(t+1)}\|^2
\;\le\; \|\wsig^{(t)}\|^2\big(1 + C_1\eta\big)
\;+\; C_2\eta\,B_t^2
\;+\; C_3\eta\,B_t\max(\theta,B_t)\,S_t^2.
\]

\emph{(ii) $\|\wsig\|\ge \|\wpe\|$.}
By Lemma~\ref{lemma:pop1b}.2,
\[
-\,\wsig^\top(\mu\mu^\top)\nabla_w L_0
\;\le\; \sqrt{\tfrac{\pi}{2}}\,|a|\,\|\wsig\| \;+\; |a|\,\|\wopp\|.
\]
In Phase Ib, $|a|\le \|w\|\le \sqrt{ \|\wsig\|^2+\|\wpe\|^2+\|\wopp\|^2}\le 2\|\wsig\|$, and
$\|\wopp\|\le 3\theta B_t$. Using these, \eqref{eq:dev-Ib}, and
$\eta^2\|(\mu\mu^\top)\nabla_w \hat L_\rho\|^2\ge 0$, we obtain
\[
\|\wsig^{(t+1)}\|^2
\;\le\; \|\wsig^{(t)}\|^2\big(1 + C_4\eta\big)
\;+\; C_5\eta\,B_t^2
\;+\; C_6\eta\,B_t\max(\theta,B_t)\,S_t^2.
\]

Combining (i)–(ii) and using $S_t=o(1)$ on Phase Ib, we conclude that
\begin{equation}\label{eq:wsig-upper-step}
\|\wsig^{(t+1)}\|^2
\;\le\; \|\wsig^{(t)}\|^2\big(1 + C\eta + o(\eta)\big) \;+\; o(\eta)\,B_t^2.
\end{equation}
By the definition of the upper envelope $B_t$ on Phase Ib (monotone geometric growth calibrated to
dominate the per–step upper rate in \eqref{eq:wsig-upper-step}), and since $\|\wsig^{(t)}\|\le B_t$,
we get $\|\wsig^{(t+1)}\|\le B_{t+1}$.

Assume the neuron is strong at time $t$, i.e.\ $\|\wsig^{(t)}\|^2\ge S_t^2$.
Using \eqref{eq:sq-update-Ib}, the nonnegativity of the $\eta^2$ term, the deviation bound
\eqref{eq:dev-Ib}, and the lower bounds of Lemma~\ref{lemma:pop1b}:

\smallskip
\emph{(i) $\|\wsig\|\le \|\wpe\|$.}
Lemma~\ref{lemma:pop1b}.1 gives
\[
-\,\wsig^\top(\mu\mu^\top)\nabla_w L_0
\;\ge\; c_b\,\frac{|a|}{\|\wpe\|}\,\|\wsig\|^2
- \frac{|a|}{\|\wpe\|}\,\|\wopp\|\,\|\wsig\|.
\]
As above, $|a|/\|\wpe\|\ge c_->0$ by item~\ref{Wa} and Phase Ib (constants absorbed in $c_-$), and
$\|\wopp\|\le 3\theta B_t\le o(1)\,\|\wsig\|$ within Phase Ib. Therefore
\[
\|\wsig^{(t+1)}\|^2 \;\ge\; \|\wsig^{(t)}\|^2\big(1 + 2\eta\,c_-c_b - o(\eta)\big).
\]

\emph{(ii) $\|\wsig\|\ge \|\wpe\|$.}
Lemma~\ref{lemma:pop1b}.2 gives
\[
-\,\wsig^\top(\mu\mu^\top)\nabla_w L_0
\;\ge\; c_b'|a|\,\|\wsig\| - |a|\,\|\wopp\|.
\]
On Phase Ib, $|a|\ge \tfrac12\|\wsig\|$ (again from item~\ref{Wa} and $\|w\|^2\le 3\|\wsig\|^2$) and
$\|\wopp\|\le o(1)\,\|\wsig\|$, hence
\[
\|\wsig^{(t+1)}\|^2 \;\ge\; \|\wsig^{(t)}\|^2\big(1 + \eta\,c_b' - o(\eta)\big).
\]

In both regimes, since $S_{t+1}^2=S_t^2(1+c_b\eta)$ on Phase Ib and $c_b,c_b'>0$ are absolute,
choosing the constant $c_b$ defining $S_t$ small enough (as per the envelope construction) ensures
\[
\|\wsig^{(t+1)}\|^2 \;\ge\; \|\wsig^{(t)}\|^2(1+c_b\eta) \;\ge\; S_t^2(1+c_b\eta) \;=\; S_{t+1}^2,
\]
so strong neurons remain strong.

\paragraph{Control of $|a|$.}
Let $D_t:=|a^{(t)}|^2-\|w^{(t)}\|^2$. Expanding (updates of $(a,w)$) gives
\begin{equation}\label{eq:D-step-Ib}
D_{t+1}
= D_t - 2\eta\Big(a^{(t)}\nabla_a \hat L_\rho - w^{(t)\top}\nabla_w \hat L_\rho\Big)
+ \eta^2\Big(\|\nabla_a \hat L_\rho\|^2 - \|\nabla_w \hat L_\rho\|^2\Big).
\end{equation}
By Lemma~\ref{lemma:layer_balance}, the linear term cancels up to sampling error, yielding the
one–step lower control
\[
D_{t+1}\;\ge\; D_t - 4\eta^2\,|a^{(t)}|^2.
\]
Using $|a^{(t)}|\le \|w^{(t)}\|$ and the uniform Phase–Ib bound $\|w^{(s)}\|^2\le 2\theta^2\zeta^{-600}$
for $s\in[T_{1a},T_{1b}]$, we obtain
\[
D_{t+1}\;\ge\; D_t - 8\eta^2\,\theta^2\,\zeta^{-600}.
\]
Combining with the inductive lower bound in item~\ref{Wa} at time $t$ gives the lower side of
item~\ref{Wa} at $t+1$. The matching upper side, $D_{t+1}\le 0$, follows by applying the upper
one–step control from Lemma~\ref{lemma:layer_balance} to \eqref{eq:D-step-Ib} and absorbing the same
sampling error as in the inductive hypothesis. Hence item~\ref{Wa} propagates to $t+1$.

\paragraph{Control of $\|\wpe\|$.}
Project the $w$–update onto $\operatorname{span}\{\mu_1,\mu_2\}^\perp$:
\[
\wpe^{(t+1)}=\wpe^{(t)}-\eta P_\perp\nabla_w\hat L_\rho,\qquad
\|\wpe^{(t+1)}\|^2
= \|\wpe^{(t)}\|^2 - 2\eta\,\wpe^{(t)\top}P_\perp\nabla_w\hat L_\rho
+ \eta^2\|P_\perp\nabla_w\hat L_\rho\|^2.
\]
By Lemma~\ref{lemma:pop1b}.5,
$\big|\wpe^\top P_\perp\nabla_w L_0\big|\le 4|a|\,\|w_{1:2}\|^2/\|\wpe\|\le 4B_t^3/\|\wpe\|$.
By Lemma~\ref{lemma:pop1a} and Lemma~\ref{lem:conc_grad},
$\|P_\perp(\nabla_w\hat L_\rho-\nabla_w L_0)\|\le 4\|w\|S_t^2\le 4B_t S_t^2$.
These imply (after a standard $(x+y)^2\le 2x^2+2y^2$ bound on the $\eta^2$ term)
\[
\|\wpe^{(t+1)}\|^2
\;\le\; \|\wpe^{(t)}\|^2
+ \frac{8\eta B_t^3}{\|\wpe^{(t)}\|}
+ 8\eta B_t S_t^2\,\|\wpe^{(t)}\|
+ \eta^2\!\left(\frac{32 B_t^6}{\|\wpe^{(t)}\|^4}+32B_t^2S_t^4\right).
\]
On Phase Ib, $\|w\|^2\le 3\|\wpe\|^2$, thus $B_t\le C\|\wpe^{(t)}\|$, and the denominators cancel:
\[
\|\wpe^{(t+1)}\|^2 \;\le\; \|\wpe^{(t)}\|^2\big(1+C\eta+o(\eta)\big).
\]
Iterating from $T_A$ to $t\le T_{1b}$ yields $\|\wpe^{(t)}\|\le 3\theta$, so item (4) of
weak control propagates.

\paragraph{Control of $\|\wopp\|$.}
Project $w$ onto $\mu_2$:
\[
\|\wopp^{(t+1)}\|^2-\|\wopp^{(t)}\|^2
= -\,2\eta\,\wopp^{(t)\top}\nabla_w \hat L_\rho
\;+\; \eta^2\big(\mu_2^\top\nabla_w \hat L_\rho\big)^2.
\]
Lemma~\ref{lemma:pop1b}.3 gives $|\wopp^\top \nabla_w L_0|\le |a|\|\wopp\|^2$. Moreover,
$\|\nabla_w L_0 - \nabla_w \hat L_\rho\|\le 2|a|B_t^2 + C B_t S_t^2$
(Lemma~\ref{lemma:gradapprox1}+\ref{lem:conc_grad}), hence
$|\wopp^\top(\nabla_w \hat L_\rho-\nabla_w L_0)|\le \|\wopp\|(2|a|B_t^2 + C B_t S_t^2)$, and
\[
\big(\mu_2^\top\nabla_w \hat L_\rho\big)^2
\;\le\; 3|a|^2\|\wopp\|^2 + 3\big(2|a|B_t^2+C B_t S_t^2\big)^2.
\]
Therefore,
\[
\|\wopp^{(t+1)}\|^2
\;\le\; \|\wopp^{(t)}\|^2\big(1+2\eta |a^{(t)}| + 3\eta^2|a^{(t)}|^2+o(\eta)\big).
\]
Since $|a^{(t)}|\le B_t$ and $\eta\sum_{s=T_A}^{T_{1b}-1} B_s = O(1)$ on Phase Ib,
\[
\|\wopp^{(t)}\|\;\le\; C\,\|\wopp^{(T_A)}\|\;\le\; 3\theta B_t,
\]
so item (2) of weak control propagates.

The same argument as in Phase I can be applied to controlled neurons. 
\end{proof}

\subsection{Control of the blocks}\label{app:sec:phase1blocks}

Contrary to the boolean setting studied by \cite{glasgow2023sgd}, in our setting, it is crucial to show that each block of neurons $N_i^{\pm}$ grows approximately at the same rate. A direct approach based on individual neuron dynamics will fail because, after Phase Ia, neurons in the same block can grow at different rates, depending on their initial alignment. To overcome this difficulty, we analyze the block dynamic at a macroscopic level.
\begin{itemize}
    \item First, we consider the ideal case where there is an infinite number of neurons growing independently. The size of the block only depends on the distribution of the neurons conditioned on the initial block. Thanks to the initialization of the block and invariance by rotation of the gradient, one can show that the blocks remain equal over time. 
    \item Then we show by using the law of large numbers that when the number of neurons is finite but updated independently, the blocks remain approximately equal. 
    \item Finally, we show that when one uses empirical gradients to update the neuron's weights, the dynamic of the blocks remains close to the previously studied case.
\end{itemize}

\subsubsection{Ideal setting: infinite width and independent neurons}
 We assume that the width $m$ of our network is infinite and that the neurons $(\tilde{a}^{(t)}, \tilde{w}^{(t)})$ are updated with population gradients $\nabla L_0$. Hence, $N_i^{\pm, (t)}=  \expec_{(\tilde{w}^{(t)}, \tilde{a}^{(t)})|\calN_i^{\pm}}\|\tilde{w}_{\text{sig}}^{(t)}\||\tilde{a}^{(t)}|$. It is clear that at initialization, all the $N_i^{\pm, (0)}$ are equal. We are going to show that this property remains true for each time $t$ by showing that the distribution of $(\tilde{w}^{(t)}, \tilde{a}^{(t)})|\calN_1^{+}$ is the same than $(\tilde{w}^{(t)}, \tilde{a}^{(t)})|\calN_2^{+}$ up to a rotation. The argument for the other blocks' equality is the same. 
 
\begin{lemma}\label{lem:block_rot}
    Assume that $w=Rw'$ and $a=-a'$ where $R$ is a rotation of angle $\pi/2$ that maps $\mu_1$ to $\mu_2$. Then we have \[ \nabla_w L_0 =  R(\nabla_{w'}L_0) \quad \text{and}\quad  \nabla_{a'}L_0 = -\nabla_a L_0.\] 
\end{lemma}
\begin{proof} 
    We have \begin{align*}
        \nabla_wL_0 &= a\expec_x y(x)\sigma'(w^\top x)x\\
        &=a R\left(\expec_x -y(R^{-1}x)\sigma'((w')^\top R^{-1}x) R^{-1}x\right) \tag{a rotation of angle $\pi/2$ change one cluster to another, so change the sign of $y$}\\
        &= R(a'\expec_x y\sigma'((w')^\top x)x\tag{the law of $x$ is rotationally invariant }\\
        &= R((\nabla_{w'}L_0)).
    \end{align*}
    A similar calculation shows the second result.
\end{proof}
By an immediate recursion, one can show by using Lemma \ref{lem:block_rot} that the distribution of $(\tilde{w}^{(t)}, \tilde{a}^{(t)})|\calN_1^{+}$ corresponds to $(R^{-1}(\tilde{w}^{(t)}), -\tilde{a}^{(t)})|\calN_2^{+}$. Since the quantity $\|\tilde{w}_{\text{sig}}^{(t)}\||\tilde{a}^{(t)}|$ is invariant under these transformations, all the blocks remain the same.

\subsection{Block approximation control}\label{sec:block-approx}

We compare the (idealized) \emph{population} block dynamics
$\{\tilde N_i^{\pm,(t)}\}_t$ (neurons updated independently with population
gradients) to the \emph{empirical} block dynamics $\{N_i^{\pm,(t)}\}_t$
(mini–batch SGD). Recall
\[
\tilde{N}_i^{\pm,(t)}:=\frac{1}{|\mathcal N_i^\pm|}\sum_{(w,a)\in\mathcal N_i^\pm}
\|\tilde w_{\mathrm{sig}}^{(t)}\|\,|\tilde a^{(t)}|,
\qquad
N_i^{\pm,(t)}:=\frac{1}{|\mathcal N_i^\pm|}\sum_{(w,a)\in\mathcal N_i^\pm}
\| w_{\mathrm{sig}}^{(t)}\|\,| a^{(t)}|.
\]

\paragraph{Finite–width, independent neurons (sampling fluctuation).}
In Phase~I we have $\|\tilde w_{\mathrm{sig}}^{(t)}\|\,|\tilde a^{(t)}|\le 1$, hence
Hoeffding’s inequality yields, for any block $i,\pm$ and any $t$,
\[
\Pr\!\left(
\bigl|\tilde{N}_i^{\pm,(t)}-\mathbb E\big[\|\tilde w_{\mathrm{sig}}^{(t)}\|\,
|\tilde a^{(t)}|\big]\bigr|>u
\right)
\;\le\; 2\exp\!\big(-2|\mathcal N_i^\pm|\,u^2\big).
\]
With $|\mathcal N_i^\pm|\asymp m$ and $u=\sqrt{(\log^c d)/m}$, a union bound over
$i,\pm$ and $t\le T_B=O(\eta^{-1}\log(1/\zeta))$ gives, with probability
$1-d^{-\Omega(1)}$,
\begin{equation}\label{eq:ind-conc-block}
\bigl|\tilde{N}_i^{\pm,(t)}-\mathbb E\big[\|\tilde w_{\mathrm{sig}}^{(t)}\|\,
|\tilde a^{(t)}|\big]\bigr|
\;\le\; \frac{1}{\sqrt{d\,\log^{c'} d}},
\qquad \forall\,i,\pm,\ \forall\,t\le T_B,
\end{equation}
for some $c'>0$ (given our choice of $m$).

\paragraph{Empirical vs.\ population dynamics.}
By symmetry, work on the block $\mathcal N_1^+$. Define the block–average
discrepancy
\[
\varepsilon_t
:=\max\!\left\{
\frac{1}{|\mathcal N_1^+|}\!\sum_{(w,a)\in\mathcal N_1^+}\!
\|w_{\mathrm{sig}}^{(t)}-\tilde w_{\mathrm{sig}}^{(t)}\|,
\quad
\frac{1}{|\mathcal N_1^+|}\!\sum_{(w,a)\in\mathcal N_1^+}\!
|a^{(t)}-\tilde a^{(t)}|
\right\}.
\]
From the updates
\[
w^{(t+1)}=w^{(t)}-\eta\nabla_w\hat L_\rho,\qquad
\tilde w^{(t+1)}=\tilde w^{(t)}-\eta\nabla_w L_0,
\qquad
a^{(t+1)}=a^{(t)}-\eta\nabla_a\hat L_\rho,\quad
\tilde a^{(t+1)}=\tilde a^{(t)}-\eta\nabla_a L_0,
\]
projecting to the signal coordinate and using the triangle inequality we obtain
\begin{equation}\label{eq:eps-rec-B}
\varepsilon_{t+1}\;\le\;\varepsilon_t
\;+\; \eta\,\Delta_t^{\mathrm{conc}}
\;+\; \eta\,\Delta_t^{\mathrm{pop}},
\end{equation}
where
\[
\Delta_t^{\mathrm{conc}}
:=\frac{1}{|\mathcal N_1^+|}\!\sum_{(w,a)\in\mathcal N_1^+}
\Big(\big\|\nabla_{w_{\mathrm{sig}}}\hat L_\rho-\nabla_{w_{\mathrm{sig}}}L_\rho\big\|
+\big|\nabla_a\hat L_\rho-\nabla_a L_\rho\big|\Big),
\]
\[
\Delta_t^{\mathrm{pop}}
:=\frac{1}{|\mathcal N_1^+|}\!\sum_{(w,a)\in\mathcal N_1^+}
\Big(\big\|\nabla_{w_{\mathrm{sig}}}L_\rho(w,a)-\nabla_{w_{\mathrm{sig}}}L_0(\tilde w,\tilde a)\big\|
+\big|\nabla_a L_\rho(w,a)-\nabla_a L_0(\tilde w,\tilde a)\big|\Big).
\]

\emph{Sampling (concentration) term.}
By Lemma~\ref{lem:conc_grad} and Lemma~\ref{lemma:gradapprox1} (items \ref{Ga}, \ref{Gw}),
using $|a|\vee\|w\|\le B_t$ for (weakly) controlled neurons,
\begin{equation}\label{eq:conc-block}
\Delta_t^{\mathrm{conc}}
\;\lesssim\;
\begin{cases}
\theta\,N^{(t)} & (t\le T_A),\\[2pt]
\zeta'\,N^{(t)} & (T_A\le t\le T_B),
\end{cases}
\qquad \zeta'=\log^{-c}d,
\end{equation}
where $N^{(t)}:=\frac{1}{m}\sum\|w_{\mathrm{sig}}^{(t)}\||a^{(t)}|$ is the global Phase~I
mass proxy.

\emph{Population term.}
Write $\rho=\|\wsig\|/\|w\|$, $\tilde\rho=\|\tilde\wsig\|/\|\tilde w\|$, and
\[
\nabla_{w_{\mathrm{sig}}} L_0(w,a)
= a\,\mathbb{E}_{(z,\xi)}\!\big[\,y\,z_1\,\sigma'(w^\top \xi+\wsig^\top z)\,\big],
\qquad z_1=\mu^\top z.
\]
Set $F_\rho(x):=\mathbb E_\xi[\sigma'(\rho x+\sqrt{1-\rho^2}\,\xi)]$ so that the integrand is
$y\,z_1\,F_\rho(z_1)$. For ReLU,
\[
\frac{\mathrm d}{\mathrm d\rho}F_\rho(x)
= \frac{1}{\sqrt{2\pi}}\exp\!\Big(\!-\frac{\rho^2 x^2}{2(1-\rho^2)}\Big)\,
\frac{x}{(1-\rho^2)^{3/2}},
\]
hence
\[
\left|\frac{\mathrm d}{\mathrm d\rho}\,
\mathbb E_z\big[y\,z_1\,F_\rho(z_1)\big]\right|
\;\le\; \frac{1}{\sqrt{2\pi}}\,
\mathbb E\!\left[\frac{z_1^2}{(1-\rho^2)^{3/2}}\,
\exp\!\Big(\!-\frac{\rho^2 z_1^2}{2(1-\rho^2)}\Big)\right]
= 1,
\]
uniformly in $\rho\in[0,1]$ (Gaussian integral with $u=z_1/\sqrt{1-\rho^2}$). Therefore,
\begin{equation}\label{eq:pop-Lip-core}
\big\|\nabla_{w_{\mathrm{sig}}}L_0(w,a)-\nabla_{w_{\mathrm{sig}}}L_0(\tilde w,\tilde a)\big\|
\;\lesssim\; |a-\tilde a| + |a|\cdot|\rho-\tilde\rho|.
\end{equation}
Again for ReLU,
\begin{equation}\label{eq:pop-Lip-a}
\big|\nabla_a L_0(w,a)-\nabla_a L_0(\tilde w,\tilde a)\big|
\;\lesssim\; \|w-\tilde w\| + |a-\tilde a|.
\end{equation}
To control $|\rho-\tilde\rho|$, write
\[
|\rho-\tilde\rho|
\;\le\; \frac{\|\wsig-\tilde\wsig\|}{\|w\|}
+\frac{\|\tilde\wsig\|}{\|w\|\,\|\tilde w\|}\,\big|\|w\|-\|\tilde w\|\big|.
\]
By item~\ref{Wa} of weak control,
\[
\big||a|^2-\|w\|^2\big|\;\le\; \theta^2\Gamma_t,
\qquad
\Gamma_t:=\zeta^{1/2}+C\eta^2(t-T_A)\zeta'^2,
\]
hence $|\,\|w\|-|a|\,|\le \theta\sqrt{\Gamma_t}$ and
$\big|\|w\|-\|\tilde w\|\big|\le |a-\tilde a|+2\theta\sqrt{\Gamma_t}$.
Using $|a|\le\|w\|$ then gives
\begin{equation}\label{eq:rho-diff-final-block}
|a|\,|\rho-\tilde\rho|
\;\le\; \|w_{\mathrm{sig}}-\tilde w_{\mathrm{sig}}\| + |a-\tilde a|
+ 2\theta\sqrt{\Gamma_t}.
\end{equation}
Combining \eqref{eq:pop-Lip-core}, \eqref{eq:pop-Lip-a}, \eqref{eq:rho-diff-final-block} and
averaging over the block,
\begin{equation}\label{eq:pop-Lip-avg}
\frac{1}{|\mathcal N_1^+|}\!\sum_{(w,a)}
\Big(\big\|\nabla_{w_{\mathrm{sig}}}L_0(w,a)-\nabla_{w_{\mathrm{sig}}}L_0(\tilde w,\tilde a)\big\|
+\big|\nabla_a L_0(w,a)-\nabla_a L_0(\tilde w,\tilde a)\big|\Big)
\;\lesssim\; \varepsilon_t + \theta\sqrt{\Gamma_t}.
\end{equation}

Finally, the \emph{approximation} gap $L_\rho-L_0$ appears both in the sampling term
(through $\hat L_\rho-L_\rho$) and in the population term. The global bounds of
Lemma~\ref{lemma:gradapprox1} and the coordinate bound of Lemma~\ref{lemma:approxsig}
imply
\[
\frac{1}{|\mathcal N_1^+|}\!\sum_{(w,a)}
\Big(\big\|\nabla_{w_{\mathrm{sig}}}L_\rho-\nabla_{w_{\mathrm{sig}}}L_0\big\|
+\big|\nabla_a L_\rho-\nabla_a L_0\big|\Big)
\;\lesssim\;
\begin{cases}
\theta\,N^{(t)} & (t\le T_A),\\
\zeta'\,N^{(t)} & (T_A\le t\le T_B),
\end{cases}
\]
using $|a|\vee\|w\|\le B_t$ and the batch/scale choices. Together with
\eqref{eq:pop-Lip-avg}, we conclude
\begin{equation}\label{eq:pop-drift-clean-block}
\Delta_t^{\mathrm{pop}}
\;\lesssim\; \varepsilon_t + \theta\sqrt{\Gamma_t}
\quad\text{and}\quad
\Delta_t^{\mathrm{conc}} \text{ as in \eqref{eq:conc-block}}.
\end{equation}

\paragraph{Phase Ia.}
In Phase~Ia, for controlled neurons Lemma~\ref{lemma:pop1} gives
\[
-\,\wsig^\top \nabla_w L_0
= \tau\,\frac{|a|}{\|\wpe\|}\,\|\wsig\|^2 \;+\; o\!\left(\frac{|a|}{\|\wpe\|}\,\|\wsig\|^2\right),
\qquad \frac{|a|}{\|\wpe\|}=1+o(1),
\]
hence the per–step growth matches
\[
\|\wsig^{(t+1)}\| \;=\; \|\wsig^{(t)}\|\big(1+\tau\eta+o(\eta)\big)\qquad\text{uniformly over controlled neurons.}
\]
The negligible set with too small initial correlation contributes $o(1)$ at block
scale and is absorbed by \eqref{eq:ind-conc-block}. Therefore the \emph{same} multiplicative
factor $(1+\tau\eta+o(\eta))$ governs both $N^{(t)}$ and the linearization in \eqref{eq:eps-rec-B}, i.e.
\begin{equation}\label{eq:B34-Ia}
\varepsilon_{t+1}
\;\le\; \big(1+(\tau+o(1))\eta\big)\varepsilon_t
\;+\; \eta\Big(\theta\,N^{(t)} + \theta\sqrt{\Gamma_t}\Big),
\qquad t\le T_A.
\end{equation}
Let $\delta_t:=\varepsilon_t/N^{(t)}$. Since $N^{(t+1)}=(1+\tau\eta+o(\eta))N^{(t)}$,
\eqref{eq:B34-Ia} gives
\[
\delta_{t+1}
\;\le\; \delta_t \;+\; \theta \;+\; \frac{\theta\sqrt{\Gamma_t}}{N^{(t)}}.
\]
Here $\sqrt{\Gamma_t}=O(1)$ (Phase~Ia) and $N^{(t)}$ is increasing, so from $\delta_0=0$,
\[
\delta_{T_A}\;\lesssim\;\theta
\quad\Longrightarrow\quad
 \varepsilon_{T_A}\;\le\; \theta'\,N^{(T_A)}\,,\ \ \theta'=\theta^{0.9}. 
\]
This explicit Phase~Ia handoff uses the \emph{exact} growth rate $\tau$ on both sides, which is
necessary since $T_A=\Theta(\eta^{-1}\log d)$.

\paragraph{Phase Ib .}
On $[T_A,T_B]$, combine \eqref{eq:eps-rec-B}, \eqref{eq:conc-block}, \eqref{eq:pop-drift-clean-block}:
\begin{equation}\label{eq:B34-Ib}
\varepsilon_{t+1}
\;\le\; (1+O(\eta))\,\varepsilon_t \;+\; \eta\,\zeta' N^{(t)} \;+\; \eta\,\theta\sqrt{\Gamma_t}.
\end{equation}
Using $\sqrt{\Gamma_t}\le \zeta^{1/4}+O(\eta(t-T_A)\zeta')$ and the strong–mass lower bound
$N^{(t)}\gtrsim \theta^2\zeta(\log d)^{-1/2}$ on $[T_A,T_B]$, we absorb the last term into
$\zeta' N^{(t)}$ (choose $c$ large in $\zeta'=\log^{-c}d$, recall $\zeta=o(\log^{-2}d)$). Thus,
\[
\varepsilon_{t+1} \;\le\; (1+O(\eta))\,\varepsilon_t \;+\; \eta\,\zeta' N^{(t)}.
\]
Unroll for $T_B-T_A=O(\eta^{-1}\log\log d)$ steps and use $\varepsilon_{T_A}\le \theta' N^{(T_A)}$:
\[
\varepsilon_t
\;\le\;
(1+O(\eta))^{t-T_A}\,\varepsilon_{T_A}
\;+\; \zeta'\,\eta\sum_{s=T_A}^{t-1} (1+O(\eta))^{t-1-s}\,N^{(s)}.
\]
Since $N^{(s)}$ is increasing, the sum is $\lesssim \zeta'\,\log\log d\cdot N^{(t)}$. Furthermore,
$(1+O(\eta))^{T_B-T_A}=\mathrm{polylog}(d)$; taking $\theta'=\theta^{0.9}$ small and $c$ large in $\zeta'$,
\[
\mathrm{polylog}(d)\cdot \theta' \ \ll\ \log^{-c_U}d,
\qquad
\zeta'\log\log d \ \ll\ \log^{-c_U}d.
\]
Therefore, for all $T_A\le t\le T_B$,
\begin{equation}\label{eq:Ib-final}
\ \varepsilon_t \;\le\; \zeta'\,N^{(t)} \;\le\; \log^{-c_U}d\cdot N^{(t)}.
\end{equation}

\paragraph{Conclusion (pre-balanced blocks).}
Combining the Phase~Ia handoff and Phase~Ib accumulation, for all $t\le T_B$,
\[
\varepsilon_t \;\le\;
\begin{cases}
\theta'\,N^{(t)}, & t\le T_A,\\[2pt]
\zeta'\,N^{(t)}, & T_A\le t\le T_B,
\end{cases}
\quad\Longrightarrow\quad
|N_i^{\pm,(t)}-\tilde N_i^{\pm,(t)}|
\;\le\; 2B_t\,\varepsilon_t
\;=\; o(N^{(t)}),
\]
uniformly over blocks $i,\pm$. Together with \eqref{eq:ind-conc-block} this yields, w.h.p., for all
$t\le T_1(\le T_B)$,
\[
U^{(t)} \;=\; \max_{i,\pm}\frac{|N_i^{\pm,(t)}-\tilde N_i^{\pm,(t)}|}{N^{(t)}}
\;+\; \max_{i,\pm}\frac{|\tilde N_i^{\pm,(t)}-\mathbb E[\|\tilde w_{\mathrm{sig}}^{(t)}\||\tilde a^{(t)}|]|}{N^{(t)}}
\;\le\; \log^{-c_U} d,
\]
which is Lemma~\ref{lem:ineq_U}. The approximation bounds (Lemmas~\ref{lemma:gradapprox1}–\ref{lemma:approxsig})
are used both in \eqref{eq:conc-block} and in \eqref{eq:pop-drift-clean-block}, and the only place where
the precise per–step constant matters is Phase~Ia, where it is $\tau+o(1)$ as required.

\subsection{Conclusion: Proof of Lemma \ref{lem:phase1}}\label{sec:app:dynamicIc}We have shown by recursion that after $T_1$ gradient steps, all neurons such that $\norm{\wsig^{(0)}}\geq \theta/\sqrt{d}$ (strong neurons) are weakly controlled and the others are either controlled or weakly controlled. It implies that  \[ \expec_{\rho^{(T_1)}}\norm{\wpe+\wopp}^2\leq 4\theta^2\] and $\theta \zeta^{-1} \leq \norm{\wsig} \leq 1$ for strong neurons.

\section{Analysis of Phase II}\label{sec:app:phase2}

\subsection{Analysis on Clean Gradient with Oracle} \label{sec:clean_gradient_oracle}

The following lemma can control the distance with the clean gradients (it corresponds to Lemma E.5 in \cite{glasgow2023sgd}, we adapted and simplified the argument for the Gaussian case).
\begin{lemma}\label{claim:approxerror}
For any neuron $(a,w) \in \mathcal{N}$, we have the followings:
\begin{enumerate}
  \setlength{\parskip}{0cm} %
  \setlength{\itemsep}{0cm} %
    \item $\|\ncl_w L_{\rho} - \nabla_w L_{\rho}\| \leq |a|\expec_\rho \norm{a\wpe}$.
    \item $\|\ncl_{a} L_{\rho} - \nabla_{a} L_{\rho}\| \leq \|w\|  \expec_\rho \norm{a\wpe}$.
    \item $\|\ncl_w L_{\rho} - \ncl_w L_{\rho, id}\| \leq 4|a|N^{(t)}U^{(t)}$.
    \item $\|\ncl_{a} L_{\rho} - \ncl_{a} L_{\rho, id}\| \leq 4\|w\|N^{(t)} U^{(t)}$.  
\end{enumerate}
\end{lemma}
\begin{proof}
We study the first statement.
    Let us define $\Delta_x := (\ell'_{\rho}(x) - \ell'_{\rho}(z))\sigma'(w^\top x)$. We evaluate the difference of gradients as
\begin{align*}
    \|\ncl_w L_{\rho} - \nabla_w L_{\rho}\| & = |a |\|\mathbb{E}_x \Delta_x x\| \\
    &=|a |\sup_{v : \|v\| = 1} \mathbb{E}_x \Delta_x \langle{v, x\rangle} \\
    &\leq |a |\sup_{v : \|v\| = 1} \sqrt{\mathbb{E}_x \Delta_x^2} \sqrt{\mathbb{E}_x\langle{v, x\rangle}^2}\\
    &= |a |\sqrt{\mathbb{E}_x \Delta_x^2}.
\end{align*}
Now, we obtain
\begin{align*}
    \mathbb{E}_x[\Delta_x^2] &\leq \mathbb{E}_x[(\ell_{\rho}'(x) - \ell'_{\rho}(z))^2]\\
    &\leq \expec_x(f_\rho(x)-f_\rho(z))^2\tag{$\ell_{\rho}$ is $2$-Lipschitz } \\
    &\leq \expec_x(\expec_\rho |a||\xi^\top w|)^2\\
    & \leq \left(\expec_\rho (\expec_\xi |a|^2|\xi^\top w|^2)^{1/2}\right)^2 \tag{by Minkowski's integral inequality}\\
    &\leq \left( \expec_\rho \norm{a\wpe}\right)^2.
\end{align*}

For the other statements, the same argument can be used to prove the other inequalities combined with the fact that \[ \abs{f_{\rho^{(t)}, id}(z)-f_{\rho^{(t)}} (z)}\leq 4 \norm{z}N^{(t)}U^{(t)},\] 
hence we omit the proof.
\end{proof}
\begin{remark}
    In Phase II, we will show that $\expec_\rho \norm{a\wpe}\leq \zeta \expec_\rho \norm{aw}$. 
\end{remark}

\subsection{Population gradients evaluation.}\label{app:phase2:pop_grad}

In addition to the average margin $g_\mu$, we also define 
\begin{align}
    g_\rho = \expec_z|\ell'_{\rho, id}(z)|.
\end{align}

\begin{lemma} \label{lem:pop2_grad}For any neuron $(w,a)\in \calS$ such that $\wsig^\top \mu >0$ holds for some $\mu \in \lbrace \pm \mu_1, \pm \mu_2 \rbrace$, we have \[ \mu^\top \ncl_w L_{\rho, id} = -|a|g_{\mu}(1 \pm \zeta^{0.25}), ~~\mbox{and}~~ -y\ncl_a L_{\rho, id} = (1 \pm \zeta^{0.25})\norm{\wsig}g_\mu.\]
\end{lemma}

\begin{lemma}\label{lem:laplace} There exist constants $0<c_1< c_2$ such that for all $\rho$ large enough, we have \[ \frac{c_1 }{\rho^3}\leq g_\mu \leq \frac{c_2 }{\rho^3}, ~~ \mbox{and}~~ \frac{c_1 }{\rho}\leq g_\rho \leq \frac{c_2 \log \rho }{\rho}.\]   
\end{lemma}
\begin{proof}[Proof of Lemma \ref{lem:laplace}]
    To control $g_\rho$, it is sufficient to control \[ \expec_z e^{-yN^{(t)}(\sigma(\mu_1^\top z )+\sigma(-\mu_1^\top z )-\sigma(\mu_2^\top z )-\sigma(-\mu_2^\top z ))}.\] It is easy to check that by definition \[ y(\sigma(\mu_1^\top z )+\sigma(-\mu_1^\top z )-\sigma(\mu_2^\top z )-\sigma(-\mu_2^\top z )) = \abs{|X_1|-|X_2|}\] where $X_1=\mu_1^\top z$ and $X_2=\mu_2^\top z$. Since $X_1$ and $X_2$ are independent standard Gaussian r.v., we obtain \begin{align*}
         \expec_z e^{-yN^{(t)}(\sigma(\mu_1^\top z )+\sigma(-\mu_1^\top z )-\sigma(\mu_2^\top z )-\sigma(-\mu_2^\top z ))} &= \expec_{X_1, X_2}e^{-N^{(t)}\abs{|X_1|-|X_2|}}\\
         &\geq \expec_{X_1, X_2}e^{-N^{(t)}\abs{X_1-X_2}}\\
         &\geq \frac{c_1}{N^{(t)}},
    \end{align*} for some constant $c_1\in (0,1)$ by Lemma \ref{lem:folded_gauss}.

    For the upper-bound, notice that \[ \prob(\abs{|X_1|-|X_2|}\leq \frac{\log N^{(t)}}{N^{(t)}})=O(\frac{\log N^{(t)}}{N^{(t)}}),\] and \[ \expec_{X_1, X_2}e^{-N^{(t)}\frac{\log N^{(t)}}{N^{(t)}}}=\frac{1}{N^{(t)}}.\]
     Recall that by construction, $\ell'_{\rho, id}(z)$ is invariant by rotation of angle $\pi/4$. Using the fact, we have \begin{align*}
    g_\mu&= \expec_{z|y=1}|\ell'_{\rho, id}(z)|\sigma'(\mu^\top z)\mu^\top z-\expec_{z|y=-1}|\ell'_{\rho, id}(z)|\sigma'(\mu^\top z)\mu^\top z\\
    &=2\expec_r\expec_{\theta\in (-\pi/2,0)}r|\ell'_{\rho, id}(r,\theta)|(\cos \theta - \sin \theta) -  \expec_r\expec_{\theta\in  (0,\pi/4)}r|\ell'_{\rho, id}(r,\theta)|(\cos \theta - \sin \theta)\\
    &= 2\expec_r\expec_{\theta\in (0,\pi/2)}r|\ell'_{\rho, id}(r,\theta)|( \cos \theta +\sin \theta) - \expec_r\expec_{\theta\in (0,\pi/2)}r|\ell'_{\rho, id}(r,\theta)||\cos \theta - \sin \theta|\\
    &= 2\expec_r\expec_{\theta\in (0,\pi/4)}r|\ell'_{\rho, id}(r,\theta)|2\sin \theta +\expec_r\expec_{\theta\in (\pi/4,\pi/2)}r|\ell'_{\rho, id}(r,\theta)|2\cos\theta.
\end{align*}
Since we have \[ \expec_{\theta\in (0,\pi/4)}|\ell'_{\rho, id}(r,\theta)|2\sin \theta \geq 0.6 \expec_{\theta\in (\pi/8,\pi/4)}|\ell'_{\rho, id}(r,\theta)|,\] it is easy to check that $\expec_{\theta\in (\pi/8,\pi/4)}|\ell'_{\rho, id}(r,\theta)| \geq 0.5 e^{-0.3N^{(t)}r}$ holds. It remains to control integrals of the form $\expec_r re^{-cr}$ where $r$ has density $f(x)=x^2e^{-x^2/2}\indic_{x\geq 0}$. By using the change of variable $u=cx$ and using dominated convergence, one can easily check that \[\frac{1}{c^3} \lesssim \expec_r re^{-cr} \lesssim \frac{1}{c^3}.\]
\end{proof}

\begin{proof}[Proof of Lemma \ref{lem:pop2_grad}]
    W.l.o.g. we can assume that $\mu=\mu_1$. First, we are going to show that $\expec_x\ell'_{\rho, id}(z)\sigma'(w^\top z)\mu^\top z \approx \expec_x \ell'_{\rho, id}(z)\sigma'(\wsig^\top z)\mu^\top z$.

    We have \begin{align*}
        \abs{\expec_z\ell'_{\rho, id}(z)\sigma'(w^\top z)\mu^\top z - \expec_z \ell'_{\rho, id}(z)\sigma'(\wsig^\top z)\mu^\top z} &\leq  \expec_z\ell'_{\rho, id}(z)\indic_{|\wopp ^\top z|\geq |\wsig^\top z|}|\mu^\top z|\\
        &\leq \sqrt{ \expec_z (\ell'_{\rho, id}(z))^2}\sqrt{\expec_z  \indic_{\lbrace |\wopp^\top z| \geq |\wsig^\top z|\rbrace}|\mu^\top z|} \tag{by Cauchy-Schwartz}.
    \end{align*}
    Notice that $\wopp^\top z$ and $\wsig ^\top z$ are orthogonal Gaussian r.v. so we obtain\begin{align*}
        \expec_z  \indic_{\lbrace |\wopp^\top z| \geq |\wsig^\top z|\rbrace}|\mu^\top z| &= \expec_{G\sim \calN(0,1)}|G|\prob_{G'\perp G }(|G'|\geq \frac{\norm{\wsig}}{\norm{\wopp}}|G|)\\
        &\leq \sqrt{\frac{2}{\pi}}\frac{\norm{\wopp}}{\norm{\wsig}}\expec_G e^{-\frac{\norm{\wsig}^2G^2}{2\norm{\wopp}^2}} \tag{since $\prob(|G|\geq t)\leq \frac{1\sqrt{2}}{t\sqrt{\pi}}e^{-t^2/2}$}\\
        &\leq\sqrt{\frac{2}{\pi}}\frac{\norm{\wopp}}{\norm{\wsig}} \frac{1}{\sqrt{1+\frac{\norm{\wsig}^2}{\norm{\wopp}^2}}}\tag{because $\expec e^{-tG^2}=\frac{1}{\sqrt{1+2t}}$}\\
        &\leq \sqrt{\frac{2}{\pi}}\left(\frac{\norm{\wopp}}{\norm{\wsig}}\right)^{1.5}\\
        &\leq \sqrt{\frac{2}{\pi}}\zeta^{1.5} \tag{$(w,a)\in \calS$ }.
    \end{align*}
    Also notice that \[  \expec_z (\ell'_{\rho, id}(z))^2 \leq \expec_z e^{-y f_{2\rho, id}}.\] So, we have \[  \abs{\expec_z\ell'_{\rho, id}(z)\sigma'(w^\top z)\mu^\top z - \expec_z \ell'_{\rho, id}(z)\sigma'(\wsig^\top z)\mu^\top z}\leq \left( \frac{2}{\pi}\right)^{0.25}\zeta^{0.75}g_{2\rho}^{0.5}.\]

By Lemma \ref{lem:laplace} and the choice of $\zeta$ we have $\zeta^{0.75}g_{2\rho}^{0.5}\leq \zeta^{0.25}g_\mu$ for $\rho \lesssim \log \log d$.
In conclusion, we have shown that \[   -|a|g_{\mu}(1+\zeta^{0.25})\geq \mu^\top \ncl_w L_{\rho, id} \geq  -|a|g_{\mu}(1-\zeta^{0.25}).\]
The same calculation leads to the second result of the lemma.
\end{proof}

\subsection{Proof of Lemma \ref{lem:inductive2} and Lemma \ref{lem:inductive3}}\label{app:proof_inductive_lem}
We will first focus on the dynamic of non-heavy neurons and show that they won't grow faster than heavy neurons. Then, we will study the dynamic of heavy neurons in more detail.

\subsubsection{Non-Heavy Neurons Dynamics}
\begin{lemma}\label{lem:non_heavy}For any neuron $(w,a)\not \in \calS$ we have \[ \abs{\ncl_a L}= \abs{w^\top \ncl_w L}\leq \norm{w}g_\mu.\]  
\end{lemma}

\begin{proof}
 W.l.og. we can assume that $\wsig$ is aligned with $\mu_1$. Recall that the distribution of $z|y=1$ is the same as $R(z)|y=-1$ where $R$ is a rotation of angle $-\pi/2$. Hence, we can write \[ g_{\mu_1} =\frac{1}{2} \expec_{z|y=1} \ell'_{\rho, id} (z)(\sigma(\mu_1^\top z)-\sigma(\mu_2^\top z))  \] since $\ell'_{\rho, id} (z)$ is invariant by rotation of angle $-\pi/2$. In particular, when conditioned on $y=1$, $\ell'_{\rho, id} (z)>0$. By using the same decomposition, we can write \begin{align*}
    \ncl_a L&= \expec_z \expec_\xi \ell'_{\rho, id} (z)\sigma(\norm{\wsig}\mu_1^\top z + \norm{\wopp}\mu_2^\top z + \wpe^\top \xi )\\
    &= \frac{1}{2}\expec_{z|y=1}  \ell'_{\rho, id} (z)\expec_\xi\left(\underbrace{\sigma(\norm{\wsig}\mu_1^\top z + \norm{\wopp}\mu_2^\top z + \wpe^\top \xi )-\sigma(\norm{\wsig}\mu_2^\top z - \norm{\wopp}\mu_1^\top z + \wpe^\top \xi )}_{I}\right)
\end{align*}
To simplify the notation, let us denote $a= \norm{\wsig}\mu_1^\top z + \norm{\wopp}\mu_2^\top z $, $b=\norm{\wsig}\mu_2^\top z - \norm{\wopp}\mu_1^\top z $ and $c= \wpe^\top \xi$.

\paragraph{Case $a\geq b$.} It is easy to check that when $c<-a$, $I=0$, when $c\geq -b $, $I= a-b$ and when $-a\leq c \leq -b $ we have $I= a+ c$. By integrating over $\xi$, we obtain \begin{align*}
    \expec_\xi I &= (a-b)\prob(\wpe^\top \xi \geq -b)+\expec_\xi (\wpe^\top \xi+a)\indic_{-a\leq \wpe^\top \xi \leq -b }\\
    &\leq (a-b)\prob(\wpe^\top \xi \geq -b)+(a-b) \prob(-a\leq \wpe^\top \xi \leq -b )
    &\leq (a-b)\prob(-a\leq \wpe^\top \xi)\leq a-b.
\end{align*} 

\paragraph{Case $a< b$.} When $c < -b$, $I=0$, when $c\geq -a$, $I=(a-b)$ and when $-b \leq c \leq -a$ we have $I=-b-c$. By integrating over $\xi$, we obtain \begin{align*}
    \expec_\xi I&= (a-b)\prob(\wpe^\top \xi \geq -a)-\expec_\xi (\wpe^\top \xi+b)\indic_{-b\leq \wpe^\top \xi \leq -a }\\
    &\leq (a-b)\prob(\wpe^\top \xi \geq -a).
\end{align*}
By consequence $\expec_\xi I \leq (a-b)$. But \[ a-b = \mu_1^\top z(\norm{\wsig}+\norm{\wopp})+\mu_2^\top z(\norm{\wopp}-\norm{\wsig}) \leq (\norm{\wsig}+\norm{\wopp})(\sigma(\mu_1^\top z)+\sigma(\mu_2^\top z)).\]
 Since $g_\mu\geq 0$ we have $\expec_{z|y=1}\ell'_{\rho, id} (z)\sigma(\mu_2^\top z) \leq \expec_{z|y=1}\ell'_{\rho, id} (z)\sigma(\mu_1^\top z)$. By consequence, we obtain \[ \ncl_a L\leq \frac{(\norm{\wsig}+\norm{\wopp})}{2}g_{\mu_1}\leq \norm{w}g_{\mu_1}. \] One can derive a similar lower bound with the same argument. 
\end{proof}
By using the same argument as in Lemma E.11 in \cite{glasgow2023sgd} we obtain the following corollary.
\begin{corollary}\label{cor:ng}
    For every neuron, w.h.p. \[ \norm{w^{(t+1)}}^2\leq \norm{w^{(t)}}^2(1+2\eta(1+2\zeta H) g_\mu).\]
\end{corollary}
\subsubsection{Heavy Neurons Dynamics}

We develop a proof of several lemmata for the inductive properties, such as Lemma \ref{lem:inductive2} and Lemma \ref{lem:inductive3}.
\paragraph{Initialization.}
Let us define \begin{align*}
    \tilde{\calN}_1^+ & = \lbrace (w,a)\in \calN: a>0, w^\top \mu_1>\theta/\sqrt{d} \rbrace\\
    \tilde{\calN}_1^- & = \lbrace (w,a)\in \calN: a>0, w^\top \mu_1<-\theta/\sqrt{d} \rbrace\\
    \tilde{\calN}_2^+ & = \lbrace (w,a)\in \calN: a<0, w^\top \mu_2>\theta/\sqrt{d} \rbrace\\
    \tilde{\calN}_2^- & = \lbrace (w,a)\in \calN: a<0, w^\top \mu_2<-\theta/\sqrt{d} \rbrace.
\end{align*}
The result of Phase I shows that for any block $B\in \lbrace \tilde{\calN}_i^+, \tilde{\calN}_i^-\rbrace_{i=1}^2$ \[ \expec_\rho \indic_{(w,a)\in B}\norm{aw} \geq \zeta^{-2}\theta ^2\]
and all the properties of a signal-heavy network are satisfied.

\paragraph{Inductive step.}
Recall that $N^{(t)}$ corresponds to the mass of each block of the ideal oracle network. In order to establish Lemma \ref{lem:inductive2} and Lemma \ref{lem:inductive3} we will first prove the following lemma.
\begin{lemma}\label{lem:phase2_heavy}
    Assume that at time $t$, the network is $(\zeta, H)$ signal-heavy and $\eta \leq \zeta^3$. Then with probability $1-d^{-\Omega(1)}$ we have \begin{enumerate}
        \item $ N^{(t+1)} = (1+2\eta g_\mu^{(t)}(1\pm o(1)))N^{(t)}$,
        \item $\norm{\wpe^{(t+1)}}+\norm{\wopp^{(t+1)}}\leq \zeta(1+O(\eta \zeta^{0.25})) \norm{\wsig^{(t+1)}}$,
        \item $U^{(t+1)}\leq  (1+\eta g_\mu)U^{(t)}+5\eta \zeta ^{0.25}g_\mu$ and $U^{(T_1+t)}\leq 2U^{(T_1)}+t\eta \zeta^{0.25}$.
    \end{enumerate}
\end{lemma}

\begin{proof}
    W.l.o.g., one can assume that $\wsig$ is aligned with $\mu_1$.  
    \paragraph{Evolution of $N^{(t)}$.}
    For each heavy neuron, we have \begin{align*}
        a^{(t+1)}\norm{\wsig^{(t+1)}} &= (a^{(t)}-\eta\nabla_a \hat{L}_\rho)(\norm{\wsig^{(t)}}-\eta\mu_1^\top \nabla_w \hat{L}_\rho)\\
        &= (a^{(t)}-\eta\ncl_a L_{\rho,id})(\norm{\wsig^{(t)}}-\eta\mu_1^\top \nabla_w L_{\rho,id}) \\
        & \qquad +\underbrace{\eta (|a| \norm{ \nabla_w \hat{L}_\rho- \ncl_w L_{\rho,id}}+\norm{\wsig} \norm{ \nabla_a \hat{L}_\rho- \ncl_a L_{\rho,id}})}_{E_1}\\
        & \qquad +\underbrace{\eta^2 \norm{ \nabla_a \hat{L}_\rho- \ncl_a L_{\rho,id}} \norm{ \nabla_w \hat{L}_\rho- \ncl_w L_{\rho,id}}}_{E_2}.
    \end{align*} 
    By Lemma \ref{lem:pop2_grad} we have \[ \left(a^{(t)}-\eta\ncl_a L_{\rho,id}\right)\left(\norm{\wsig^{(t)}}-\eta\mu_1^\top \ncl_w L_{\rho,id}\right) = a^{(t)}\norm{\wsig^{(t)}}\left(1+2(1+o(1))\eta g_\mu\right)+O\left(\eta^2 a^2 \norm{\wsig^{(t)}}^2\right).\] 
    Lemma \ref{lem:conc_grad} and Lemma \ref{claim:approxerror} show that \[ E_1\leq \eta a^2\left((\log d^{-c})+\expec_\rho \norm{a\wpe}+4N^{(t)}U^{(t)}\right)+\norm{w}^2\left(\log d^{-c}+\expec_\rho \norm{a\wpe}+4N^{(t)}U^{(t)}\right)\] 
    Beside, the definition of heavy network implies that \[ \expec_\rho \norm{a\wpe}\leq \zeta N^{(t)}= o(g_\mu) \text{ and } N^{(t)}U^{(t)}=o(g_\mu).\]
    $E_2$ can be controlled in a similar way: Lemma \ref{lem:conc_grad} and Lemma \ref{claim:approxerror} gives \[ E_2\leq \eta^2 |a|\norm{w^{(t)}}\left(\log^{-c}d+\expec_\rho \norm{a\wpe}+4N^{(t)}U^{(t)} \right)^2.\]
 For non-heavy neurons, Corollary \ref{cor:ng} implies that \begin{align*}
      m^{-1}\sum_{(a,w)\notin \calS}\norm{a^{(t+1)}w^{(t+1)}}&\leq m^{-1}\sum_{(a,w)\notin \calS}\norm{w^{(t+1)}}^2\\
      &\leq  m^{-1}\sum_{(a,w)\notin \calS}\norm{w^{(t)}}^2\left(1+2\eta(1+2\zeta H) g_\mu\right)\\
      &\leq \zeta \left(1+2\eta(1+2\zeta H) g_\mu\right) N^{(t)}. 
 \end{align*}
 For heavy neurons, we have \[ m^{-1}\sum_{(a,w)\in \calS} \norm{a^{(t+1)}w^{(t+1)}}=\left(1+2\eta (1+o(1))g_\mu\right)m^{-1}\sum_{(a,w)\in \calS} \norm{a^{(t)}w^{(t)}}. \]
The stated result follows.
\paragraph{Control of $U^{(t+1)}$.}  One can show similarly to Phase I that \[ U^{(t+1)} \leq (1+2\eta \zeta^{0.25}g_\mu^{(t)})U^{(t)}+5 \eta \zeta^{0.25}g_\mu^{(t)}.\] 

It is easy to show by recursion that if a sequence $(u_n)$ is such that $u_{n+1}\leq au_n+b$ then $u_n\leq a^n+b(\sum_{k=0}^{n-1}a^k)$.
By applying this result with $a=1+2\eta \zeta^{0.25}$ and $b=5\eta \zeta^{0.25}$ and noting that \[ \frac{1-a^{t}}{1-a}= \frac{(1+2\eta \zeta^{0.25})^t-1}{2\eta \zeta^{0.25}}\leq 2t \] since $2\eta \zeta^{0.25}=o(1)$ we obtain \[ U^{T_1+t}\leq (1+2\eta \zeta^{0.25})^tU^{(T_1)}+10t\eta \zeta^{0.25}.\] But by choice of $t\leq T_2=O(\log \log d/\eta)$ we have $(1+2\eta \zeta^{0.25})^t=1+o(1)$.

    \paragraph{Control of $\norm{\wpe^{(t+1)}}$ for heavy neurons.} %
     W.l.o.g., assume that $a>0$, i.e. $ \wsig$ is aligned with $\mu_1$.
     First, we control the approximation error \[ F_1 = \abs{\wopp^\top \nabla L_{\rho, id}-\expec \ell'_{\rho, id}(z)\sigma'(\wsig^\top z)w^\top \xi}.\]
     
     Let us denote $X_1=\mu_1^\top z$ and $X_2=\mu_2^\top z$. We have \begin{align*}
         F_1&\leq \expec \abs{\ell'_{\rho, id}(z)}\indic_{|X_1|\leq \frac{|\wpe^\top \xi|+\norm{\wopp}|X_2|}{\norm{\wsig}}}|\wpe^\top \xi|\\
         &\leq \sqrt{\expec \abs{\ell'_{\rho, id}(z)}^2}\sqrt{\expec \indic_{|X_1|\leq \frac{|\wpe^\top \xi|+\norm{\wopp}|X_2|}{\norm{\wsig}}}|\wpe^\top \xi|^2}\tag{by Cauchy-Schwartz}\\
         &\lesssim \sqrt{g_{2\rho}}\norm{\wpe}\sqrt{\frac{\norm{\wopp}+\norm{\wpe}}{\norm{\wsig}}}\\
         &\leq \sqrt{g_{2\rho}}\norm{\wpe} \zeta^{0.5} \ll g_\mu \norm{\wpe}.
     \end{align*}
     
   Since \[ \expec_{\xi}\ell'_{\rho, id}(z)\sigma'(\wsig^\top z)w^\top \xi=0,\] only the approximation term contributes to the gradient. 

  \paragraph{Control of $\norm{\wopp^{(t+1)}}$ for heavy neurons.}
     Similarly to Lemma \ref{lemma:pop1b}, let us denote $X_1=\mu_1^\top z$ and $X_2=\mu_2^\top z$. We have \begin{align*}
         \expec_x& \ell'_{\rho, id}(z)\sigma'(w^\top z)\mu_2^\top z\\
         &=\expec_{X_1, X_2}(\indic_{|X_1|\geq |X_2|}-\indic_{|X_2|\geq |X_1|})e^{-N^{(t)}\abs{|X_1|-|X_2|}}\sigma'(\norm{\wsig}X_1+\norm{\wopp} X_2)X_2\\
          &= \expec_{X_1, X_2}(\indic_{|X_1|\geq |X_2|}-\indic_{|X_2|\geq |X_1|})e^{-N^{(t)}\abs{|X_1|-|X_2|}}\sigma'(\norm{\wsig}X_1)X_2+E
     \end{align*} where \[ E=\expec_{X_1, X_2}(\indic_{|X_1|\geq |X_2|}-\indic_{|X_2|\geq |X_1|})e^{-N^{(t)}\abs{|X_1|-|X_2|}}(\sigma'(\norm{\wsig}X_1+\norm{\wopp} X_2)-\sigma'(\norm{\wsig}X_1))X_2. \]
     By using the same argument as in Lemma \ref{lem:pop2_grad}, one can show that \[ |E|\lesssim \zeta^{0.75}g_{2\rho}^{0.5}. \]
     Furthermore, by using the symmetry of $X_2$, one obtain that \[ \expec_{X_2}\indic_{|X_1|\geq |X_2|}e^{-N^{(t)}\abs{|X_1|-|X_2|}}\sigma'(\norm{\wsig}X_1)X_2=0.\]

    We can conclude that \begin{align*}
        \norm{\wopp^{(t+1)}}+\norm{\wpe^{(t+1)}}&\leq (\norm{\wopp^{(t)}}+\norm{\wpe^{(t)}})(1+\eta \sqrt{g_{2\rho}}\zeta^{0.5})\\
        &\leq \zeta (1+\eta \sqrt{g_{2\rho}}\zeta^{0.5})\norm{\wsig^{(t)}}\\
        &\leq  \zeta (1+\eta \sqrt{g_{2\rho}}\zeta^{0.5})\norm{\wsig^{(t+1)}}.
    \end{align*}
\end{proof}

To show that after one gradient iteration, the network is still heavy, it remains to show that the layer weights balance condition is satisfied. This can be done as in \cite{glasgow2023sgd} by using Lemma \ref{lemma:layer_balance}.

\subsection{Conclusion: Proof of Theorem \ref{thm:main}}

To complete the proof of Theorem \ref{thm:main}, we need to establish a connection between \( N^{(T)} \) and \( \mathbb{E}_x \ell_{\rho^{(T)}}(x) \). This proof proceeds in three main steps: 

(i) By contradiction, we show that there exists a stopping time $T$ such that \( N^{(T)} \) is of order \( \log \log d \).  
(ii) We upper-bound \( \mathbb{E}_z \ell_{\rho^{(T)}, \text{id}}(z) \) by separately analyzing the regions where \( z \) is close to and far from the decision boundary.  
(iii) We control the approximation \( \mathbb{E}_z \ell_{\rho^{(T)}, \text{id}}(z) \approx \mathbb{E}_x \ell_{\rho^{(T)}}(x) \).

\begin{proof}[Proof of Theorem \ref{thm:main}] 
    \textbf{Step (i):} The analysis of Phase I (Lemma \ref{lem:phase1}) shows that the resulting network is $(\zeta, H)$-signal-heavy with $\zeta=\log ^{-c} d$  for some constant $c>0$ large enough and  $N^{(T_1)}\gtrsim (\zeta^{-1}\theta)^2=\log^{-c'}d $ for some constant $c'>0$. Let us define the stopping time $T_2= c''\eta^{-1}(\log \log d)^4$ for some constant $c''>0$ that will be chosen later.

Assume that $N^{(t)}$ doesn't exceed $C_1\log \log d$ for $t\in [T_1, T_2]$ and some constant $C_1>0$ small enough. Then for $d$ large enough, we obtain a contradiction as follows \begin{align*}
    N^{(T_1+T_2)}&\geq N^{(T_1)}\left( 1+(2-o(1))\eta \min_{t\in [T_1,T_2]}g_\mu^{{(t)}} \right)^{T_2} \tag{by Lemma \ref{lem:inductive3}}\\
    &\gtrsim \frac{1}{\log^{-c'}d} \left( 1+(2-o(1))\eta c_1 \frac{1}{(\log \log d)^3} \right)^{T_2}\tag{by Lemma \ref{lem:pop2_grad_main}} \\
    &\gtrsim \frac{1}{\log^{C-c}d}e^{\Theta(\log\log d) }\\
    &\gg \log \log d. 
\end{align*}
We will denote by $T\leq T_2$ the first time $N^{(T)}$ exceeds $C_1\log \log d$.

\textbf{Step (ii):}
We overcome the absence of a strict margin between the clusters by splitting the input data into two classes: 1) the one for which the points are close to the boundary decision and hence the associated network evaluation in small, and the other 2) that are enough separated from the boundary decision and the sign of the network corresponds to the label. 

Recall the polar decomposition of $z=r(\cos \theta, \sin \theta)$ where $r^2$ follows a chi-square distribution with two degrees of freedom, and $\theta$ is uniform over $[0,2\pi)$. We have by definition \[ -yf_{\rho^{(T)}, id}(z) = \sqrt{2}N^{(T)}r( \abs{\sin \theta }\wedge \abs{\cos \theta} ).\]
Let us define the set
\[
\calC_\epsilon=\Big\{(r,\theta)\in\R^+\times[0,2\pi): r\big(|\sin\theta|\wedge|\cos\theta|\big)\le \epsilon\Big\}.
\]
When $X\sim\calN(0,I_2)$ and we write $X=r(\cos\theta,\sin\theta)$, we have
$r|\cos\theta|=|X_1|$ and $r|\sin\theta|=|X_2|$, hence
\[
\{X\in\calC_\epsilon\}=\{\min(|X_1|,|X_2|)\le \epsilon\}.
\]
Therefore, by a union bound,
\[
\prob(X\in\calC_\epsilon)\le \prob(|X_1|\le\epsilon)+\prob(|X_2|\le\epsilon)
=2\prob(|Z|\le\epsilon)=O(\epsilon),
\]
where $Z\sim\calN(0,1)$.
Choose $\epsilon=(\log\log d)^{-1/2}$. We have
\[
\expec_z(\ell_{\rho^{(T)},id}(z))
\le \expec_{z\notin\calC_\epsilon}(\ell_{\rho^{(T)},id}(z))
+\log(2)\prob(\calC_\epsilon)
\le e^{-\epsilon N^{(T)}}+O(\epsilon)=O(\epsilon),
\]
since $\epsilon N^{(T)}=\Theta(\sqrt{\log\log d})$.

\textbf{Step (iii): }
 We now bound the risk $\expec_x\ell_{\rho^{(T)}}(x)$. Recall that for all $z$ we have \[ |f_{\rho^{(T)}}(z)-f_{\rho^{(T)}, id}(z)|\leq N^{(T)}U^{(T)} \norm{z}. \]
 Since $\ell$ is 1-Lipschitz, we obtain that for some constant $c>0$ 
 \begin{align*}
     \expec_x\ell_{\rho^{(T)}}(x)&\leq \abs{\expec_x \ell_{\rho^{(T)}}(x)-\ell_{\rho^{(T)}}(z)}+ \abs{\expec_x \ell_{\rho^{(T)}}(z)-\ell_{\rho^{(T)}, id}(z)}+\expec_z(\ell_{\rho^{(T)}, id}(z))\\
     &\leq \expec_x\abs{f_{\rho^{(T)}}(x)-f_{\rho^{(T)}}(z) } + \expec_x\abs{f_{\rho^{(T)}}(z)-f_{\rho^{(T)},id}(z) }+ \expec_z(\ell_{\rho^{(T)}, id}(z))\\
     &\leq \expec_x\expec_{\rho^{(T)}} |a\xi^\top \wpe|+2 N^{(T)}U^{(T)} + \expec_z(\ell_{\rho^{(T)}, id}(z))\\
     &\leq \expec_{\rho^{(T)}} \norm{a\wpe}+O(\log^{-c}d) + \expec_z(\ell_{\rho^{(T)}, id}(z)) \tag{by the proof of Lemma \ref{lem:phase2_heavy}}\\
     &\leq \zeta^{(T)}N^{(T)}+O(\log^{-c}d)+\expec_z(\ell_{\rho^{(T)}, id}(z)) \tag{by the definition of signal-heavy network}\\
     &\lesssim \epsilon \tag{by Step (ii) and the fact $\zeta^{(T)}N^{(T)}\ll \epsilon $}
 \end{align*}     
\end{proof}

\subsection{Proof of Corollary~\ref{cor:classification}}

Fix $\epsilon \in (0,1)$ and let $x = z+\xi$ be as in the statement, with
$z\notin \mathcal C_\epsilon$ and $\xi\sim\mathcal N(0,I_{d-2})$
independent Gaussian noise in the orthogonal subspace. Let
$y=f^*(z)\in\{-1,1\}$ be the label of $x$ (which depends only on the
projection $z$).

\paragraph{Step 1: Margin at $z$ for the oracle and the trained network.}
By the explicit expression of the oracle network, we have for all
$z\notin\mathcal C_\epsilon$
\[
    y f_{\rho^{(T)},\mathrm{id}}(z)
    \;\ge\; c_\epsilon N^{(T)},
\]
for some constant $c_\epsilon>0$ depending only on~$\epsilon$. On the
other hand, the block-balance control used in the proof of
Theorem~\ref{thm:main} yields
\[
    \bigl|f_{\rho^{(T)}}(z) - f_{\rho^{(T)},\mathrm{id}}(z)\bigr|
    \;\le\; N^{(T)} U^{(T)} \|z\|,
\]
where $U^{(T)}\le \log^{-c_1}d$ for some $c_1>0$. Since
$N^{(T)} \asymp \log\log d$, we have
$N^{(T)} U^{(T)} \|z\| = o\bigl(N^{(T)}\bigr)$. Hence, for $d$ large
enough,
\[
    y f_{\rho^{(T)}}(z)
    = y f_{\rho^{(T)},\mathrm{id}}(z)
      + y\bigl(f_{\rho^{(T)}}(z) - f_{\rho^{(T)},\mathrm{id}}(z)\bigr)
    \;\ge\; \tfrac{c_\epsilon}{2} N^{(T)}.
\]
Set $c := c_\epsilon/2 > 0$ and denote
\[
    h(\xi) := y f_{\rho^{(T)}}(z+\xi).
\]
Then $h(0) = y f_{\rho^{(T)}}(z) \ge c N^{(T)}$.

\paragraph{Step 2: Lipschitz constant of $h$ in the orthogonal noise.}
We view $h$ as a function of $\xi$ in the orthogonal subspace. We have for any $\xi,\xi'$
\begin{align*}
    |h(\xi) - h(\xi')|
    &= \bigl|y f_{\rho^{(T)}}(z+\xi) - y f_{\rho^{(T)}}(z+\xi')\bigr| \\
    &\le \mathbb E_{\rho^{(T)}}\Bigl[
        |a|\,\bigl|\sigma(w^\top(z+\xi))-\sigma(w^\top(z+\xi'))\bigr|
    \Bigr].
\end{align*}
Since ReLU is $1$-Lipschitz and $\xi,\xi'$ lie in the orthogonal
subspace, only the orthogonal component $w_\perp$ contributes, and
\[
    \bigl|\sigma(w^\top(z+\xi))-\sigma(w^\top(z+\xi'))\bigr|
    \le |w_\perp^\top(\xi-\xi')|
    \le \|w_\perp\|\,\|\xi-\xi'\|.
\]
Therefore
\[
    |h(\xi) - h(\xi')|
    \;\le\;
    \mathbb E_{\rho^{(T)}}\bigl[|a|\,\|w_\perp\|\bigr]\;\|\xi-\xi'\|.
\]
Define
\[
    L := \mathbb E_{\rho^{(T)}}\|a w_\perp\|.
\]
By the signal-heavy property at time $T$, we have
$L \le \zeta^{(T)} N^{(T)}$ with $\zeta^{(T)} = \log^{-c_2}d$ for some
$c_2>0$. Thus $h$ is $L$-Lipschitz in~$\xi$.

\paragraph{Step 3: Mean perturbation in the orthogonal direction.}
We next control the mean shift $h(\xi)-h(0)$ when $\xi$ is Gaussian.
Using the one-dimensional Gaussian computation for each neuron and
averaging (see the proof of Theorem~\ref{thm:main}), one obtains
\[
    \bigl|\mathbb E_\xi\bigl(h(\xi) - h(0)\bigr)\bigr|
    \le \mathbb E_{\rho^{(T)}}\|a w_\perp\|
    = L
    \le \zeta^{(T)} N^{(T)}.
\]
Since $\zeta^{(T)} = \log^{-c_2}d$ and $N^{(T)}\asymp \log\log d$, we
have $\zeta^{(T)} N^{(T)} = o\bigl(N^{(T)}\bigr)$. In particular, for
$d$ sufficiently large,
\[
    \mathbb E_\xi h(\xi)
    = h(0) + \mathbb E_\xi\bigl(h(\xi)-h(0)\bigr)
    \ge c N^{(T)} - \zeta^{(T)} N^{(T)}
    \ge \tfrac{c}{2} N^{(T)}.
\]

\paragraph{Step 4: Gaussian concentration and misclassification probability.}
Since $h$ is $L$-Lipschitz and $\xi\sim\mathcal N(0,I_{d-2})$, the
Gaussian concentration inequality yields, for all $t>0$,
\[
    \mathbb P_\xi\bigl(|h(\xi) - \mathbb E_\xi h(\xi)| \ge t\bigr)
    \le 2\exp\!\bigl(-t^2/(2L^2)\bigr).
\]
Misclassification corresponds to $h(\xi)\le 0$. On the event
$\{h(\xi)\le 0\}$ we must have
\[
    |h(\xi) - \mathbb E_\xi h(\xi)|
    \;\ge\; \mathbb E_\xi h(\xi)
    \;\ge\; \tfrac{c}{2} N^{(T)},
\]
for $d$ large enough. Thus
\[
    \mathbb P_\xi\bigl(h(\xi)\le 0\bigr)
    \;\le\;
    \mathbb P_\xi\Bigl(
        |h(\xi) - \mathbb E_\xi h(\xi)|
        \ge \tfrac{c}{2} N^{(T)}
    \Bigr)
    \;\le\;
    2\exp\!\left(
        -\frac{(c N^{(T)}/2)^2}{2L^2}
    \right).
\]
Using $L\le \zeta^{(T)} N^{(T)}$ and $\zeta^{(T)} = \log^{-c_2}d$, we
get
\[
    \frac{(c N^{(T)}/2)^2}{2L^2}
    \;\ge\;
    \frac{c^2}{8\,(\zeta^{(T)})^2}
    \;=\; C (\log d)^{2c_2},
\]
for some constant $C>0$. Absorbing the leading factor $2$ into the
constant and renaming $2c_2$ as $c>0$, we obtain
\[
    \mathbb P_\xi\bigl(h(\xi)\le 0\bigr)
    \;\le\;
    \exp\!\bigl(-C(\log d)^c\bigr),
\]
as claimed. This shows that, for every $z\notin\mathcal C_\epsilon$, the
probability (over the orthogonal Gaussian noise~$\xi$) that the trained
network misclassifies $x=z+\xi$ is at most
$\exp(-C(\log d)^c)$. In particular, such points are correctly
classified with overwhelming probability, which proves the corollary.

\section{Additional Experiments}\label{app:xp}
We report here the full details and figures complementing Section~\ref{sec:xp}. 
Unless otherwise stated, we set $d=120$, $m=100$, $\eta=0.05$, $\theta=0.01$, $M=6400$ 
and train up to $10^6$ epochs. 

\subsection{Beyond Gaussian Inputs}
We consider two non-Gaussian input distributions: 
(a) a uniform distribution on $[-1,1]^d$, and 
(b) a Gaussian XOR distribution inspired by \cite{xu2023benign}.  
Both yield successful training as long as the distribution is symmetric 
with respect to $\pm \mu_1, \pm \mu_2$.  
Figures~\ref{fig:xp2loss}--\ref{fig:xp2subfig4} show the test loss and the block evolution.

\begin{figure}[!ht]
\label{xp2loss_weights}
    \centering
   \begin{subfigure}{0.48\textwidth}
    \includegraphics[width=\textwidth]{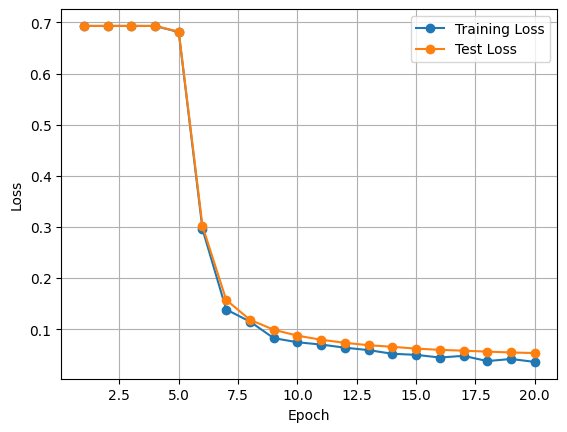}
    \caption{Test loss with uniform inputs}
    \label{fig:xp2loss}
\end{subfigure}
  \hfill
  \begin{subfigure}{0.48\textwidth}
    \includegraphics[width=\textwidth]{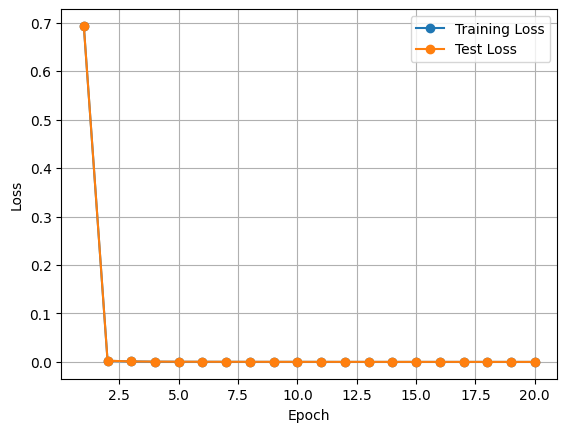}
    \caption{Test loss with Gaussian XOR inputs}
    \label{fig:xp2subfig2}
\end{subfigure}
\\
  \begin{subfigure}{0.48\textwidth}
    \includegraphics[width=\textwidth]{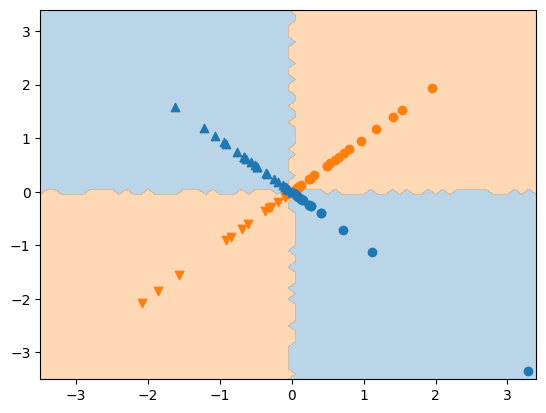}
    \caption{$T=2\times 10^4$, uniform inputs}
    \label{fig:xp2weights}
\end{subfigure}
  \hfill
  \begin{subfigure}{0.48\textwidth}
    \includegraphics[width=\textwidth]{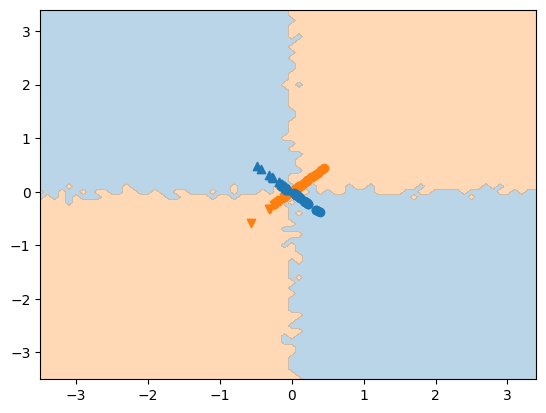}
    \caption{$T=2\times 10^4$, Gaussian XOR inputs}
    \label{fig:xp2subfig4}
\end{subfigure}
  
  \caption{Evolution of the test loss and the oracle network blocks size.}
\end{figure}

\subsection{Sensitivity to Label Noise}

We introduce label noise by flipping the labels of $5\%$ of the data points uniformly at random. 
As shown in Figure~\ref{fig:label_noise}, this modification leads to a clear degradation of the test loss. 
The corresponding decision boundary becomes biased due to the corrupted labels, and the weight vectors fail to maintain the orthogonal growth observed in the noise-free setting.

\begin{figure}[h]
\centering
\begin{subfigure}{0.48\columnwidth}
    \centering
    \includegraphics[width=\linewidth]{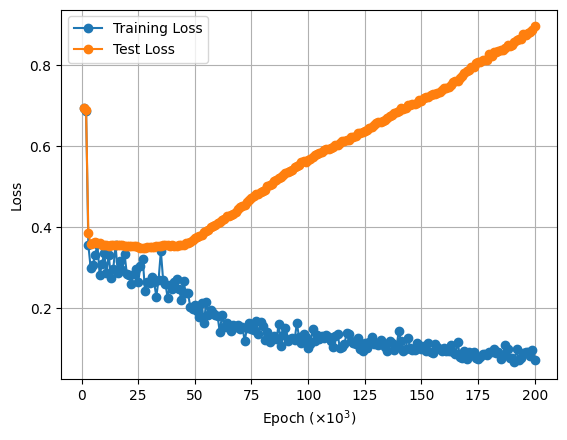}
    \caption{Test loss}
\end{subfigure}
\hfill
\begin{subfigure}{0.48\columnwidth}
    \centering
    \includegraphics[width=\linewidth]{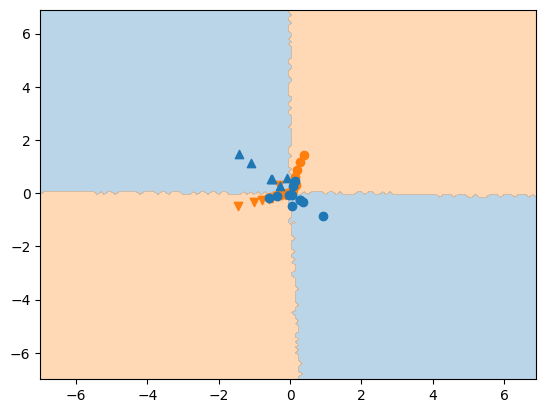}
    \caption{Boundary}
\end{subfigure}
\caption{Effect of label noise on training dynamics.}
\label{fig:label_noise}
\end{figure}

\subsection{Non-isotropic Gaussian inputs} 
We next consider inputs with covariance 
$\Gamma = 5\,\mu_1 \mu_1^\top + 1\,\mu_2 \mu_2^\top$. 
As shown in Figure~\ref{fig:anisotropy}, the decision boundary is still successfully recovered. 
However, the neuron orientations become disordered compared to the isotropic setting, reflecting the influence of the anisotropic covariance.

\begin{figure}[h]
\centering
\begin{subfigure}{0.48\columnwidth}
    \centering
    \includegraphics[width=\linewidth]{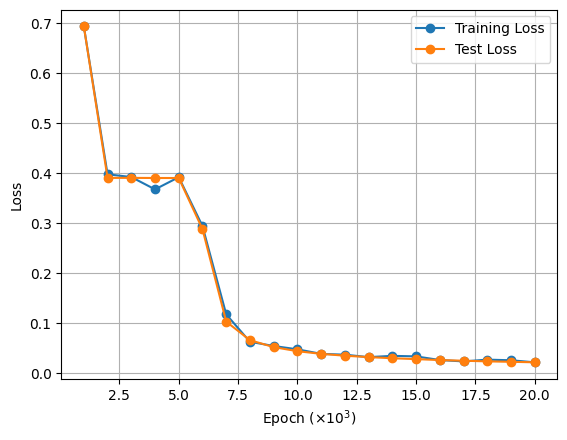}
    \caption{Test loss}
\end{subfigure}
\hfill
\begin{subfigure}{0.48\columnwidth}
    \centering
    \includegraphics[width=\linewidth]{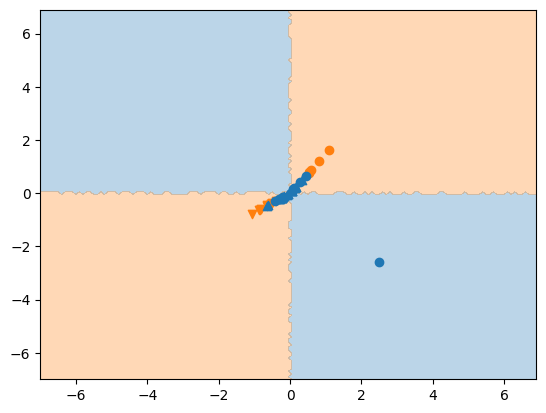}
    \caption{Boundary}
\end{subfigure}
\caption{Effect of anisotropic inputs on training dynamics.}
\label{fig:anisotropy}
\end{figure}

\subsection{Nonlinear decision boundaries} 
We next test the target function $f(x) = \mathrm{sgn}(x_2 - \sin(x_1))$ with Gaussian inputs. 
As shown in Figure~\ref{fig:nonlinear}, the learned boundary is approximately linear, 
and the weight vectors align within $\mathrm{span}\{e_1,e_2\}$. 
Training progresses for several thousand epochs before plateauing around epoch~5000.

\begin{figure}[h]
\centering
\begin{subfigure}{0.32\columnwidth}
    \centering
    \includegraphics[width=\linewidth]{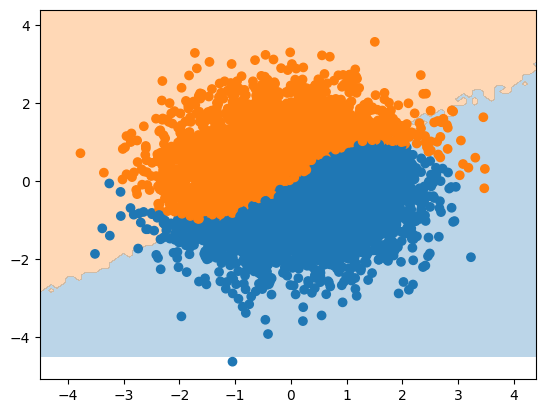}
    \caption{Learned boundary}
\end{subfigure}
\hfill
\begin{subfigure}{0.32\columnwidth}
    \centering
    \includegraphics[width=\linewidth]{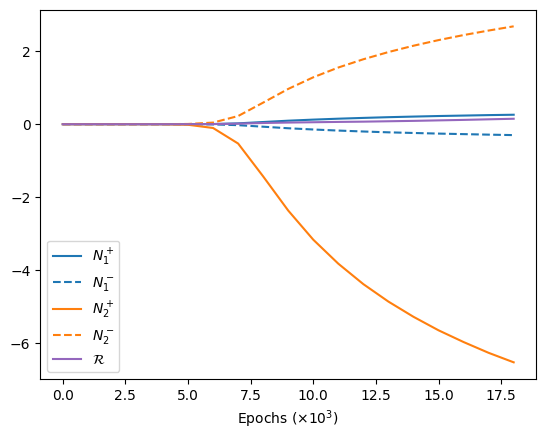}
    \caption{Weight alignment}
\end{subfigure}
\hfill
\begin{subfigure}{0.32\columnwidth}
    \centering
    \includegraphics[width=\linewidth]{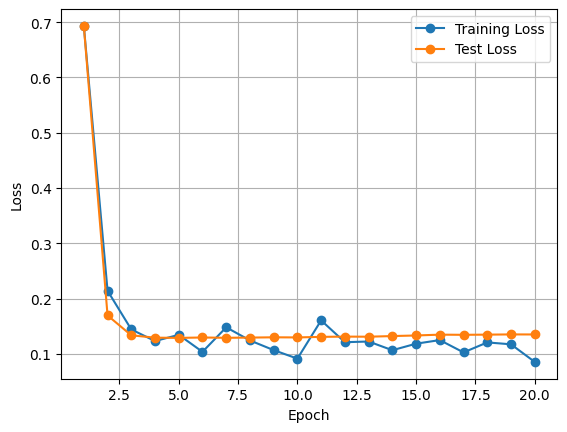}
    \caption{Training loss}
\end{subfigure}
\caption{Learning with nonlinear target $f(x) = \mathrm{sgn}(x_2 - \sin(x_1))$. 
The network recovers an approximately linear boundary, 
with weights restricted to $\mathrm{span}\{e_1,e_2\}$, 
and training plateaus after about 5000 epochs.}
\label{fig:nonlinear}
\end{figure}

\end{document}